\documentclass[journal]{IEEEtran}

\usepackage{bm}
\usepackage{amsmath}
\usepackage{epsf}
\usepackage{graphics}
\usepackage{ amssymb }
\usepackage{stackrel}
\usepackage[dvips]{graphicx}
\usepackage{mathtools}%
\usepackage{epsfig}
\usepackage{cite}
\usepackage{colortbl}
\usepackage{color}
\usepackage{enumitem}
\usepackage{soul,xcolor}
\usepackage{ragged2e}
\allowdisplaybreaks

\usepackage{bm}
\usepackage{amsmath}
\usepackage{epsf}
\usepackage{graphics}
\usepackage{ amssymb }
\usepackage[dvips]{graphicx}
\usepackage{epsfig}
\usepackage{cite}
\usepackage{graphicx}
\usepackage{epsfig}
\usepackage{latexsym}
\usepackage{amsfonts}
\usepackage{here}
\usepackage{rawfonts}

\usepackage[utf8]{inputenc}
\usepackage[english]{babel}
\usepackage{amsmath}
\usepackage{amsfonts}
\usepackage{amssymb}
\usepackage{color} %
\usepackage{bm}
\usepackage{listings}
\usepackage{caption}
\usepackage{amssymb}
\usepackage{amsthm}
\usepackage{graphicx}
\usepackage{epstopdf}
\usepackage{listings}
\usepackage{float}
\usepackage{amsmath}
\usepackage{amssymb}
\usepackage{amsfonts}
\usepackage{epstopdf}

\usepackage{multirow}
\usepackage{amscd}
\usepackage{mathrsfs}
\usepackage{graphicx}
\usepackage{makecell}
\usepackage{color}
\usepackage{url}
\usepackage{bm}
\usepackage{algorithm}
\usepackage{algorithmic}
\usepackage{setspace}
\usepackage{footnote}
\usepackage{xcolor}
\DeclareMathOperator*{\argmax}{arg\,max}

\DeclareMathOperator*{\argmin}{arg\,min}
\newtheorem{theorem}{Theorem}
\newtheorem{lemma}{Lemma}
\newtheorem{definition}{Definition}

\newtheorem{remark}{Remark}

\newtheorem{requirement}{Requirement}
\newtheorem*{requirement*}{Requirement}
\definecolor{columbiablack}{rgb}{0.61, 0.87, 1.0}
\newtheorem{assumption}{Assumption}

\usepackage{mathtools}
\DeclarePairedDelimiter{\ceil}{\lceil}{\rceil}

\addto\captionsenglish{}
\usepackage{bbm}

\usepackage{multicol}
\usepackage{hyperref}
\usepackage{mathtools}

\usepackage{lipsum,graphicx,subcaption}

\usepackage{diagbox}
\usepackage[protrusion=true,expansion=true]{microtype}
\pdfoutput=1
\usepackage[font=footnotesize]{caption}
\let\emptyset\varnothing

\usepackage{graphicx}
\usepackage{grffile}
\usepackage{tabularx}
\usepackage{booktabs}
\makeatletter
\newcommand{\vast}{\bBigg@{4}}

\newcommand{\Vast}{\bBigg@{5}}
\makeatother
\newcolumntype{Y}{>{\centering\arraybackslash}X}
\makeatletter 
\newcommand\semiHuge{\@setfontsize\semiHuge{22.72}{27.38}}
\makeatother

\begin{document}

\title{\semiHuge Online Federated Learning via Non-Stationary Detection and Adaptation amidst Concept Drift}

\author{Bhargav Ganguly and  Vaneet Aggarwal\thanks{The authors are with the School of Industrial Engineering,  Purdue University, West Lafayette IN 47907, USA, email: \{bganguly,vaneet\}@purdue.edu. V. Aggarwal is also with Computer Science,  King
Abdullah University Of Science And Technology, Thuwal 23955, Saudi
Arabia. }}
\maketitle

\begin{abstract}
	Federated Learning (FL) is an emerging domain in the broader context of artificial intelligence research. Methodologies pertaining to FL assume distributed model training, consisting of a collection of clients and a server, with the main goal of achieving optimal global model with restrictions on data sharing due to privacy concerns. It is worth highlighting that the diverse existing literature in FL  mostly assume stationary data generation processes; such an assumption is unrealistic in real-world conditions where concept drift occurs due to, for instance, seasonal or period observations, faults in sensor measurements. In this paper, we introduce a multiscale algorithmic framework which combines theoretical guarantees of \textit{FedAvg} and \textit{FedOMD} algorithms in near stationary settings with a non-stationary detection and adaptation technique to ameliorate FL generalization performance in the presence of concept drifts. We present a multi-scale algorithmic framework leading to $\Tilde{\mathcal{O}} ( \min \{ \sqrt{LT} , \Delta^{\frac{1}{3}}T^{\frac{2}{3}} + \sqrt{T} \})$ \textit{dynamic regret} for $T$ rounds with an underlying general convex loss function, where $L$ is the number of times non-stationary drifts occurred and $\Delta$ is the cumulative magnitude of drift experienced within $T$ rounds.
\end{abstract}

\section{Introduction}

Advancements in technology and science have led to an increase in raw data generation and processing power at smart devices, prompting the development of large-scale distributed machine learning architectures, such as Federated Learning (FL), which enables local training at end-user devices and periodic global model synchronization. However, most existing FL literature assumes time-invariant data generation processes \cite{hosseinalipour2020federated}. In reality, data generation processes at end devices can be impacted by abrupt changes in the underlying environment (e.g., pandemic on flight booking data \cite{garg2021distribution}). This paper models the non-stationary data generation environment and analyzes the \textit{dynamic regret} of the proposed algorithms in this setup.

In FL literature, such non-stationary behavior observed in the data generation model over time is commonly called \textit{concept drift} \cite{ganguly2022multi,mallick2022matchmaker}.  We note that such a phenomenon is also called \textit{covariate shift} in the wider context of general machine learning \cite{sugiyama2012machine}. Conventional FL methodologies such as the \textit{FedAvg} \cite{mcmahan2017communication} being agnostic to such time varying data shifts end up producing worse generalization results especially on ML classification/regression tasks. Hence, it is critical to augment such learning frameworks with non-stationarity detection and adaptation procedures to mitigate staleness/poor generalization ability of obtained ML models. Most of the existing work in non-stationary FL literature leverage heuristic techniques to ensure model robustness in the midst of drifts, including sliding window  based adaptive learning \cite{yang2021lightweight}, ensemble learning \cite{abbasi2021elstream}, and regularization mechanisms \cite{casado2022concept}. Although the problem has been widely studied for algorithms, %
this paper provides the first results on \textit{dynamic regret} for online convex optimization for {\color{black}general convex functions} in the FL setup. We note that online convex optimization has been studied in dynamic environments for both centralized and distributed settings (See Sec. \ref{Sec: related_work} for detailed comparison). {\color{black} We note that the methodologies proposed for centralized learning in non-stationary environments leverage the convenience of having the exact knowledge of both newly collected datasets and models at every learning round which is not available in federated learning.} {\color{black} Furthermore,  raw data offloading by the clients is restricted in FL due to privacy concerns.}

\begin{table*}[t]
\centering
\resizebox{1.5\columnwidth}{!}	{\begin{tabular}{|c|c|c|c|c|c|}
		\hline
		References             & Problem setting          &  Non-Stationary Measures   & Prior Knowledge  & Key Assumptions & \textit{Dynamic Regret} Bound
		\\ 
		\hline
		\cite{zinkevich2003online} &     Centralized             & $C_T$  & No & BLCL & $\tilde{\mathcal{O}}(\sqrt{T}(C_T+1))$ \\ 
		\hline
		\cite{jadbabaie2015online} &     Centralized             & $C_T, D_T, \Delta$  & No & BLCL, LG & $\Tilde{\mathcal{O}}(\min\{\sqrt{D_T + 1} + \sqrt{(D_T + 1)C_T}, (D_T + 1)^{1/3}T^{1/3}{\Delta}^{1/3}\})$ \\ 
		\hline
		\cite{shahrampour2017distributed}  &      Distributed              & $C_T$  & Yes & BLCL & $\tilde{\mathcal{O}}(\sqrt{T(C_T + 1)})$ \\ 
		\hline
		\cite{lu2019online} &      Distributed              & $C_T$  & No & BLCL &  $\tilde{\mathcal{O}}(\sqrt{(C_T + 1)}T^{3/4})$
		\\ 
		\hline
		\cite{li2021distributed} &      Distributed              & $C_T, D_T$  & No &  BLCL,  LG & $\tilde{\mathcal{O}}(\sqrt{TC_T}  + \sqrt{T} + D_T + \sqrt{C_T D_T})$
		\\ 
		\hline
		\rowcolor{columbiablack}
		This work             & Federated  & $L, \Delta$ & No & BLCL & $\Tilde{\mathcal{O}} ( \min \{ \sqrt{LT} , \Delta^{\frac{1}{3}}T^{\frac{2}{3}} + \sqrt{T} \})$ \\ \hline
\end{tabular}}
\caption{ \normalfont{ \textit{Dynamic regret} bounds for centralized, distributed, and federated online convex optimization methodologies. For the key assumptions, we use the following shortened descriptions :  BLCL $\rightarrow$ Bounded, Lipschitz, Convex Loss Function; LG $\rightarrow$ Lipschitz Gradient.}}
\label{regret_table}
\end{table*}

{\color{black}We highlight that the key novelty of this work is an efficient \textit{drift detection and adaptation} method that is suitably augmented with a randomized baseline FL algorithm scheduling procedure and facilitates \textit{training at multiple scales}.  To demonstrate what our framework accomplishes at a high level of abstraction, let us consider a scenario where FL training is desired to be conducted over $T = 7$ with \textit{FedAvg}. Our proposed framework will equip clients to train over shorter allowable time horizons in powers of 2, i.e., $\{1,2,4\}$ for the current example, with suitable learning rates for each such time chunks, i.e., \textit{FedAvg} $\{\frac{1}{\sqrt{1}}, \frac{1}{\sqrt{2}}, \frac{1}{\sqrt{4}} \}$ for horizons of lengths 1,2,4, respectively. We highlight that our framework carefully randomizes how such shorter training chunks are scheduled and it is not necessary that every allowable chunk is included for training.  
Furthermore, for the current example, during each round $t \in \{1,2, \cdots, 7 \}$, additional non-stationary tests that involve FL training loss collected by the clients are introduced in our method to identify and adaptively train for \textit{concept drift}.}

{ \color{black} In the context of ensuing discussion on our proposed methodology, the construction of the aforementioned pieces: \textit{drift detection and adaptation}, and \textit{multiscale learning} are in fact closely tied to the \textit{near-stationarity} properties of the baseline FL algorithms.} More specifically, we demonstrate mathematically why baseline FL algorithms may not directly ensure optimal learning in high degrees of drifts. Along these lines, we further show how their favorable behavior in \textit{near-stationary} regime may be leveraged to derive a supplementary quantity that can support drift aware learning in a {multiscale} fashion. In the later sections, we explicitly delineate how the key components of our algorithmic framework are concretely established based on the aforementioned mathematical ideas and bring out the connections more explicitly in the theoretical analysis of \textit{dynamic regret} incurred by our algorithm.  In summary, the major contributions of our work are as follows:
\begin{enumerate}[leftmargin=*]%
\item We propose a multi-scale algorithmic framework which can equip any existing baseline FL methodologies that work well in \textit{near-stationary} environments with a suitably designed change detection and adaption technique to mitigate model staleness in high drifting learning environments. More specifically, our current methodology is a general unified framework which is shown to be directly compatible with two widely deployed baseline FL algorithms: \textit{FedAvg} \cite{mcmahan2017communication} and \textit{FedOMD} \cite{fedomdpaper}.  

\item We conduct comprehensive theoretical analysis of the proposed framework and characterize the performance by \textit{dynamic regret} over all the FL rounds in terms of both the accumulated magnitude of such drifts (denoted by $\Delta$, see Definition \ref{defn: model drift}), and the number of times data drifts have occurred  (denoted by $L$, see Definition \ref{defn: num_drifts}). Mathematically, we show that adopting our approach leads to bounding the \textit{dynamic regret} over $T$ FL rounds with a general convex underlying loss measure as $\Tilde{\mathcal{O}} ( \min \{ \sqrt{LT} , \Delta^{\frac{1}{3}}T^{\frac{2}{3}} + \sqrt{T} \})$.
\item {\color{black} We provide proof-of-concept experimental evaluations on ML image classification datasets bolstering the efficacy of the proposed framework, wherein we compare the performance with two other widely-used competing FL algorithms.}
\end{enumerate}

We organize rest of the paper as summarized next. In Section \ref{Sec: related_work}, we present a %
comparison of our results with prior works in centralized and distributed non-stationary optimization. In Section \ref{sec: problem_formulation}, we formally describe the problem of \textit{dynamic regret} minimization for Federated Learning in drifting environments, whereby we also introduce key assumptions and definition necessary for our algorithm development. In Section  \ref{Sec: algo_framework}, we explain the various components and sub-routines associated with our multi-scale Federated Learning methodology. Furthermore, we present rigorous analysis as well as elucidation of the obtained mathematical results for the proposed algorithm leading to the aforementioned \textit{dynamic regret} bound in Section \ref{sec: key_theory}. {\color{black} In Section \ref{Sec: experiments}, we report our simulation results investigating how our framework behaves in practice when deployed on real-world datasets under various drift scenarios.}  %
\section{Related Work on Online Convex Optimization in Dynamic Environments} \label{Sec: related_work}
For centralized convex optimization in dymanic environments, several existing works propose online algorithms with \textit{dynamic regret} bounds in terms of problem-specific quantities. In this regard, three very commonly used metrics are: \textit{comparator regularity} accumulating the changes in the minimizers over $T$ horizon, i.e., $C_{T}$; \textit{temporal variability} capturing the differences in the loss function values, i.e., $\Delta$; and \textit{gradient variability} which tracks loss in gradients instead of actual function lossses, i.e., $D_T$ (see definition of $C_T, D_T$ in Appendix \ref{app: other_drift_measures_defn}). In our problem setting, the usage of $\Delta$ (see Definition \ref{defn: model drift}) is inline with mathematical formulations of drift in recent FL literature \cite{ganguly2022multi,hosseinalipour2022parallel},  and is direct adoption of \textit{temporal variability} to a distributed learning setting. Furthermore, we replace $C_{T}$ with a more  intuitive drift measure $L$, i.e., the exact number of times drift occurred over horizon $T$. Previous works based on adaptive online gradient descent \cite{zinkevich2003online}, online mirror descent \cite{jadbabaie2015online} are centralized counterparts to our setting; and produce a \textit{dynamic regret} of $\tilde{\mathcal{O}}(\sqrt{T}(C_T + 1))$ and $\Tilde{\mathcal{O}}(\min\{\sqrt{D_T + 1} + \sqrt{(D_T + 1)C_T}, (D_T + 1)^{1/3}T^{1/3}{\Delta}^{1/3}\})$,  respectively. On the other hand, we provide a general \textit{dynamic regret} bound of $\Tilde{\mathcal{O}} ( \min \{ \sqrt{LT} , \Delta^{\frac{1}{3}}T^{\frac{2}{3}} + \sqrt{T} \})$ with the flexibility of choosing any baseline FL optimizer which can guarantee worst case $\tilde{\mathcal{O}}(\sqrt{T})$ regret in any \textit{near-stationary} environment. Furthermore, we note that the strategies for centralized setup do not directly provide the \textit{dynamic regret} guarantees in the distributed/federated scenarios since the distributed/federated setup is more challenging as the data collection and learning process is spread across a network of client devices.

In the distributed setting, an online mirror descent algorithm was proposed in \cite{shahrampour2017distributed} achieving a $\tilde{\mathcal{O}}(\sqrt{T(C_T + 1)})$ assuming prior knowledge of $C_T$ and $T$. The authors of \cite{lu2019online} present an online gradient tracking methodology; thereby obtaining a $\tilde{\mathcal{O}}(\sqrt{(C_T + 1)}T^{3/4})$  with no prior knowledge of the environment dynamics. Furthermore, the aforementioned bound was improved by introducing \textit{gradient variability} $D_T$ in conjunction with $C_T$ to $\tilde{\mathcal{O}}(\sqrt{TC_T}  + \sqrt{T} + D_T + \sqrt{C_T D_T})$ in \cite{li2021distributed}.  We note that when $C_T = O(T^{1-\epsilon})$ for small $\epsilon$, the bound in \cite{lu2019online} is larger than $\tilde{\mathcal{O}}(T)$, making this bounds not interesting in large drift scenarios, since in our case when $L=O(T^{1-\epsilon})$, the bound is $\tilde{\mathcal{O}}(T^{1-\epsilon/2})$ which is sub-linear. Further, the bound in \cite{li2021distributed} is in terms of gradient changes, while our bound is independent of gradient changes. 
More specifically, we highlight that the bounds revealed in terms of \textit{gradient variability} measure $D_T$ in \cite{jadbabaie2015online,li2021distributed} are direct consequences of additional Lipschitz smoothness assumption on loss gradients, whereas we keep our analysis restricted to general bounded, Lipschitz smooth convex loss measures.

{\color{black}We note that  \textit{dynamic regret} bounds could be  improved to $\tilde{\mathcal{O}}(1 + C_T)$ when loss functions are strongly convex as reported in recent studies pertaining to distributed ML \cite{dixit2020online,eshraghi2022improving}. However, our proposed methodology relies only on the underlying loss measure being general convex, and therefore is not directly comparable to the aforementioned studies requiring strong convexity assumptions. }

To the best of our knowledge, our algorithm design and consequent theoretical findings in FL setting extends the aforementioned prior bounds in the distributed paradigm where the underlying loss function is general convex, while it introduces tighter bounding terms using the metric $L$ representing the number of changes instead of $C_T$. Furthermore, we note that although our algorithmic framework is designed for federated optimization, it can be straightforwardly extended to a distributed learning setup with consensus-based model averaging. More specifically, our results are indifferent to how model averaging is conducted across the clients, therefore, will hold in general distributed convex optimization in dynamic environments. The comparison of \textit{dynamic regret} results discussed both in the centralized/distributed regime discussed thus far are summarized in Table \ref{regret_table}.  

To emphasize, our work is well-aligned with dynamic federated optimization problem setting with the following high-level attributes: (i) it can combine any conventional baseline FL optimizer that works well in a \textit{near stationary} setting  with a multi-scale procedure that can detect and adapt in highly dynamic environments, and (ii) to the best of our knowledge, our framework is guaranteed to produce tighter sublinear \textit{dynamic regret} bounds for a more generic set of non-stationarity measures $L, \Delta$ and runtime horizon $T$. Prior knowledge of $L, \Delta$, and $T$ are not required.
\section{Problem Formulation} \label{sec: problem_formulation}
In this section, we formally put forth the mathematical problem aiming towards \textit{dynamic regret} minimization for FL over client devices with time-varying data drifts. In the following discussion, we also denote the client devices as Data Processing Units (DPUs) since they conduct local ML training, as well as to clearly differentiate from the central-coordinator server which only performs model syncronization and training orchestration tasks. We collectively denote the set of DPUs as $\mathcal{N}$, and the central server as $\mathcal{S}$. Furthermore, we assume that FL training and model syncronization is performed over $t = 1, 2, \cdots, T$ rounds. Further, we denote the dataset generated at each DPU $n \in \mathcal{N}$ during rounds $t = 1,2, \cdots , T$ as $\mathcal{D}_n^{(t)}$.
For the ease of our analysis and presentation of results, we define the following quantities:
\begin{align}
    & D_n^{(t)} = | \mathcal{D}_n^{(t)} |, \   D^{(t)} = \sum_{n \in \mathcal{N}} D_n^{(t)},\   p_{n}^{(t)} = \frac{ D_n^{(t)}}{D^{(t)}}.\label{eqn: ML dataset definition} 
\end{align}
At rounds $t = 1,\cdots, T$, each DPU $n\in\mathcal{N}$ is associated with a \textit{local loss} function ${F}_{n}^{(t)}(\mathbf{x})$. It locally computes the loss gradient at the current global parameter vector $\mathbf{x}^{(t-1)}$, i.e., $\nabla {F}_{n}^{(t)}(\mathbf{x}^{(t-1)})$. For the ML model vector, we assume that $\mathbf{x} \in \mathbbm{R}^{d}$. Formally,
   \begin{align}\label{eq:localLossInit}
    & {F}_{n}^{(t)}(\mathbf{x}) =  \frac{1}{{{D}}_{{n}}^{(t)}}{\underset{\xi \in {{\mathcal{D}}}_{{n}}^{(t)}}{\sum} {f}(\mathbf{x}} ; \xi), \\
    & \nabla {F}_{n}^{(t)}(\mathbf{x}^{(t-1)}) = \frac{1}{{{D}}_{{n}}^{(t)}}{\underset{\xi \in {{\mathcal{D}}}_{{n}}^{(t)}}{\sum} \nabla {f}(\mathbf{x}} ; \xi) \big|_{\mathbf{x} = \mathbf{x}^{(t-1)}},
\end{align}
where $f(\cdot~;~\cdot)$ is the underlying ML loss function. Subsequently, all DPUs $n \in \mathcal{N}$ locally update their local ML model, i.e., $\mathbf{x}_n^{(t)}$. We note that this update procedure is the key difference across different conventional FL algorithms. We call this update procedure as \texttt{FL-UPDATE($\cdot$)}. Hence, we have:  
\begin{align}
    \label{eqn: FL-update} \mathbf{x}_n^{(t)} = \texttt{FL-UPDATE}\big(\mathbf{x}^{(t-1)}, \nabla {F}_{n}^{(t)}(\mathbf{x}^{(t-1)}) \big).
\end{align}
In this regard, we highlight that the explicit formulation of \texttt{FL-UPDATE($\cdot$)} for \textit{FedAvg, FedOMD} has been provided in Section \ref{sec: base_FL_disc_main}.\\

\vspace{-3.4mm}
Subsequently, each  DPU $n \in \mathcal{N}$ communicates its locally updated model, i.e., $\mathbf{x}_n^{(t)}$ to the central aggregation server which we denote by $\mathcal{S}$. $\mathcal{S}$ then performs the aggregation step to obtain the global ML model, i.e., $\mathbf{x}^{(t)}$ for each $t$, according to the  following:
\begin{align} 
    \mathbf{x}^{(t)} = \sum_{n \in \mathcal{N}} p_{n}^{(t)} \mathbf{x}^{(t)}_{n}. \label{eqn: global_ML_aggr}
\end{align}
After the aggregation step, this $\mathbf{x}^{(t)}$ is sent back by $\mathcal{S}$ to DPUs in $\mathcal{N}$ for conducting next round of local ML training. This process of local training and periodic aggregation is repeated across all the federated learning rounds $t \in [1, T]$.

{ \color{black} \begin{remark}
    Often a small number of model parameters are
updated locally for any client DPU in a FL learning setting, and usually the learnt ML models (for
instance DNN-based image classifiers) are large-scale training architectures which potentially induce
communication bottlenecks in the network. In such cases, exchanging model differentials instead of the actual model itself ameliorates communication overhead significantly. %
Our framework can directly accommodate such communication efficient message passing setups that involve model differentials (as detailed in Appendix \ref{app: model_differential_revision}).
\end{remark}
}
Next, we elaborate on the essential quantities that are heavily utilized throughout our theoretical analysis pertaining to online \textit{dynamic regret}. In this context, we first highlight that we are concerned with the ML losses incurred by the DPUs at the end of each round of FL training and global aggregation. More precisely, the ML model $\mathbf{x}^{(t)}$ updated during rounds $t = 1,2,\cdots, T$ is applied on the observed datasets $\mathcal{D}_n^{(t)}$ to collect local loss ${F}_{n}^{(t)}(\mathbf{x}^{(t)})$ as defined in Eq. \eqref{eq:localLossInit}. The local ML losses are combined according to the proportion of dataset sizes thereby obtaining the global losses as follows:
\begin{align}
    \label{eqn:defn_global_loss_1} {\color{black} {F}^{(t)}(\mathbf{x}^{(t)}) = \displaystyle \sum_{n \in \mathcal{N}} p_n^{(t)}  F_{n}^{(t)}(\mathbf{x}^{(t)}). }
\end{align}
Finally, in our system model, the objective of the network comprised by DPUs in $\mathcal{N}$ and server $\mathcal{S}$ is to minimize the global \textit{dynamic regret} $R_{[1,T]}$ given by:
\begin{align}
   R_{[1,T]} = \sum_{t = 1}^{T} {F}^{(t)}(\mathbf{x}^{(t)}) ~-  \sum_{t = 1}^{T} {F}^{(t)}(\mathbf{x}^{(t), *}), \label{defn: dynamic regret}
\end{align}
where ${\color{black}\mathbf{x}^{(t), *} = \underset{\mathbf{x} \in \mathbbm{R}^d}{\argmin} ~{F}^{(t)}(\mathbf{x}). }$
We now introduce the key assumptions and definitions pertaining to loss functions and baseline FL methods that will be used throughout the course of our analysis.
\begin{assumption}[Convexity, Smoothness and Boundedness of ML loss function] \label{assumption: convexity_+_lipschitz} We assume that the underlying ML loss function, i.e., $f(\cdot; \xi) $ is convex, bounded in $[0,1]$ and $\mu$-Lipschitz w.r.t. $\|\cdot\|$, which implies the following $\forall \mathbf{x}, \Tilde{\mathbf{x}} \in \mathbbm{R}^{d}$:
\begin{align}
    {f}(\mathbf{x} ; \xi) - {f}(\Tilde{\mathbf{x}} ; \xi) \leq \mu \|\mathbf{x} - \Tilde{\mathbf{x}} \|.
\end{align}
\end{assumption}
We highlight that Assumption \ref{assumption: convexity_+_lipschitz} is general enough and captures a wide range of problems pertaining to Machine Learning and Online Optimization tasks.
\begin{definition}[\textit{Concept Drift}] \label{defn: model drift}
The online {Concept Drift} between two consecutive rounds of global aggregation $t-1$ and $t$ is measured by $\Delta_{t} \in \mathbb{R}^{+}$, which captures the maximum variation of the global loss function for any arbitrary ML model $\mathbf{x} \in \mathbbm{R}^d$ according to
\begin{align}
    |F^{(t)}(\mathbf{x}) - F^{(t-1)}(\mathbf{x})| \leq \Delta_{t}. \label{eqn: drift defn}
\end{align}
Further, the cumulative concept drift incurred across rounds $t_1, t_1 + 1, \cdots, t_2$ is denoted by:
\begin{align}
    \Delta_{[t_1, t_2]}  \triangleq \sum_{\tau = t_1}^{t_2} \Delta_{\tau}.
\end{align}
\end{definition}
\begin{definition}[Number of Non-Zero Distribution Shift events] \label{defn: num_drifts} Using the notion of online concept drift $\Delta$ in Definition \ref{defn: model drift}, we quantify the number of FL rounds with non-zero drift across $t = 1,2, \cdots, T$, and denote it by $L$. Mathematically,
\begin{align}
    L = \sum_{t = 1}^{T} \mathbbm{1}[\Delta_t \neq 0].
\end{align}
\end{definition}
The non-stationary measures $L, \Delta$ quantify the degree of time-varying nature of the datasets generated at the DPUs. Hence, larger drifts should directly imply more frequent calibration (i.e., change detection and adaptation) to achieve staleness resiliency for learnt ML models. {\color{black} We note that the existing methodologies for drift management include both local drift, where the statistical properties of the data distribution that individual clients experience over time \cite{localdrift1}, and global drift, where  the statistical properties of the data distribution across all the clients over time \cite{mallick2022matchmaker,wang2020tackling}. In this paper, we consider global drift, while the results would be applicable even for the local drift. } Next, we define the functions $\rho (\cdot), C(\cdot)$ which we later use to mathematically characterize the notion of \textit{near-stationarity}.
\begin{definition} \label{defn: rho_def}
We define a non-increasing function $\rho : [t] \rightarrow \mathbbm{R}$ such that  $\rho(t) \geq \frac{1}{\sqrt{t}}$, and $C(t) = t\rho(t)$ is an increasing function of $t$. Furthermore, for a given $T$, we define $\hat{\rho}(t) = 6(\log_2 T + 1)\log (T/\delta)\rho(t)$.
\end{definition}
In our framework, we require the underlying baseline FL algorithms to have certain guarantees in environments where the cumulative drift is small. We denote such environments as \textit{near-stationary}. This requirement from baseline algorithms in \textit{near-stationary} settings is expressed via Assumption \ref{assmptn: near stationary} which is presented next.
\begin{requirement} [Base algorithm performance guarantee in a near-stationary environment] \label{assmptn: near stationary} We assume that the base algorithm produces an auxiliary quantity $\Tilde{F}^{(t)}$ at the end of each global round of aggregation $t \in \{1,2, \cdots, T\}$ satisfying the following:  
\begin{align}
    & \Tilde{F}^{(t)} \leq \underset{\tau \in [1, t]}{\max} F^{(\tau)}(\mathbf{x}^{(\tau),*}) + \Delta_{[1, t]}, \label{eqn: optimistic_est_1} \\
    & \frac{1}{t}\sum_{t = 1}^{t} [{F}^{(t)}(\mathbf{x}^{(t)}) - \Tilde{F}^{(t)}] \leq \rho(t) + \Delta_{[1, t]}. \label{eqn: optimistic_est_2}
\end{align}
where $\rho(.)$ is described in Definition \ref{defn: rho_def}. Further, $\Delta_{[1,t]}$ represents the cumulative concept drift experienced with $\Delta_{[1,t]} \leq \rho(t)$ , i.e., near-stationary environment.
\end{requirement}
We emphasize that this supplementary estimator $\Tilde{F}^{(t)}$ is not a direct output of most conventional FL algorithms. However, the construction of this quantity is possible from the models learnt over time. Specifically, we outline the exact form of this quantity and verify the correctness of Assumption \ref{assmptn: near stationary} for conventional baseline FL algorithms: \textit{FedAvg}, \textit{FedOMD} in Appendix  \ref{sec: optimistic_estimator_verification}.\\

{\color{black} Furthermore, a key intuition behind Requirement \ref{assmptn: near stationary} is the idea a single instance of a vanilla baseline FL algorithm should be enough when environment is \textit{near-stationary}, i.e., $\Delta_{[1,t]} \leq \rho(t)$. However, with increased degree of non-stationarity i.e., $\Delta_{[1,t]} > \rho(t)$, the conditions stated in Requirement \ref{assmptn: near stationary} should be somehow proxied to detect drifts and restart learning. We demonstrate this key intuition with mathematical justifications in Appendix \ref{sec: near_stationary_intuition}.}

In our subsequent discussion delineating the algorithmic framework for non-stationary FL, we explicitly demonstrate how this auxiliary quantity is useful to design suitable drift detection tests (\textbf{Test 1} and \textbf{Test 2} in Algorithm \ref{detection}, Section \ref{Sec: algo_framework}) {\color{black} that attempt to proxy the conditions of Requirement \ref{assmptn: near stationary}, thereby promoting multi-scale learning in dynamic environments.} 

\section{Algorithmic Framework} \label{Sec: algo_framework}
In this section, we provide an elaborate description of our multi-scale algorithmic framework to perform FL in dynamic environments. In this regard, we note that we adopt the notions of multi-scale base algorithm initializations and drift detection mechanism for FL from \cite{wei2021non}, which was originally proposed for non-stationary reinforcement learning training. Roughly speaking, in our framework, multiple base FL instances are scheduled, and in turn, are equipped with a carefully engineered change detection mechanism to mitigate model staleness. These base FL instances are conventional FL algorithms that satisfy Assumption \ref{assmptn: near stationary}, we specifically show in Appendix \ref{sec: optimistic_estimator_verification} that this Assumption holds for \textit{FedAvg} and \textit{FedOMD}.

Next, we present Randomized Scheduling Procedure in Algorithm \ref{scheduling_algo}. { \color{black} %
An instance $\mathcal{A} =  (\mathcal{A}.s, \mathcal{A}.e, \eta^{A}, \mathbf{x}^{\mathcal{A}} )$ is desired to be scheduled over the horizon $[\mathcal{A}.s, \mathcal{A}.e]$.  $\eta^{A}=\frac{1}{\sqrt{\mathcal{A}.e - \mathcal{A}.s + 1}}$ is the learning rate for $\mathcal{A}$ over the scheduled horizon $[\mathcal{A}.s, \mathcal{A}.e]$  and  $\mathbf{x}^{\mathcal{A}}$ is the corresponding model learnt over the aforementioned horizon for instance  $\mathcal{A}$. We also provide an elaborate mathematical reasoning for the choices of learning rates for \textit{FedAvg}, \textit{FedOMD} in Section \ref{sec: key_theory}. Furthermore, we highlight that all the scheduling orchestrations and change detection procedure are conducted via server $\mathcal{S}$, whereas model training is still locally conducted by the client DPUs, i.e., $\mathcal{N}$.
}

\begin{algorithm}[ht]
	\caption{Randomized Scheduling Procedure} \label{scheduling_algo}
    \begin{algorithmic}[1]
         \STATE {{\textbf{Input:} $m$, $\rho(\cdot)$, First round $t$.}} 
          \STATE {{\textbf{Output:} A collection of base FL instances $\bm{A}$}}
          \STATE \textbf{Initialize:} $\bm{A} = \emptyset$.
          \FOR{$\tau = t, t+1, \cdots, t + 2^m - 1$}
          \FOR{$k = m, m-1, \cdots, 0$}
          \IF{$\tau ~\texttt{mod} ~2^k = 0$}
          {\color{black} \STATE Schedule $\mathcal{A}:= (\mathcal{A}.s, \mathcal{A}.e, \eta^{A}, \mathbf{x}^{\mathcal{A}} )$ w.p. {\color{black} $\frac{\rho(2^m)}{\rho(2^k)}$}.}
          \STATE Update Instance set : $\bm{A} = \bm{A} \cup \mathcal{A}$.
          \ENDIF
          \ENDFOR
          \ENDFOR
    \end{algorithmic}
\end{algorithm}
Algorithm \ref{scheduling_algo} acts as a subroutine that can schedule multiple instances of base algorithm at different orders of scale in a carefully constructed randomized fashion. It populates a set of base FL instances, i.e., $\bm{A}$ with run lengths upto input order $m$ and scheduled to start at different time periods for a particular block of rounds, i.e., $[t, t + 2^m -1]$. {\color{black} Next, we present a more formal definition of \textit{maximum scheduling order} input $m$ for Algorithm \ref{scheduling_algo}.  %

\begin{definition} [\textit{Maximum Scheduling Order} $m$]
    The maximum order for which Randomized Scheduling Procedure (Algorithm \ref{scheduling_algo}) can be executed over an arbitrary interval $\mathcal{I} = [t_1, t_2]$ is given by $m_{\mathcal{I}} = \ceil{\log_{2} (t_2 - t_1 + 1)}$.
\end{definition}
Next, in Algorithm \ref{scheduling_algo}, scheduling of base FL instances only happens with certain probability, i.e., {\color{black} $\frac{\rho(2^m)}{\rho(2^k)}$} for a certain order-$k$ time block.} The scheduling procedure is executed at the FL central server $S$, and subsequently communicated to all the DPUs in $\mathcal{N}$. It is necessary to emphasize that each DPU $n \in \mathcal{N}$ follows the same base algorithm schedule, however it conducts training only on its own local dataset $\mathcal{D}_{n}^{(t)}$ and updates local model to $\mathbf{x}_{n}^{(t)}$.

\begin{algorithm}[t]
    \caption{Multi-Scale FL Runner (MSFR)}
	\label{MALG}
	\begin{algorithmic}[1]
	   \STATE {\textbf{Input:} $m$, $\rho(\cdot)$, FL round $t$, \texttt{RunScheduler}} 
		\STATE \textbf{Output:} Current Loss $F^{(t)}$, Optimistic FL Loss estimate $\tilde{F}^{(t)}$.
		\IF{\texttt{RunScheduler} = {True}}
		\STATE $\bm{A} \leftarrow \text{Randomized Scheduling Procedure}(m, \rho(\cdot), t)$.
		\STATE \texttt{RunScheduler} $\leftarrow$ {False}.
		\ENDIF
        \STATE {\color{black} At round $t$, $\mathcal{A}_t  = \underset{\mathcal{A} \in \bm{A}}{\argmin} \mathcal{A}.e - t$ to be run at all DPUs.} 
		\STATE {\color{black} Receive $\{F^{(t)}, \tilde{F}^{(t)} \}$ from $\mathcal{A}_t$; $F^{(t)}$ is on model $\mathbf{x}^{(t), \mathcal{A}_t}$.}
        \end{algorithmic}
\end{algorithm} 

The aforementioned Randomized Scheduling Procedure (Algorithm \ref{scheduling_algo}) is only triggered via Multi-Scale FL Runner (Algorithm \ref{MALG}). Multi-Scale FL Runner (Algorithm \ref{MALG}) is dedicated to perform two tasks: scheduling base algorithms at the client DPUs requested via \texttt{RunScheduler} input flag, and decide the correct base instance run at the DPUs. Clearly, such an orchestration mechanism requires intervention of the server node $\mathcal{S}$.  {\color{black} From $\mathcal{S}$, all the DPUs in $\mathcal{N}$ are directed to run the base instance with the \textit{shortest remaining run length} whose schedule overlap with current FL round $t$ in step 7 of Algorithm \ref{MALG}. We denote this instance as $\mathcal{A}_t$.  We illustrate this greedy selection procedure for the baseline algorithm instances via an example presented next.

In Figure \ref{fig:master_fig_demo}, we demonstrate the execution of different scheduled instances based on the shortest remaining length greedy rule in step 7 of Algorithm \ref{MALG}. We provide a toy example with $m=4$ to illustrate how this rule works, as shown in \cite{wei2021non}. As depicted in Figure \ref{fig:master_fig_demo}, Algorithm \ref{scheduling_algo} schedules one instance of length $2^4$ (in red), two instances of length $2^2$ (in green), two instances of length $2^1$ (in black), and five instances of length 1 (in purple). The bold sections of the lines represent the active periods of each instance. Furthermore, consider that the base instances are \textit{FedAvg}. This implies the learning rates are $\frac{1}{\sqrt{2^0}}, \frac{1}{\sqrt{2^1}}, \frac{1}{\sqrt{2^2}}, \frac{1}{\sqrt{2^4}}$ for the various order $k=0,1,2,4$ instances respectively. Clearly, the order-4 instance ran for the first round, and then was paused for the next 8 rounds. Subsequently, it was executed for 3 more rounds before it was paused. The dashed line marked as \textcircled{1} means that learning restarts at round 9 for the order-4 instance with the learnt model at the end of round 1. While, the dashed line marked as \textcircled{2} represents that 2 order-0 instances were scheduled consecutive to each other. This means that although the two instances are successive, but learning happens from scratch in both the instances.      
\begin{figure}[t]
    \centering
    \includegraphics[width=0.43\textwidth]{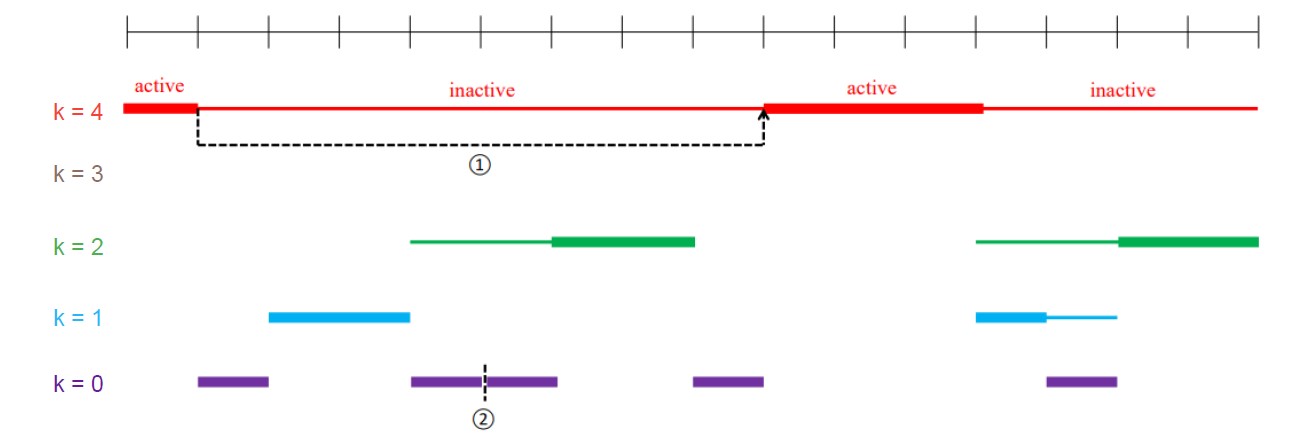} %
    \caption{\color{black} Master-FL execution example with $m=4$.}
    \label{fig:master_fig_demo}
\end{figure}
}

Consequently in step 8 of Algorithm \ref{MALG}, each DPU $n$, picks its model $\mathbf{x}_n^{(t-1), \mathcal{A}_t}$, performs model update according to  Eq. \eqref{eqn: FL-update}, and shares $\{\mathbf{x}_n^{(t), \mathcal{A}_t}, F^{(t)}_{n}(\mathbf{x}_n^{(t), \mathcal{A}_t}) \}$ with the central server $\mathcal{S}$. Finally, $\mathcal{S}$ computes and stores the instantaneous $\{F^{(t)}, \tilde{F}^{(t)} \}$. 
{\color{black}
\begin{remark}[Computation of $F^{(t)}$]
    To enforce protection of sensitive data while DPU's share their local losses $\{F^{(t)}_{n}\}_{n \in \mathcal{N}}$ to server $\mathcal{S}$ for computation of global loss measure $F^{(t)}$, privacy-preserving techniques, such as differential privacy, homomorphic encryption or secure multi-party computation, can be deployed for loss function computations in a privacy-preserving manner \cite{yin2021comprehensive} , while a detailed investigation is left as future work.%
\end{remark}
}

This Multi-Scale FL Runner (Algorithm \ref{MALG}) is executed during each FL round via Master-FL (Algorithm \ref{detection}) at the central server $\mathcal{S}$, we describe Master-FL next. 
\begin{algorithm}[ht]
    \caption{Master-FL} \label{detection}
			\begin{algorithmic}[1]
			\STATE \textbf{Input:} $\hat{\rho}(\cdot)$ (see Definition \ref{defn: rho_def}).
		\STATE \textbf{Initialize:} $t \leftarrow 1$.
		\FOR{$m=0,1, ... $}
		\STATE Set ${\color{black} t_{new}} \leftarrow t$, \texttt{RunScheduler} $\leftarrow$ {True}.
		\STATE \texttt{Test-1-Flag} $\leftarrow$ 1, \texttt{Test-2-Flag} $\leftarrow$ 1.
		\WHILE{$t < {\color{black} t_{new}} + 2^m$}
		\STATE $\{F^{(t)}, \tilde{F}^{(t)} \} \leftarrow \text{MSFR}(m, \rho (\cdot), t,\texttt{RunScheduler})$.  {\color{black} \textit{~~// Local training at DPUs $\mathcal{N}$}}\\
		\STATE Set $U_t = {\max}_{\tau \in [{\color{black} t_{new}}, t]} ~\tilde{F}^{(t)}$
		\IF{\texttt{Test-1-Flag} = 0 \OR \texttt{Test-2-Flag} = 0}
		\STATE \texttt{RunScheduler} $\leftarrow$ {True}.
		\STATE \texttt{Break}.
		\ENDIF
		\ENDWHILE
		 \ENDFOR
		{\color{black} \STATE \underline{\textbf{Test 1:}} Current $\mathcal{A}_t$ is some order $k$ base instance. 
         \IF{$t = \mathcal{A}_t.e$}
         \IF{$U_t \geq \sum_{\tau = \mathcal{A}.s}^{\mathcal{A}.e} ~F^{(t)} + 9\hat{\rho}(2^k)$}
         \STATE \texttt{Test-1-Flag} $\leftarrow$ 0.
         \ENDIF
         \ENDIF
        \item[]
		\STATE \underline{\textbf{Test 2:}}
        \IF{$\frac{1}{t - t_{new} + 1} \sum_{\tau = t_{new}}^{t} [F^{(t)} - \Tilde{F}^{(t)}] \geq 3 \hat{\rho}(t - t_{new} + 1)$}
        \STATE \texttt{Test-2-Flag} $\leftarrow$ 0.
         \ENDIF}
	\end{algorithmic}
\end{algorithm}
Master-FL tracks model training via Multi-Scale FL Runner (Algorithm \ref{MALG}) at each FL round $t = 1,2,\cdots$. Additonally, it performs two tests (see line 15, 16 in Algorithm \ref{detection}) to identify whether a significant drift has occurred. In this regard, \textbf{Test 1} intuitively allows the server $\mathcal{S}$ to examine whether a ``sudden" drift has occurred in the recent training rounds, by tracking every base algorithm upon its completion via actual loss $F^{(t)}$ and $U_t$ derived from auxiliary loss quantity $\Tilde{F}^{(t)}$. Whereas, \textbf{Test 2} keeps track of non-stationary drifts that gradually accumulated over longer time windows. {\color{black} In the following, we present more mathematical intuitions behind design of \textbf{Test 1,2} via an illustrative example.}

{\color{black} \noindent\textbf{Unpacking mathematical intuitions behind Test 1, 2:} ~Consider the decomposition of \textit{dynamic regret} expression as follows:
\begin{align}
    &R_{[1,t]} = \sum_{\tau = 1}^{t} {F}^{(\tau)}(\mathbf{x}^{(\tau)}) ~-  \sum_{\tau = 1}^{t} {F}^{(\tau)}(\mathbf{x}^{(\tau), *}), \label{dyn_regr_resp10_first}\\
    &= \underbrace{\sum_{\tau = 1}^{t} \Big[{F}^{(\tau)}(\mathbf{x}^{(\tau)}) - \Tilde{F}^{(\tau)} \Big]}_{\text{(a)}} %
   + \underbrace{\sum_{\tau = 1}^{t} \Big[\Tilde{F}^{(\tau)} - {F}^{(\tau)}(\mathbf{x}^{(\tau), *})  \Big]}_{\text{(b)}}. \label{dyn_regr_resp10}
\end{align}
When \textit{near-stationarity} conditions pertaining to Requirement \ref{assmptn: near stationary} are met via only a single instance of baseline FL algorithm, then term (a) is $t.\rho(t)$ due to Eq. \eqref{eqn: optimistic_est_2} and term (b) is  $\leq 0$. The overall regret remains unchanged, i.e., $\tilde{\mathcal{O}}(t.\rho(t))$.

However, beyond the \textit{near-stationarity} regime, i.e., $\Delta_{1,t} > \rho(t)$, both the terms can become drastically large. Note that components of term (a) are observable at the server $\mathcal{S}$ via periodic synchronizations, thereby allowing easy detection of abnormal changes for term (a). To this end, \textbf{Test 2} (line 21-24 in Algorithm \ref{detection}) takes care of term (a).\\

Term (b) in Eq. \eqref{dyn_regr_resp10} cannot be directly evaluated  since optimal models $\mathbf{x}^{(\tau), *}$ are unavailable. Large values of term (b) are owed to the fact that for one or more iterations $\tau$, the instanteneous optimal model $\mathbf{x}^{(\tau), *}$ was possibly sub-optimal at rounds $\leq \tau -1 $,  and would have caused global loss to be very high if used in those rounds. It is not possible to detect such models via single instance of one base FL algorithm particularly because one instance naturally explores only optimal gradient directions from its starting model using a fixed learning rate. Our proposed randomized multi-scale orchestration scheme that schedules and maintains several base FL instances particularly to enable detection for term (b). We summarize the idea of how this is achieved via a synthetic example in Figure \ref{fig:test1_illustration}.
\begin{figure}[t]
    \centering
    \includegraphics[trim=.8in .9in .4in 1.4in, clip, width = 0.4\textwidth]{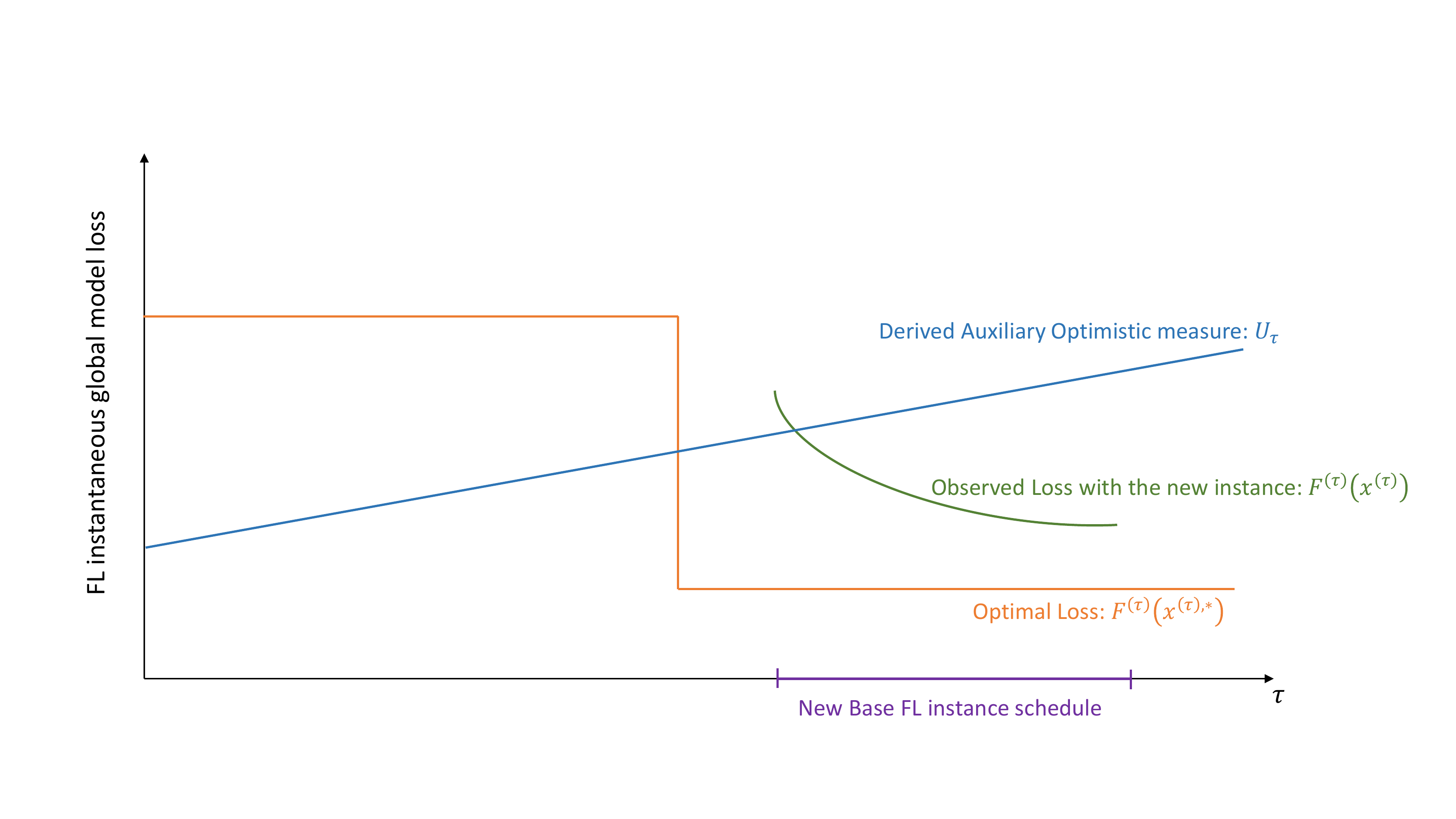}
    \caption{An illustration of how multiple base FL instances help identifying significant concept drifts.}
    \label{fig:test1_illustration}
\end{figure}
The derived optimistic quantity in Algorithm \ref{detection}: $U_t = {\max}_{\tau \in [{\color{black} t_{new}}, t]} ~\tilde{F}^{(t)}$ follows an increasing curve as depicted in Figure \ref{fig:test1_illustration}. In a stationary environment where $\Delta_{[1,t]} \rightarrow 0$, eq. \eqref{eqn: optimistic_est_1} implies that $U_{\tau}$ is always a lower bound of base FL instance's performance for every round $\tau$, which is the reason why we use the term ``optimism". Now, if a new base FL instance beats the optimistic sequence $U_{\tau}$ with an observable gap as per Figure \ref{fig:test1_illustration}, this will correspond to a significantly drifting environment. Since, we have $\text{term (b)} \leq \sum_{\tau = 1}^{t} \Big[U_{\tau} - {F}^{(\tau)}(\mathbf{x}^{(\tau), *})  \Big]$, so tracking this derived quantity sequence $\{U_{\tau}\}$ and the new base FL instance will intuitively facilitate drift detection and restarting FL learning before term (b) grows abruptly with the stale model. These mathematical intuitions in turn support our design of \textbf{Test 1} (line 15-20 in Algorithm \ref{detection}) that proxies detection of term (b) in Eq. \eqref{dyn_regr_resp10}.}\\

{\color{black} Here, we emphasize that our designed drift signaling tests significantly deviate from prior drift detection mechanisms in FL literature \cite{casado2022concept,mallick2022matchmaker}, since the proposed tests are not simply heuristic-based. In fact, the tests are designed via careful decomposition of the \textit{dynamic regret} in a non-stationary environment as suggested by the preceding discussion pertaining to mathematical intuition development for \textbf{Test 1, 2}. In Section \ref{sec: key_theory}, we unfold how the designed signaling mechanisms are tied to worst-case FL \textit{dynamic regret} bounds via comprehensive mathematical analysis. Note that another striking feature of our proposed tests is that they have a unified mathematical formulation and can be wrapped with a variety of baseline FL algorithms, for instance \textit{FedAvg, FedOMD}.}\\

{\color{black} \noindent\textbf{Random/Warm Model Initializations at restart:}} ~Finally, the  FL learning mechanism begins from scratch triggered via Algorithm \ref{MALG} if any of the aforementioned tests confirm non-stationarity. We re-iterate that this Master-FL routine is conducted directly at the server $\mathcal{S}$. {\color{black} Also, when we say that ``the base FL
instances are scheduled to start learning from scratch", it means that the previous learning schedule will now be discarded and replaced by a new learning schedule via Algorithm \ref{scheduling_algo}. In mathematical terms, the previous base FL instance collection $\bm{A}_{prev}$ will be replaced by a new collection $\bm{A}_{new}$ upon executing Algorithm 1. Also, let $X_{prev} = \{\mathbf{x}^{\mathcal{A}} : \mathcal{A} \in \bm{A}_{prev}\}$ and  $X_{new} = \{\mathbf{x}^{\mathcal{A}} : \mathcal{A} \in \bm{A}_{new}\}$ denote the ML model sets for the previous and the newly created instance sets respectively (see explicit definition of an instance in line 7 of Algorithm \ref{scheduling_algo}). Then, instances in $X_{new}$ can either be ``randomly initialized" new model vectors, or they could be populated by ``recycling" from $X_{prev}$. Our theoretical analysis presented next in Section \ref{sec: key_theory} stay unchanged with ``warm/random initializations" upon restarts.}

\section{Main Results} \label{sec: key_theory}
In this section, we will provide the dynamic regret guarantees of the proposed algorithm. In Section \ref{sec: base_FL_disc_main}, we derive the \textit{dynamic regret} bound of baseline FL algorithms: \textit{FedAvg}, \textit{FedOMD}, and discuss how the bound is impacted by different degrees of non-stationarity. In Section \ref{sec: multi_scale_near_stationary_analysis}, we show that multi-scale framework indeed preserves the properties of the base FL instances executed over arbitrary \textit{near stationary} horizons, as well as provide a concrete storage complexity bound for the proposed method.\\

{\color{black} For ease of our mathematical analysis, here we introduce the notions of ``blocks" and ``epochs" of consecutive FL training rounds. The idea of a \textit{``block"}  is motivated by the observation that in Master-FL (Algorithm \ref{detection}) training happens only in increasing chunks which are sized $2^k, ~k \in \mathbb{N}$ (please refer to line 3-6 of Algorithm \ref{detection}). More specifically, training restarts inevitably at the end of increasingly sized intervals with lengths $2^0, 2^1, \cdots$. However, if a restart trigger happened  at any round for an arbitrary order-$k$ block, then this current block is finished with less than $2^k$ rounds. Recall that a restart is only triggered via \texttt{Test-1} and/or \texttt{Test-2} in Algorithm \ref{detection}. Subsequently, Master-FL (Algorithm \ref{detection}) restarts FL training from scratch for a new order-$k+1$ block with $2^{k+1}$ rounds. Next, a mathematical definition of ``block" is presented.
\begin{definition}[Block] \label{def: block_endpoints_def_main}
Over the execution horizon of Master-FL (Algorithm \ref{detection}) for rounds $[1,T]$, $\mathcal{B} = [t_m, E_m]$ is called an order-$m$ block iff $E_m \leq t_m + 2^m -1$, and training restarts from scratch (via line 3-6 in Algorithm \ref{detection}) at rounds $t_m$ as well as $E_m + 1 \leq T$.   
\end{definition}
In the following, we unravel the meaning of an ``epoch" which we will utilize afterwards in our theoretical analysis. In contrast to a ``block", the notion of an ``epoch" is strongly associated to only restart triggers of Master-FL (Algorithm \ref{detection}).  Such restart triggers can possibly be separated across multiple successive blocks or just one single block. In our manuscript, we use ``epoch"  to denote the collection of all rounds between two consecutive restart triggers. A formal definition for an ``epoch" has been outlined next.
\begin{definition}[Epoch] \label{defn: epoch_def_main}
Over the execution horizon of Master-FL (Algorithm \ref{detection}) for rounds $[1,T]$, An interval $\mathcal{E} = [t_0, E]$ consisting of successive FL rounds is an ``epoch" if $t_0 = 1$ or a restart was triggered at $t_0 -1$. Furthermore, $E = T$ or the next restart was triggered at round $E$.
\end{definition}

In Section \ref{sec: block_regr_analysis_main}, we mathematically analyze the regret incurred by a block. Finally, in Section \ref{sec: dynamic_regr_maintext_disc}, we characterize the epoch regret and further use it to obtain the mathematical expression for the overall \textit{dynamic regret} incurred by Master-FL (Algorithm \ref{detection}) over $T$ rounds of Federated Learning.
}

\subsection{Base FL algorithm theoretical guarantees in dynamic environments} \label{sec: base_FL_disc_main}
In this subsection, we first detail the \textit{dynamic regret} analysis of vanilla \textit{FedAvg} and \textit{FedOMD} algorithms, and interpret how the results are impacted by degree of non-stationarity.

\textbf{\textit{FedAvg} \textit{dynamic regret} analysis}.  We note that the local \texttt{FL-UPDATE}($\cdot$) pertaining to \textit{FedAvg} can be written $ \forall ~n \in \mathcal{N}$ as:
\begin{align}
    \texttt{FL-UPDATE:} \hspace{3mm} \mathbf{x}_{n}^{(t)} = \mathbf{x}^{(t-1)} - \eta_t\nabla {F}_{n}^{(t)}(\mathbf{x}^{(t-1)}),  
\end{align}
where $\eta_t$ is the learning rate/step size during ML training at round $t$. Consequently, combining with Eq. \eqref{eqn: global_ML_aggr} gives the global ML model upon aggregation by server $\mathcal{S}$ as:
\begin{align}
   \mathbf{x}^{(t)} =\mathbf{x}^{(t-1)} - \eta_t\sum_{n \in \mathcal{N}} p_{n}^{(t)}\nabla {F}_{n}^{(t)}(\mathbf{x}^{(t-1)}), \label{eqn: fedavg_global_model_update_main}
\end{align}
In the following, we present the mathematical bound for \textit{dynamic regret} of \textit{FedAvg}. 
\begin{theorem}[Dynamic Regret for Convex Loss function with \textit{FedAvg}] \label{thm: fedavg_static_regr_main} Assume that the underlying ML loss measure $f(\cdot;\cdot)$ satisfies Assumption \ref{assumption: convexity_+_lipschitz} and local learning rates at the DPUs collectively represented by $\mathcal{N}$ are set to $\eta_t = \frac{1}{\sqrt{T}}$ for $t \in  \{1,2,\cdots, T\}$, the cumulative dynamic regret incurred by \textit{FedAvg} Algorithm is bounded by:
\begin{align}
    R_{[1,T]} \leq \frac{\sqrt{T}}{2} \| \mathbf{x}^{(1)} - \mathbf{x}^{*} \|^2 + \frac{{\mu}^2 \sqrt{T}}{2} + 2T\Delta_{[1, T]},
\end{align}
\end{theorem}
\begin{proof}
From the definition of \textit{dynamic regret} in \eqref{defn: dynamic regret}, we have:
\begin{align}
    R_{[1,T]} &= \sum_{t = 1}^{T} {F}^{(t)}(\mathbf{x}^{(t)}) -  \sum_{t = 1}^{T} {F}^{(t)}(\mathbf{x}^{(t),*}), \label{eqn: fedavg_dreg_def2_main} \\
    & = \underbrace{\sum_{t = 1}^{T} {F}^{(t)}(\mathbf{x}^{(t)}) - \sum_{t = 1}^{T} {F}^{(t)}(\mathbf{x}^{*})}_\text{(a)} \nonumber \\
    & + \underbrace{\sum_{t = 1}^{T} {F}^{(t)}(\mathbf{x}^{*}) - \sum_{t = 1}^{T} {F}^{(t)}(\mathbf{x}^{(t),*})}_\text{(b)},   \label{eqn: fedavg_dreg1_main}
\end{align}
where, we define the \textit{static comparator} $\mathbf{x}^{*}$ as follows:
\begin{align}
    \mathbf{x}^{*} = \underset{\mathbf{x}}{\min} ~\sum_{t=1}^{T} F^{(t)}(\mathbf{x}). \label{eqn: static_comp_def_main}
\end{align}
Next, we individually bounding terms (a) and (b) in Eq. \eqref{eqn: fedavg_dreg1_main}, and the detailed steps are given in Appendix \ref{apd: fedavg_static_regr_main}. 
\end{proof}
\textbf{\textit{FedOMD} \textit{dynamic regret} analysis}. In the following, we first explain the \texttt{FL-UPDATE}($\cdot$) rule at the DPUs in $\mathcal{N}$ for \textit{FedOMD} algorithm. In this context, we summarize the mathematical details around the notion of Bregman Divergence. Specifically, consider an arbitrary 1-strongly convex function $\phi : \mathbb{R}^d \rightarrow \mathbb{R}$ w.r.t. L2 euclidean norm i.e., $\|\cdot\|$. Therefore, due to strong convexity of $\phi(\cdot)$, the following holds:
\begin{align}
    \phi(\mathbf{y}) \geq \phi(\mathbf{x}) + \langle \mathbf{y}-\mathbf{x}, \nabla{\phi(\mathbf{x})}  \rangle + \frac{1}{2} \| \mathbf{y} - \mathbf{x} \|^2, ~\forall \mathbf{x}, \mathbf{y} \in \mathbb{R}^d. \label{eqn: phi_strongly_convex_defn_main}
\end{align}
Consequently, the Bregman Divergence w.r.t $\phi(\cdot)$ is defined as:
\begin{align}
    B_{\phi}(\mathbf{y};\mathbf{x}) \triangleq \phi(\mathbf{y}) -\phi(\mathbf{x}) - \langle \mathbf{y} - \mathbf{x}, \nabla{\phi}(\mathbf{x}) \rangle.
\end{align}
Furthermore, the Bregman Divergence $B_{\phi}$ is chosen such that it satisfies the assumption stated next.
\begin{assumption} \label{assumption: bregman_weighted_main}
For any collection of arbitrary points $\mathbf{z}_1, \mathbf{z}_2, \cdots, \mathbf{z}_m \in \mathbbm{R}^{d}$, with scalar weights $w_1, \cdots, w_j \in [0,1]$ such that $\sum_{i = 1}^j w_i = 1$, the following holds for all $u \in \mathbbm{R}^{d}$:
\begin{align}
    B_{\phi}(u,\sum_{i = 1}^{j} w_i z_i) \leq \sum_{i = 1}^{j} w_i B_{\phi}(u, z_i).
\end{align}
\end{assumption}
It is worth highlighting that Assumption \ref{assumption: bregman_weighted_main} is in fact true for commonly used Bregman divergences such as Euclidean Distance and KL-Divergence. Furthermore, we note that this assumption is only required to achieve sub-linear convergence speeds for \textit{FedOMD} algorithm \cite{fedomdpaper}, and not a general requirement for our multi-scale algorithmic framework.\\

The local model update, i.e., \texttt{FL-UPDATE}($\cdot$) during each Federated Learning round where \textit{FedOMD} Algorithm is executed can be described as:
\begin{align}
    & \texttt{FL-UPDATE:} \hspace{5mm} \mathbf{x}^{(t+1)}_{n} = \underset{x \in \mathbb{R}^p}{\argmin} ~\psi_{n}(\mathbf{x}; \mathbf{x}^{(t)}), \label{eqn:bregman_minimize_func_main} \\
    & \hspace{5mm} \psi_{n}(\mathbf{x}; \mathbf{x}^{(t)}) \triangleq \langle \nabla {F}_{n}^{(t)}(\mathbf{x}^{(t)}), \mathbf{x} \rangle + \frac{1}{\eta_t} B_{\phi} (\mathbf{x}; \mathbf{x}^{(t)}). \label{eqn:bregman_loss_func_defn_main}
\end{align}
We restate the global ML model aggregation procedure as summarized in Eq. \eqref{eqn: global_ML_aggr} in the following: 
\begin{align} 
    \mathbf{x}^{(t+1)} = \sum_{n \in \mathcal{N}} p_{n}^{(t)} \mathbf{x}^{(t+1)}_{n}. \label{eqn: global_ML_aggr_fedomd_main}
\end{align}

\begin{theorem}[Dynamic Regret with Convex Loss function for \textit{FedOMD}] \label{thm: fedomd_convex_regr_main} Assume that the underlying ML loss measure $f(\cdot;\cdot)$ satisfies Assumption \ref{assumption: convexity_+_lipschitz} and local learning rates at the DPUs collectively represented by $\mathcal{N}$ are set to $\eta_t = \frac{1}{\sqrt{T}}$ for $t \in  \{1,2, \cdots, T \}$, the cumulative dynamic regret incurred by \textit{FedOMD} Algorithm is bounded by:
\begin{align}
    R_{[1,T]} \leq \sqrt{T} B_{\phi}(\mathbf{x}^{*},\mathbf{x}^{(1)}) + \frac{\mu^2}{2}\sqrt{T} + 2T\Delta_{[1,T]}.
\end{align}
\end{theorem}
\begin{proof}
Please refer to Appendix \ref{app: base_alg_guarantees}.
\end{proof}
In a nutshell, proving Theorem \ref{thm: fedomd_convex_regr_main} uses the same \textit{dynamic regret} decomposition technique as done in Theorem \ref{thm: fedavg_static_regr_main} via Eq. \eqref{eqn: fedavg_dreg_def2_main} - \eqref{eqn: fedavg_dreg1_main}. Consequently, the static and dynamic components are individually bounded and combined to get the final \textit{dynamic regret} bound for \textit{FedOMD}.

\textbf{Interpretation of \textit{regret} results over varying degree of drifts.}  The \textit{dynamic regret} results obtained in Theorems \ref{thm: fedavg_static_regr_main} and \ref{thm: fedomd_convex_regr_main} corroborate a crucial intuition that supports our algorithm design. More specifically, it reflects that the aforementioned conventional FL methods may work well in dynamic environments which are \textit{near stationary}. To see this mathematically, we appeal to the notion of \textit{near stationarity} introduced in Assumption \ref{assmptn: near stationary} which implies : $\Delta_{[1,T]} \leq \rho(T)$. \\

{\color{black}
\begin{remark}
In development and analysis of our algorithmic framework, we have considered full-batch gradient computations for base FL algorithms \textit{FedAvg, FedOMD}. However, it is worth highlighting that in a mini-batch stochastic gradient computation setting, the statistical properties of the proxy gradient resembles that of the true gradient. Consequently, our algorithmic framework is directly extendable to a stochastic update approach, we present a proof sketch with mathematical details and explanations in Appendix \ref{app: mini-batch-explanation}.
\end{remark}
}

{ \color{black} \begin{remark} \label{remark: linfactoring} Although the performance guarantees of the vanilla algorithms don't change at such small drifts, however higher drifts cause the bounds to get worse (from $\Tilde{\mathcal{O}}(\sqrt{T})$ in small drift scenarios to $\Tilde{\Omega}(\sqrt{T})$ in presence of high drifts) due to \textit{linear factoring} of the drift terms with horizon length $T$. 
\end{remark}
In order to mathematically delineate the claim of Remark \ref{remark: linfactoring}, we first note that the terms $\Delta_{[1,T]} T$ in the RHS of regret expressions presented via Theorems \ref{thm: fedavg_static_regr_main}, \ref{thm: fedomd_convex_regr_main} reflect this aforementioned \textit{linear factoring} effect. To appreciate how exactly theoretical performance of vanilla algorithms deteriorates in high drift scenarios, consider the violation of the \textit{near-stationarity} condition stated in Assumption \ref{assmptn: near stationary} i.e., $\Delta_{[1,T]} > \rho(t) \geq \frac{1}{\sqrt{T}}$. Consequently, the bounds of vanilla \textit{FedAvg}, \textit{FedOMD} would become at least $\Tilde{\Omega}(\sqrt{T})$ in contrast to the \textit{near-stationarity} scenario where the bounds would be at worst $\Tilde{\mathcal{O}}(\sqrt{T})$, as suggested by RHS of regret bounds presented in Theorems \ref{thm: fedavg_static_regr_main}, \ref{thm: fedomd_convex_regr_main}.  }

To mitigate this, our multi-scale algorithmic framework leverages the \textit{near stationarity} guarantees of the base algorithm by augmenting it with carefully curated non-stationarity tests (\textbf{Test 1} and \textbf{Test 2} in Algorithm  \ref{detection}). We outline our theoretical findings of how better regrets could be achieved at higher degrees of drifts in the subsequent discussions.   
\subsection{Multi-Scale Algorithm Analysis} \label{sec: multi_scale_near_stationary_analysis}
The Multi-Scale FL Runner (Algorithm \ref{MALG}) is essentially a subroutine within Master-FL (Algorithm \ref{detection}) that orchestrates different scheduled base FL instances over a specified block of aggregation rounds, and produces the sequence of optimistic loss quantities $\{\Tilde{F}^{(t)} \}$. {\color{black} We redirect the readers to Appendix \ref{sec: optimistic_estimator_verification} wherein we outline the construction of the sequence $\{\Tilde{F}^{(t)} \}$ and validate Requirement \ref{assmptn: near stationary}  for \textit{FedAvg, FedOMD}.} 
\begin{lemma} \label{lemma:multi_scale_regr_main} Let $\hat{m} = \log_2 T + 1$ and Multi-Scale FL Runner (Algorithm \ref{MALG}) is executed with input $m \leq \log_2 T$. Furthermore, we assume that with any instance of base FL algorithm $\mathcal{A}$ initiated within  
Algorithm \ref{MALG} and any $t \in [\mathcal{A}.s, \mathcal{A}.e]$, the cumulative concept drift satisfies $\Delta_{[\mathcal{A}.s,t]} \leq \rho(t')$ where $t' = t - \mathcal{A}.s +1$. Then, with probability $1-\frac{\delta}{T}$, the following holds:
\begin{align}
    & \Tilde{F}^{(t)} \leq \underset{\tau \in [\mathcal{A}.s,t]}{\max} F^{(\tau)}(\mathbf{x}^{(\tau),*}) + \Delta_{[\mathcal{A}.s,t]}, \label{eqn: opt_bound_eq1_main}\\
    & \frac{1}{t'} \sum_{\tau = \mathcal{A}.s}^{t} F^{(\tau)}(\mathbf{x}^{(\tau)}) - \Tilde{F}^{(\tau)} \leq \hat{\rho}(t') + \hat{m}\Delta_{[\mathcal{A}.s,t]}, \label{eqn: opt_bound_eq2_main}
\end{align}
and, the number of instances running within the interval $[\mathcal{A}.s,t]$ is bounded by $6\hat{m}\log (T/\delta) \frac{C(t')}{C(1)}$, where $C(t)$ is as described in Definition \ref{defn: rho_def}.
\end{lemma}
\begin{proof}
Please refer to Appendix \ref{sec: multiscale_analysis}.
\end{proof}
\textbf{Interpretation of Lemma \ref{lemma:multi_scale_regr_main} results.} The first part of the Lemma ensures that the behavior of the baseline FL methods, i.e., \textit{FedAvg and FedOMD} remain unchanged even under the proposed multi-scale orchestration scheme. Formally, it proves a more general version of the requirements specified in Assumption \ref{assmptn: near stationary} for arbitrary time intervals, under the assumption that they are \textit{near-stationary}. Furthermore, it provides a storage complexity analysis in terms of worst case bound on the number of base FL instances scheduled as $\tilde{\mathcal{O}}(C(T))$ over a horizon of $T$ FL rounds. 

\subsection{Block Regret Analysis} \label{sec: block_regr_analysis_main}
{\color{black} In this subsection, we summarize the \textit{dynamic regret} for any such arbitrary block of order $m$ (see Definition \ref{def: block_endpoints_def_main}). More specifically, we consider this block scheduled to run for the rounds $[t_m, t_m + 2^m - 1]$. In proving the block \textit{dynamic regret} result of Lemma \ref{eqn: blk_dyn_regr_main} stated later, we basically divide $[t_m, t_m + 2^m - 1]$ into successive intervals $\mathcal{I}_1 = [s_1, e_1]$, $\mathcal{I}_2 = [s_2, e_2]$, $\cdots$, $\mathcal{I}_K = [s_K, e_K]$ ($s_1 = t_m, e_i + 1 = s_{i+1}, e_K = t_m + 2^m -1$). Furthermore, all these intervals are assumed to satisfy:}
\begin{align}
    \Delta_{\mathcal{I}_i} \leq \rho(|\mathcal{I}_i|), ~\forall ~i. \label{eqn: delta_rho_reln_1_main}
\end{align}
For our current set of baseline algorithms: \textit{FedAvg} and \textit{FedOMD}, according to Theorem \ref{thm: fedavg_static_regr_main}, \ref{thm: fedomd_convex_regr_main} we can have $C(t) = t\rho(t) = \min\{c_1\sqrt{t} + c_2, t \}$ since the losses are bounded in $[0,1]$.
\begin{lemma}[Block Dynamic regret] \label{eqn: blk_dyn_regr_main}
{\color{black} Let $\mathcal{B} = [t_m, E_m]$ be an order-$m$ block (see Definition \ref{def: block_endpoints_def_main}) for which {Master-FL} (Algorithm \ref{detection}) is executed.} Then, the dynamic regret incurred on $\mathcal{B}$ is bounded as:
\begin{align}
   {R}_{\mathcal{B}} \leq \Tilde{\mathcal{O}} \Big( \min \Big\{ {R}_{L}(\mathcal{B}), {R}_{\Delta}(\mathcal{B}) \Big \} + \Big( c_1 + \frac{c_2}{c_1} \Big) 2^{m/2} + c_{2}^{2} \Big),
\end{align}
where, ${R}_{L}(\mathcal{B}), {R}_{\Delta}(\mathcal{B})$ are defined as follows:
\begin{align}
    & {R}_{L}(\mathcal{B}) \triangleq c_1 \sqrt{{L} |\mathcal{B}|} + c_2 {L}, \\
    & {R}_{\Delta}(\mathcal{B}) \triangleq c_{1}^{\frac{2}{3}} \Delta^{\frac{1}{3}} {|\mathcal{B}|}^{\frac{2}{3}} + c_1 \sqrt{|\mathcal{B}|} + c_{1} \sqrt{\Delta |\mathcal{B}|} \nonumber \\
    & \hspace{1.3cm} + {c_2}c_1^{-\frac{2}{3}}\Delta^{\frac{2}{3}}T^{\frac{1}{3}} + c_2(1 + \Delta) ~,
\end{align}
and, where $L, \Delta$ are as described in Definition \ref{defn: model drift} , \ref{defn: num_drifts}.
\end{lemma}
\begin{proof}
Please refer to Appendix \ref{Sec: block_dyn_regr_results}.
\end{proof}
\textbf{Interpretation of Lemma \ref{eqn: blk_dyn_regr_main} results.} The results presented in this lemma offer a key insight regarding how Master-FL behavior on each small block of FL rounds. Roughly speaking, it leverages the fact that every arbitrary block of order $m$ contains smaller intervals which are \textit{near-stationary}. Then, it achieves the stated result by combining the result in Lemma  \ref{lemma:multi_scale_regr_main} and an upper bound on how many such smaller \textit{near-stationary} segments may exist in a given block. Furthermore, it expresses the regret as the smaller of the quantities ${R}_{L}(\mathcal{B}) = \tilde{\mathcal{O}}(\sqrt{{L} |\mathcal{B}|} + {L})$ and ${R}_{\Delta}(\mathcal{B}) = \tilde{\mathcal{O}}(\Delta^{\frac{1}{3}} {|\mathcal{B}|}^{\frac{2}{3}} + \sqrt{|\mathcal{B}|})$ for any arbitrary $\mathcal{B}$. In Section \ref{sec: dynamic_regr_maintext_disc}, we show how this result translates to a concrete mathematical bound for the overall \textit{dynamic regret} of {Master-FL}.
\subsection{Dynamic Regret Analysis of {Master-FL}} \label{sec: dynamic_regr_maintext_disc}
{\color{black} In this section, we first present the \textit{dynamic regret} result pertaining to a single epoch of FL rounds for which Master-FL (Algorithm \ref{detection}) conducts training.  We reiterate that an epoch is essentially an interval between two consecutive restart triggers.} 

\begin{lemma}[Single Epoch Regret Analysis] \label{lemma: single_epoch_regr_main} Let, $\mathcal{E} = [t_0, E]$ be an epoch containing blocks upto order $m$ for which {Master-FL} (Algorithm \ref{detection}) runs. Then, the dynamic regret associated with $\mathcal{E}$ is:  
\begin{align}
    {R}_{\mathcal{E}} \leq \Tilde{\mathcal{O}} \Big(\min \Big\{ {R}_{L}(\mathcal{E}), {R}_{\Delta}(\mathcal{E}) \Big \} + (c_1+ \frac{c_2}{c_1}) \sqrt{|\mathcal{E}|} + c_2^2 \Big).
\end{align}
where ${R}_{L}(\cdot)$, ${R}_{\Delta}(\cdot)$ are as defined in Lemma \ref{eqn: blk_dyn_regr_main}.
\end{lemma}
\begin{proof}
Please refer to Appendix \ref{lemma: single_epoch_regr_discussion}.
\end{proof}
\begin{figure*}[ht]
\captionsetup[subfigure]{justification=Centering}
\begin{subfigure}[t]{0.24\textwidth}
\centering
    \includegraphics[scale = 0.25]{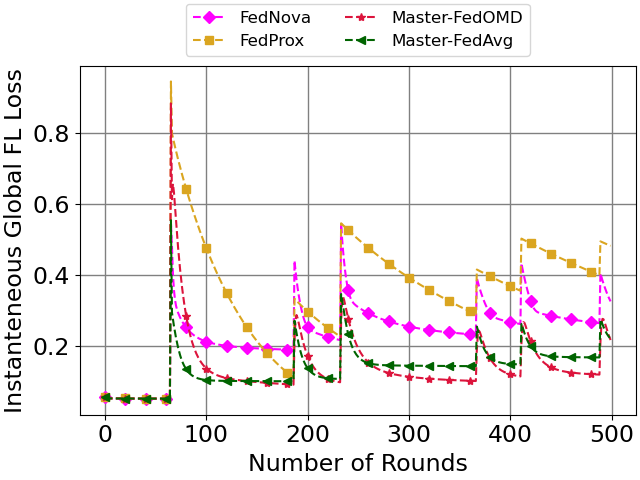}
    \caption{\color{black} Covtype (CI drift).}
\end{subfigure}\hspace{\fill}
\begin{subfigure}[t]{0.24\textwidth}
\centering
    \includegraphics[scale = 0.25]{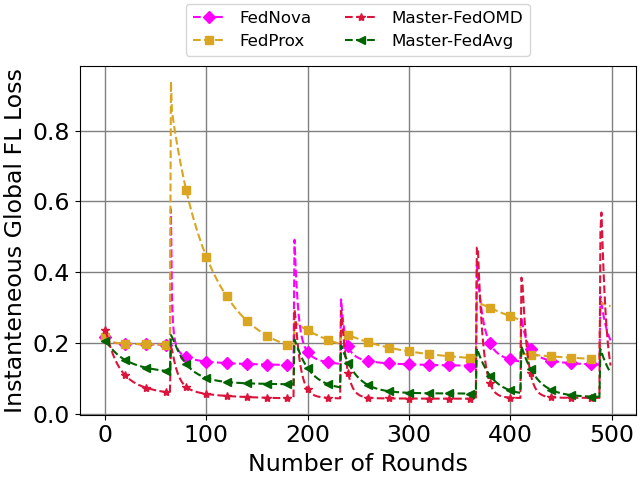}
    \caption{\color{black} Covtype (CS drift).}
\end{subfigure}
\begin{subfigure}[t]{0.24\textwidth}
    \centering
    \includegraphics[scale = 0.25]{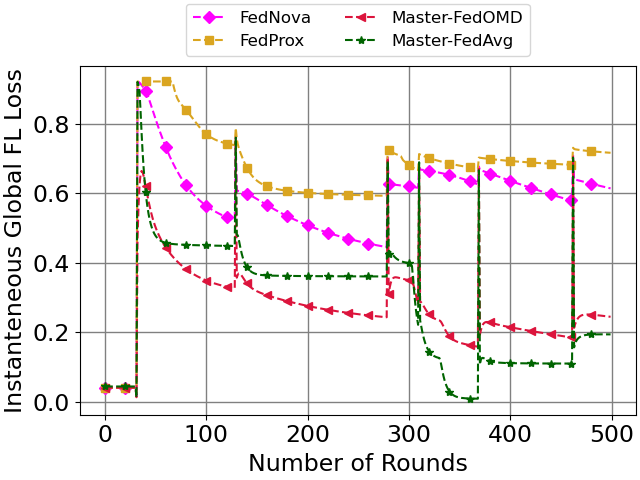}
    \caption{\color{black} MNIST (CI drift).}
\end{subfigure}\hspace{\fill}
\begin{subfigure}[t]{0.24\textwidth}
 \centering
    \includegraphics[scale = 0.25]{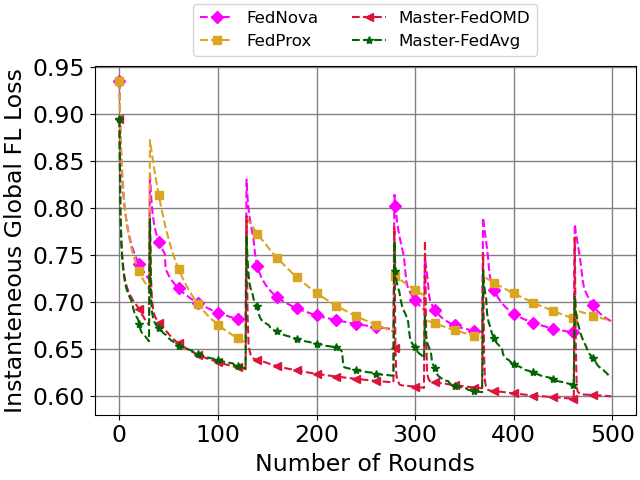}
    \caption{\color{black} MNIST (CS drift).}
\end{subfigure} %

\caption{\color{black} Instantaneous global FL loss vs aggregation rounds for various methods. CS ${\color{red} \rightarrow}$ Class Swap Drift, CI ${\color{red} \rightarrow}$ Class Introduction Drift.}
\label{fig: loss_plots}
\end{figure*}
\begin{lemma}[Statistical consistency of stationary test trigger events] \label{lemma: test_trigger_correctness_main} Let $t$ be a FL round within an epoch starting from $t_0$. If $\Delta_{[t_0, t]} \leq \rho(t - t_0 + 1)$, then with high probability, no restart is triggered during round $t$.
\end{lemma}
\begin{proof}
Please refer to Appendix \ref{sec: test_trigger_correctness_desc}.
\end{proof}

\begin{lemma}[Bound for number of epochs $M$ over the entire horizon] \label{lemma: bound_epoch_num_main}
 For total number of epochs denoted by $M$ over the horizon of length $T$ on which {Master-FL} is executed, the following results hold with high probability:
\begin{align}
& M \leq L, \label{eqn: master_N_L_bound_main}\\
& M \leq 1 + 2c_1^{-\frac{2}{3}} \Delta^{\frac{2}{3}}T^{\frac{1}{3}} + \Delta. \label{eqn: master_N_Delta_bound_main}
\end{align}
\end{lemma}
\begin{proof}
Please refer to Appendix \ref{sec: bound_epoch_num_main_desc}.
\end{proof}
\begin{theorem}[Dynamic Regret of {Master-FL} Algorithm]  \label{thm: final_dynamic_regr_master_main} Assume that individual base FL instances satisfy the conditions specified in Assumption \ref{assmptn: near stationary}, then without the knowledge of non-stationary measures $L$ or $\Delta$, {Master-FL} (Algorithm \ref{detection}) incurs dynamic regret over $T$ FL aggregation rounds which is bounded by:
\begin{align}
    \mathcal{R}_{[1,T]} = \Tilde{\mathcal{O}} \Big( \min \Big\{ \Big(c_1 + \frac{c_2}{c_1} \Big)\sqrt{LT} , c_2 c_1^{- \frac{4}{3}}\Delta^{\frac{1}{3}}T^{\frac{2}{3}} + \sqrt{T} \Big\} \Big). 
\end{align}
with high probability.
\end{theorem}
\begin{proof}
Please refer to Appendix \ref{sec: final_dynamic_regr_master}.
\end{proof}

\textbf{Interpretation of Lemma  \ref{lemma: single_epoch_regr_main} - \ref{lemma: bound_epoch_num_main}, Theorem \ref{thm: final_dynamic_regr_master_main} results.} At a high level, Lemma \ref{lemma: single_epoch_regr_main} verifies that the block \textit{dynamic regret} result directly translates to an epoch with multiple blocks as long as the change-detection/restart events are correctly triggered via \textbf{Test 1}, \textbf{Test 2} in Master-FL. Therefore, in the next step, we formally verify that restarts are not triggered as long as the blocks of FL rounds are \textit{near-stationary} in Lemma \ref{lemma: test_trigger_correctness_main}. With the exact mathematical bound in place for regret incurred in each epoch and the proof of correctness of restart triggers, we move to characterizing the upper limit on number of epochs in terms of horizon length and non-stationary measures, i.e., $T, L$, and $\Delta$ respectively in Lemma \ref{lemma: bound_epoch_num_main}. Finally, we combine the bounds of each individual epoch \textit{dynamic regret} and number of epochs, to obtain the final regret bound incurred by Master-FL in Theorem \ref{thm: final_dynamic_regr_master_main}.

{\color{black} \section{Experimental Evaluations} \label{Sec: experiments}

In this section, we perform proof-of-concept experiments comparing Master-FL-\textit{FedAvg} and Master-FL-\textit{FedOMD} with the closest competing algorithms \textit{FedNova} \cite{wang2020tackling} and vanilla \textit{FedProx} \cite{li2018federated} on 2 LIBSVM \footnote{LIBSVM Collection of dataset  can be accessed at \url{https://www.csie.ntu.edu.tw/~cjlin/libsvmtools/datasets/}} classification datasets: { covtype, mnist}. \textit{FedNova} and \textit{FedProx} are widely-used benchmark FL methods that mitigate drifts collected in the global objective via careful modifications in their optimization frameworks. The aforementioned datasets are extensively leveraged for experimental studies pertaining to FL and distributed ML architectures \cite{kovalev2021lower, canonaco2021adaptive}, as well as can be conveniently used to simulate time-varying non-stationary scenarios.

\begin{table}[t]
\centering
{
\color{black}
\begin{tabularx}{.45 \textwidth}{c *{4}{Y}}
\toprule[.05em]
\multirow{0}{*}{} & \multicolumn{2}{c}{\bf{covtype}} & \multicolumn{2}{c}{\bf{mnist}} \\
\cmidrule(lr){2-3} \cmidrule(lr){4-5}   %
& CI & CS & CI & CS   \\
\midrule
Master-FL-\textit{FedAvg} & 0.874 & 0.852  & 0.719 & 0.739  \\
Master-FL-\textit{FedOMD} & 0.876 & 0.864  & 0.682 & 0.780  \\
\textit{FedNova}          & 0.799 & 0.753  & 0.629 & 0.644  \\
\textit{FedProx}          & 0.673 & 0.686  & 0.577 & 0.601  \\
\bottomrule
\end{tabularx}

}
\caption{\centering \color{black} Average FL classification accuracy for various methods across total $T= 500$ rounds. CS ${\color{red} \rightarrow}$ Class Swap Drift, CI ${\color{red} \rightarrow}$ Class Introduction Drift.} 
\label{tab:ML_loss_table}
\end{table}

For our experiments, we consider L1 regularized multi-class logistic regression formulation for the underlying ML loss function, i.e., $f(\cdot; \xi)$ for datapoint $\xi$. More precisely, for a model $\mathbf{x}$ and datapoint $\xi$, we have:
\begin{align}
    f(\mathbf{x}; \xi) = \log \big(1 + exp(-\xi_{(2)}.\mathbf{x}^{T} \xi_{(1)})\big) + \frac{\lambda}{2} \|x\|_{2}, \label{eq: experiment_loss_func}
\end{align}
wherein every datapoint is defined as $\xi \triangleq (\xi_{(1)}, \xi_{(2)})$, i.e., tuple of the feature vector $\xi_{(1)}$ and class label $\xi_{(2)}$ respectively. We use grid-search for tuning regularization parameter\footnote{Note that \textit{FedNova} and \textit{FedProx} use L2 regularization instead of L1 in Eq. \eqref{eq: experiment_loss_func}} and presented results are for $\lambda = 2e-4$. Additionally, the learning rates for the different algorithms, $\eta$, is chosen as  $\eta = 1/\sqrt{T}$ for $T$ rounds. The number of client DPUs are $|\mathcal{N}| = 20 $, and we conduct FL training upto $T = 500$ rounds. During every FL training round $t$, each DPU $n \in \mathcal{N}$ acquires its local dataset $\mathcal{D}_{n}^{(t)}$ with sizes distributed as $\mathcal{N}(1000, 200)$ via random sampling without replacement from the core data-source. 

We consider two different categories of \textit{concept drift} that are particularly aligned with our classification tasks: \textit{class introduction (CI) drift}  and \textit{class swap (CS) drift}  \cite{canonaco2021adaptive}. In CI experiments, new classes are injected into the DPUs at certain rounds, thereby adding a non-stationary shock to the system. For ``covtype" dataset, we add CI drift for the rounds $t \in \{65, 187, 233, 367, 411, 489\}$ wherein the classes are added sequentially in the order $\{0\} \rightarrow \{1\} \rightarrow \{2\} \rightarrow \{3\} \rightarrow \{4\} \rightarrow \{5\} \rightarrow \{6\}$ starting from $t=0$ and at each drift round in the set. On the other hand, for ``mnist" dataset, CI drift is injected at rounds $t \in \{31, 129, 279, 310, 369, 462\}$ wherein the classes are added in the order $\{0,1\} \rightarrow \{2,3\} \rightarrow \{4\} \rightarrow \{5,6\} \rightarrow \{7\} \rightarrow \{8\} \rightarrow \{9\}$. In CS experiments, labels are swapped for pair(s) of classes. We follow the same schedule for injecting shock into the system. At each drift round, we used 3 pairs of classes for swapping, i.e., $\{(0,1), (2,3), (4,5)\}$.

 Figure \ref{fig: loss_plots} demonstrates how the performance of the FL methods are impacted by drift shocks. Observe that global losses spike approximately around the drift injection rounds, signifying unexpected non-stationarity experienced by the network of DPUs. Consequently, note that Master-FL-\textit{FedAvg}, Master-FL-\textit{FedOMD} outperform competing FL methods for all the dataset-drift combinations in terms of instanteneous FL loss, and this is also corroborated by the cumulative average FL classification accuracies presented in Table \ref{tab:ML_loss_table}. The results obtained supports our framework's ability to quickly re-train foregoing models that were impacted by \textit{concept drift}.

}

\section{Conclusions} \label{Sec: conclusion}
In this paper, we propose Master-FL a multi-scale change detection-restart based algorithmic framework that can leverage any baseline FL optimizer that works well in  \textit{near-stationary} environments to adaptively learn in a highly drifting environment. We derive online \textit{dynamic regret} bounds for vanilla \textit{FedAvg} and \textit{FedOMD}, as well as demonstrate their \textit{near-stationary} properties which are essential for augmentation with Master-FL. Subsequently, we provide rigorous mathematical analysis for convex loss functions leading to \textit{dynamic regret} bounds in terms of non-stationary measures $\Delta, L$ with no prior knowledge requirement or stronger convexity assumptions.  To the best of our knowledge, the \textit{dynamic regret} bounds revealed in this work are novel in the non-stationary federated optimization setting, while it extends the existing literature in the general centralized as well as distributed online optimization paradigms. As a future direction, extending to other FL baseline algorithms which may demonstrate favorable performance under stationary assumptions definitely deserves consideration. %

\bibliographystyle{IEEEtran}
\bibliography{references}
\newpage
    \setcounter{page}{1}

\onecolumn
\appendices

\section{Description of Non-Stationary Measures $C_T, D_T$} \label{app: other_drift_measures_defn}
In this Section, we present the definitions of non-stationary measures $C_T, D_T$ as provided in \cite{jadbabaie2015online}.
\begin{definition} 
The \textit{comparator regularity} measure $C_{T}$ for an $T$ length sequence of minimizers $\{\mathbf{x}^{(1)}, \mathbf{x}^{(2)}, \cdots, \mathbf{x}^{(T)}\}$ is defined as:
\begin{align}
    C_{T} \triangleq \sum_{t=1}^{T} \|\mathbf{x}^{(t)} - \mathbf{x}^{(t-1)}\|.
\end{align}
\end{definition}
\begin{definition}
The gradient variability measure $D_T$ for an $T$ length sequence of minimizers $\{\mathbf{x}^{(1)}, \mathbf{x}^{(2)}, \cdots, \mathbf{x}^{(T)}\}$ over an arbitrary sequence of differentiable loss functions $\{g^{(1)}, g^{(2)}, \ldots, g^{(T)}\}$ is defined as:
\begin{align}
    D_T \triangleq \sum_{t=1}^{T} \|\nabla{g}^{(t)}(\mathbf{x}^{(t)}) - M^{(t)} \|^{2},
\end{align}
where $\{M^{(1)}, M^{(2)}, \cdots, M^{(T)}\}$ is a predictable sequence computed by the learner during each training round $t$.
\end{definition}
\vspace{3cm}
{\color{black} \section{Outline of Master-FL extension to model differential based syncronization setups}
\label{app: model_differential_revision}
To see how our framework can be leveraged with FL setups equipped with model syncronization via differentials, first note that message passing for ML model at each training round $t$ between the client DPUs and the server in our proposed method currently involves the following two steps (eq. \eqref{eqn: FL-update} and \eqref{eqn: global_ML_aggr_app} in Section \ref{sec: problem_formulation} of manuscript):
\begin{align}
    \label{eqn: FL-update_app} \mathbf{x}_n^{(t)} = \texttt{FL-UPDATE}\big(\mathbf{x}^{(t-1)}, \nabla {F}_{n}^{(t)}(\mathbf{x}^{(t-1)}) \big).
\end{align}
\begin{align} 
    \mathbf{x}^{(t)} = \sum_{n \in \mathcal{N}} p_{n}^{(t)} \mathbf{x}^{(t)}_{n}. \label{eqn: global_ML_aggr_app}
\end{align}
Also, we have $\sum_{n \in \mathcal{N}} p_{n}^{(t)} = 1, ~\forall t \in [T]$, i.e., fractional dataset sizes add upto 1 across all the DPUs. Hence, eq. \eqref{eqn: global_ML_aggr}, can be alternatively expressed as:
\begin{align}
    \mathbf{x}^{(t)} - \mathbf{x}^{(t-1)}=  \sum_{n \in \mathcal{N}} p_{n}^{(t)} (\mathbf{x}^{(t)}_{n} - \mathbf{x}^{(t-1)}), \label{eqn: global_ML_aggr} 
\end{align}
Eq. \eqref{eqn: global_ML_aggr_app} implies that given the global model at round $t-1$ i.e., $\mathbf{x}^{(t-1)}$ is known to the network, the DPUs can communicate local model differentials i.e., $\{\mathbf{x}^{(t)}_{n} - \mathbf{x}^{(t-1)} \}_{n \in \mathcal{N}}$ instead of locally updated model itself i.e., $\{\mathbf{x}^{(t)}_{n}\}_{n \in \mathcal{N}}$ to the central server $\mathcal{S}$. In the next step, the central server will compute global model differential for round $t$ i.e., $\mathbf{x}^{(t)} - \mathbf{x}^{(t-1)}$ according to eq. \eqref{eqn: global_ML_aggr_app} and broadcast it over the network of DPUs. Hence, at the end of synchronization step during round $t$ via model differential message passing, $\mathbf{x}^{(t)}$ can still be computed at all the nodes in $\mathcal{N} \cup \mathcal{S}$.  Furthermore, it is critical to stress the fact the transmission of local dataset sizes to the central server from the DPUs are strictly unavoidable for calculation of weights $\{p^{(t)}_{n}\}_{n \in \mathcal{N}}$, however they simply are scalars and incur significantly much less communication cost compared to the ML model. Hence, with the aforementioned modifications, base FL algorithms which may use communication efficient model differential syncronization techniques can be directly integrated with our methodology.}

\clearpage
\section{Dynamic Regret Analysis for FedAvg Algorithm  - Proof of Theorem \ref{thm: fedavg_static_regr_main}}\label{apd: fedavg_static_regr_main}

From the definition of \textit{dynamic regret} in \eqref{defn: dynamic regret}, we have:
\begin{align}
    R_{[1,T]} &= \sum_{t = 1}^{T} {F}^{(t)}(\mathbf{x}^{(t)}) -  \sum_{t = 1}^{T} {F}^{(t)}(\mathbf{x}^{(t),*}), \label{eqn: fedavg_dreg_def2_main} \\
    & = \underbrace{\sum_{t = 1}^{T} {F}^{(t)}(\mathbf{x}^{(t)}) - \sum_{t = 1}^{T} {F}^{(t)}(\mathbf{x}^{*})}_\text{(a)} \nonumber \\
    & + \underbrace{\sum_{t = 1}^{T} {F}^{(t)}(\mathbf{x}^{*}) - \sum_{t = 1}^{T} {F}^{(t)}(\mathbf{x}^{(t),*})}_\text{(b)},   \label{eqn: fedavg_dreg1_main}
\end{align}
where, we define the \textit{static comparator} $\mathbf{x}^{*}$ as follows:
\begin{align}
    \mathbf{x}^{*} = \underset{\mathbf{x}}{\min} ~\sum_{t=1}^{T} F^{(t)}(\mathbf{x}). \label{eqn: static_comp_def_main}
\end{align}
Henceforth, we focus on individually bounding terms (a) and (b) in Eq. \eqref{eqn: fedavg_dreg1_main}. In order to bound term (a), we proceed as follows:
\begin{align}
    \|\mathbf{x}^{(t + 1)} &- \mathbf{x}^{*}\|^2 = \|\mathbf{x}^{(t)} - \mathbf{x}^{*} - \eta_t\sum_{n \in \mathcal{N}} p_{n}^{(t)}\nabla {F}_{n}^{(t)}(\mathbf{x}^{(t)}) \|^2, \label{eqn: fedavg_term_a_1_main} \\
    & = \| \mathbf{x}^{(t)} - \mathbf{x}^{*} \|^2 + \eta_t^2\|\sum_{n \in \mathcal{N}} p_{n}^{(t)}\nabla {F}_{n}^{(t)}(\mathbf{x}^{(t)}) \|^2 \nonumber \\
    & -2\eta_t \sum_{n \in \mathcal{N}} p_{n}^{(t)} \langle \nabla {F}_{n}^{(t)}(\mathbf{x}^{(t)}), \mathbf{x}^{(t)} - \mathbf{x}^{*} \rangle \\
    & \leq  \| \mathbf{x}^{(t)} - \mathbf{x}^{*} \|^2 + \sum_{n \in \mathcal{N}} p_{n}^{(t)} \eta_t^2 \|\nabla {F}_{n}^{(t)}(\mathbf{x}^{(t)}) \|^2 \nonumber \\
    & -2\eta_t \sum_{n \in \mathcal{N}} p_{n}^{(t)} \langle \nabla {F}_{n}^{(t)}(\mathbf{x}^{(t)}), \mathbf{x}^{(t)} - \mathbf{x}^{*} \rangle \label{eqn: fedavg_term_a_2_main} \\
    & \leq \| \mathbf{x}^{(t)} - \mathbf{x}^{*} \|^2 + \eta_t^2{\mu}^2 \nonumber \\
    & -2\eta_t \sum_{n \in \mathcal{N}} p_{n}^{(t)} \langle \nabla {F}_{n}^{(t)}(\mathbf{x}^{(t)}), \mathbf{x}^{(t)} - \mathbf{x}^{*} \label{Eq.fedavg_term_a_temp_main}\rangle
\end{align}
We note that Eq. \eqref{eqn: fedavg_term_a_1_main} is due to the aggregated model update produced by \textit{FedAvg} as indicated by Eq. \eqref{eqn: fedavg_global_model_update_main}. Also, it is worth highlighting that Eq. \eqref{eqn: fedavg_term_a_2_main} is due to convexity of squared L2 euclidean norm. Furthermore, we use $\mu$-Lipschitz property of underlying ML loss function $f(\cdot;\cdot)$ as described by Assumption \ref{assumption: convexity_+_lipschitz} to obtain Eq. \eqref{Eq.fedavg_term_a_temp_main}. 
After re-arranging Eq. \eqref{Eq.fedavg_term_a_temp_main}, we get:
\begin{align}
     \sum_{n \in \mathcal{N}} p_{n}^{(t)} \langle \nabla {F}_{n}^{(t)}(\mathbf{x}^{(t)}), & \mathbf{x}^{(t)} - \mathbf{x}^{*} \rangle  \leq \frac{1}{2\eta_t} \| \mathbf{x}^{(t)} -  \mathbf{x}^{*} \|^2 \nonumber \\
     &- \frac{1}{2\eta_t}\|\mathbf{x}^{(t + 1)} - \mathbf{x}^{*}\|^2 
     + \frac{\eta_t {\mu}^2}{2}. \label{eqn: bound (a) 1_main} 
\end{align}
Also, due to the convexity of underlying ML loss function $f(\cdot;\cdot)$ as specified in Assumption \ref{assumption: convexity_+_lipschitz}, we have:
\begin{align}
    & {F}_{n}^{(t)}(\mathbf{x}^{(t)}) - {F}_{n}^{(t)}(\mathbf{x}^{*}) \leq \langle \nabla {F}_{n}^{(t)}(\mathbf{x}^{(t)}), \mathbf{x}^{(t)} - \mathbf{x}^{*} \rangle.
\end{align}
\begin{align}
{F}^{(t)}(\mathbf{x}^{(t)}) - {F}^{(t)}(\mathbf{x}^{*}) & = \sum_{n \in \mathcal{N}} p_{n}^{(t)} \big[{F}_{n}^{(t)}(\mathbf{x}^{(t)}) - {F}_{n}^{(t)}(\mathbf{x}^{*})\big] \nonumber \\
    & \leq \sum_{n \in \mathcal{N}} p_{n}^{(t)} \langle \nabla {F}_{n}^{(t)}(\mathbf{x}^{(t)}), \mathbf{x}^{(t)} - \mathbf{x}^{*} \rangle \label{eqn: fedavg_term_a_3_main}
\end{align}
Now, we use Eq. \eqref{eqn: bound (a) 1_main} to upper bound the RHS of Eq. \eqref{eqn: fedavg_term_a_3_main}, thereby obtaining:
\begin{align}
    {F}^{(t)}(\mathbf{x}^{(t)}) - {F}^{(t)}(\mathbf{x}^{*}) \leq & \frac{1}{2\eta_t} \Big[\| \mathbf{x}^{(t)} - \mathbf{x}^{*} \|^2 - \|\mathbf{x}^{(t + 1)} &- \mathbf{x}^{*}\|^2 \Big] \nonumber \\
    & + \frac{\eta_t {\mu}^2}{2}. \label{eqn: bound (a) 2_main}
\end{align}
Conducting summation over $t = 1$ to $t = T$ in Eq. \eqref{eqn: bound (a) 2_main} with learning rates $\eta_t = \frac{1}{\sqrt{T}}$, $\forall t$, we obtain the following bound for term (a):
\begin{align}
    \sum_{t=1}^{T} {F}^{(t)}(\mathbf{x}^{(t)}) - {F}^{(t)}(\mathbf{x}^{*}) \leq \frac{\sqrt{T}}{2} \| \mathbf{x}^{(1)} - \mathbf{x}^{*} \|^2 + \frac{\mu^2 \sqrt{T}}{2}. \label{eqn: bound (a) 3_main}
\end{align}
In the following, we focus on bounding term (b) in Eq. \eqref{eqn: fedavg_dreg1_main}. More specifically, we want to show that ${F}^{(t)}(\mathbf{x}^{*}) -{F}^{(t)}(\mathbf{x}^{(t),*}) \leq 2\Delta_{[1,T]}$, $\forall t$. Suppose otherwise, in that case, $\exists ~t_0$ such that ${F}^{(t_0)}(\mathbf{x}^{*}) -{F}^{(t_0)}(\mathbf{x}^{(t_0),*}) > 2\Delta_{[1,T]}$. Then,
\begin{align}
    {F}^{(t)}(\mathbf{x}^{(t_0), *}) & \leq {F}^{(t_0)}(\mathbf{x}^{(t_0), *}) + \Delta_{[1,T]},  \label{eqn: term b 1_main}\\
    & <  {F}^{(t_0)}(\mathbf{x}^{*}) -  \Delta_{[1,T]}, \label{eqn: term b 2_main} \\ 
    & \leq F^{(t)}(\mathbf{x}^{*}), \hskip 1cm \forall t. \label{eqn: term b 3_main}
\end{align}
Summing over $t = 1$ to $t = T$ for both LHS and RHS of Eq. \eqref{eqn: term b 3_main}, we obtain:
\begin{align}
    \sum_{t = 1}^{T} {F}^{(t)}(\mathbf{x}^{(t_0), *}) <  \sum_{t = 1}^{T} F^{(t)}(\mathbf{x}^{*})  
\end{align}
thereby contradicting the definition of $\mathbf{x}^{*}$ presented in Eq. \eqref{eqn: static_comp_def_main}. Therefore, for each element of term (b), the following holds:
\begin{align}
    {F}^{(t)}(\mathbf{x}^{*}) -{F}^{(t)}(\mathbf{x}^{(t),*}) \leq 2\Delta_{[1,T]}, \forall t \label{eqn: term b 4_main}
\end{align}
Summing over $t = 1$ to $t = T$ for both LHS and RHS of Eq. \eqref{eqn: term b 4_main}, we get the final bound for term (b) as:
\begin{align}
    \sum_{t = 1}^{T} {F}^{(t)}(\mathbf{x}^{*}) - \sum_{t = 1}^{T} {F}^{(t)}(\mathbf{x}^{(t),*}) \leq 2T\Delta_{[1,T]} \label{eqn: term b 5_main}
\end{align}
Combining the bounds for term (a) and (b) as reflected via Eq. \eqref{eqn: bound (a) 3_main} and \eqref{eqn: term b 5_main} respectively, we get the final bound for cumulative dynamic regret $R_{[1,T]}$ as:
\begin{align}
    R_{[1,T]} \leq \frac{\sqrt{T}}{2} \| \mathbf{x}^{(1)} - \mathbf{x}^{*} \|^2 + \frac{{\mu}^2 \sqrt{T}}{2} + 2T\Delta_{[1, T]}
\end{align}
\newpage

\section{Dynamic Regret Analysis for FedOMD Algorithm - Proof of Theorem \ref{thm: fedomd_convex_regr_main}} \label{app: base_alg_guarantees}
In the following subsections, we provide explicit analysis of \textit{dynamic regret}, i.e., $ R_{[1,T]}$ as defined in Eq. \eqref{defn: dynamic regret} for \textit{FedOMD}. First, we recall the details of FL model update procedure for \textit{FedOMD} in Section \ref{sec: fedomd_analysis}, and subsequently derive the \textit{dynamic regret} bound in Section \ref{sec: proof_fedomd_appendix}.
\subsection{Summary of FedOMD Algorithm} \label{sec: fedomd_analysis}
Here, we briefly restate the modeling assumptions, update rule and aggregation step for \textit{FedOMD} algorithm which is detailed in the manuscript in Sec. \ref{sec: base_FL_disc_main}. For a 1-strongly convex function $\phi : \mathbb{R}^d \rightarrow \mathbb{R}$, the following holds:
\begin{align}
    \phi(\mathbf{y}) \geq \phi(\mathbf{x}) + \langle \mathbf{y}-\mathbf{x}, \nabla{\phi(\mathbf{x})}  \rangle + \frac{1}{2} \| \mathbf{y} - \mathbf{x} \|^2, ~\forall \mathbf{x}, \mathbf{y} \in \mathbb{R}^d. \label{eqn: phi_strongly_convex_defn}
\end{align}
Consequently, the Bregman Divergence w.r.t $\phi(.)$ is:
\begin{align}
    B_{\phi}(\mathbf{y};\mathbf{x}) \triangleq \phi(\mathbf{y}) -\phi(\mathbf{x}) - \langle \mathbf{y} - \mathbf{x}, \nabla{\phi}(\mathbf{x}) \rangle.
\end{align}
The Bregman Divergence $B_{\phi}$ is assumed to satisfy Assumption \ref{assumption: bregman_weighted_main}.
The local model update i.e., \texttt{FL-UPDATE}($\cdot$) during each FL round is:
\begin{align}
    & \texttt{FL-UPDATE:} \hspace{13mm} \mathbf{x}^{(t+1)}_{n} = \underset{x \in \mathbb{R}^p}{\argmin} ~\psi_{n}(\mathbf{x}; \mathbf{x}^{(t)}), \label{eqn:bregman_minimize_func} \\
    & \hspace{39mm} \psi_{n}(\mathbf{x}; \mathbf{x}^{(t)}) \triangleq \langle \nabla {F}_{n}^{(t)}(\mathbf{x}^{(t)}), \mathbf{x} \rangle + \frac{1}{\eta_t} B_{\phi} (\mathbf{x}; \mathbf{x}^{(t)}). \label{eqn:bregman_loss_func_defn}
\end{align}
The global ML model aggregation procedure as summarized in Eq. \eqref{eqn: global_ML_aggr} in the following: 
\begin{align} 
    \mathbf{x}^{(t+1)} = \sum_{n \in \mathcal{N}} p_{n}^{(t)} \mathbf{x}^{(t+1)}_{n}. \label{eqn: global_ML_aggr_fedomd}
\end{align}
\subsection{Proof of Theorem \ref{thm: fedomd_convex_regr_main}} \label{sec: proof_fedomd_appendix}
First, we restate the decomposition of \textit{dynamic regret} as conducted in Theorem \eqref{thm: fedavg_static_regr_main} via Eq. \eqref{eqn: fedavg_dreg_def2_main} - \eqref{eqn: fedavg_dreg1_main}:
\begin{align}
    R_{[1,T]} &= \sum_{t = 1}^{T} {F}^{(t)}(\mathbf{x}^{(t)}) -  \sum_{t = 1}^{T} {F}^{(t)}(\mathbf{x}^{(t),*}), \label{eqn: fedomd_dreg_def2} \\
    & = \underbrace{\sum_{t = 1}^{T} {F}^{(t)}(\mathbf{x}^{(t)}) - \sum_{t = 1}^{T} {F}^{(t)}(\mathbf{x}^{*})}_\text{(a)} + \underbrace{\sum_{t = 1}^{T} {F}^{(t)}(\mathbf{x}^{*}) - \sum_{t = 1}^{T} {F}^{(t)}(\mathbf{x}^{(t),*})}_\text{(b)},   \label{eqn: fedomd_dreg1}
\end{align}
In order to bound (a) in Eq. \eqref{eqn: fedomd_dreg1}, we leverage the convexity of $F_{n}^{(t)}, ~n \in \mathcal{N}$ to write:
\begin{align}
    {F}_{n}^{(t)}(\mathbf{x}^{(t)}) - {F}_{n}^{(t)}(\mathbf{x}^{*})  & \leq \langle \mathbf{x}^{(t)} - \mathbf{x}^{*}, \nabla {F}_{n}^{(t)}(\mathbf{x}^{(t)}) \rangle, \\
    & \leq \underbrace{\langle \mathbf{x}^{(t)} - \mathbf{x}^{(t+1)}_n, \nabla {F}_{n}^{(t)}(\mathbf{x}^{(t)}) \rangle}_\text{(c)} + \underbrace{\langle \mathbf{x}^{(t+1)}_n - \mathbf{x}^{*}, \nabla {F}_{n}^{(t)}(\mathbf{x}^{(t)}) \rangle}_\text{(d)}. \label{eqn: term_a_convexity_result}
\end{align}
We use {Fenchel-Young inequality} \cite{ando1995matrix} to bound term (c) in Eq. \eqref{eqn: term_a_convexity_result} as follows:
\begin{align}
    \langle \mathbf{x}^{(t)} - \mathbf{x}^{(t+1)}_n , \nabla {F}_{n}^{(t)}(\mathbf{x}^{(t)}) \rangle & \leq \frac{1}{2 \eta_t} \| \mathbf{x}^{(t)} - \mathbf{x}^{(t+1)}_n \|^2 + \frac{\eta_t}{2} \| \nabla {F}_{n}^{(t)}(\mathbf{x}^{(t)}) \|^2, \\
    & \leq \frac{1}{2 \eta_t} \| \mathbf{x}^{(t)} - \mathbf{x}^{(t+1)}_n \|^2 + \frac{\eta_t}{2} \mu^2, \label{eqn:bound_term_c}
\end{align}
where we use $\mu$-Lipschitz property of $f(\cdot;\cdot)$ in Eq. \eqref{eqn:bound_term_c}. Next, in order to bound term (d) we first note that strong-convexity in conjunction with the {first-order optimality condition} associated with $\psi_{n}(\mathbf{x}; \mathbf{x}^{(t)})$ allows us to obtain the following:
\begin{align}
    & \langle \nabla{\psi}(\mathbf{x}^{(t + 1)}_{n}; \mathbf{x}^{(t)}), \mathbf{x}^{*} - \mathbf{x}^{(t + 1)}_{n} \rangle \geq 0, \\
    & \langle \nabla {F}_{n}^{(t)}(\mathbf{x}^{(t)}) + \frac{1}{\eta_t}\nabla{\phi}(\mathbf{x}^{(t+1)}_{n}) - \frac{1}{\eta_t}\nabla{\phi}(\mathbf{x}^{(t)}), \mathbf{x}^{*} - \mathbf{x}^{(t + 1)}_{n} \rangle \geq 0, \label{eqn: fedomd_term_b_3} 
\end{align}

After re-arranging terms in Eq. \eqref{eqn: fedomd_term_b_3}, we get:
\begin{align}
    \langle \mathbf{x}^{(t+1)}_n - \mathbf{x}^{*}, \nabla {F}_{n}^{(t)}(\mathbf{x}^{(t)}) \rangle & \leq \frac{1}{\eta_t} \langle \nabla{\phi}(\mathbf{x}^{(t+1)}_n) - \nabla{\phi}(\mathbf{x}^{(t)}) , \mathbf{x}^{*} - \mathbf{x}^{(t+1)}_n \rangle, \\
    & = \frac{1}{\eta_t} \big[ B_{\phi}(\mathbf{x}^{*};\mathbf{x}^{(t)}) - B_{\phi}(\mathbf{x}^{*};\mathbf{x}^{(t+1)}_n) \big] - \frac{1}{\eta_t} B_{\phi}(\mathbf{x}^{(t+1)}_n; \mathbf{x}^{(t)}), \label{eqn:bound_termd_1} \\
    & \leq \frac{1}{\eta_t} \big[ B_{\phi}(\mathbf{x}^{*};\mathbf{x}^{(t)}) - B_{\phi}(\mathbf{x}^{*};\mathbf{x}^{(t+1)}_n) \big] -\frac{1}{2 \eta_t} \| \mathbf{x}^{(t)} - \mathbf{x}^{(t+1)}_n \|^2. \label{eqn:bound_termd_2}
\end{align}
Eq. \eqref{eqn:bound_termd_1} is due to {``three-point equality" for Bregman divergences} which implies $\mathbf{x}, \mathbf{y}, \mathbf{z} \in \mathbb{R}^p$, the following holds:
\begin{align}
    \langle \nabla{\phi}(\mathbf{x}) - \nabla{\phi}(\mathbf{y}), \mathbf{x} - \mathbf{z} \rangle =  B_{\phi}(\mathbf{x},\mathbf{y}) + B_{\phi}(\mathbf{z},\mathbf{x}) - B_{\phi}(\mathbf{z},\mathbf{y}). 
\end{align}
Furthemore, we note that Eq. \eqref{eqn:bound_termd_2} is a result of strong convexity of $\phi(\cdot)$. Consequently, combining Eq. \eqref{eqn: term_a_convexity_result}, \eqref{eqn:bound_term_c} and \eqref{eqn:bound_termd_2}, as well as summing over $n \in \mathcal{N}$, we obtain the following at each $t \leq T$:
\begin{align}
    {F}^{(t)}(\mathbf{x}^{(t)}) - {F}^{(t)}(\mathbf{x}^{(*)}) & \leq  \frac{1}{\eta_t}\big[ B_{\phi}(\mathbf{x}^{*},\mathbf{x}^{(t)}) - B_{\phi}(\mathbf{x}^{*},\mathbf{x}^{(t+1)}) \big] \nonumber \\
    & \hspace{5mm} + \frac{1}{\eta_t} \underbrace{\big[B_{\phi}(\mathbf{x}^{*},\mathbf{x}^{(t+1)}) - \sum_{n \in \mathcal{N}} p_n^{(t)}  B_{\phi}(\mathbf{x}^{*};\mathbf{x}^{(t+1)}_n) \big]}_\text{(e)} + \frac{\eta_t}{2} \mu^2, \\
     & \leq \frac{1}{\eta_t}\big[ B_{\phi}(\mathbf{x}^{*},\mathbf{x}^{(t)}) - B_{\phi}(\mathbf{x}^{*},\mathbf{x}^{(t+1)}) \big] + \frac{\eta_t}{2} \mu^2. \label{eqn: term_a_2}
\end{align}
We note that Eq. \eqref{eqn: term_a_2} is due to the fact that Assumption \ref{assumption: bregman_weighted_main} implies term (e) $\leq 0$. Summing LHS, RHS of Eq. \eqref{eqn: term_a_2} over $t=1$ to $t = T$, and using the choice of learning rate as $\eta_t = \frac{1}{\sqrt{T}}$, we collect the final bound for term (a):
\begin{align}
    \sum_{t = 1}^{T} {F}^{(t)}(\mathbf{x}^{(t)}) - \sum_{t = 1}^{T} {F}^{(t)}(\mathbf{x}^{*}) & \leq \sum_{t=1}^{T} \frac{1}{\eta_t} \big[ B_{\phi}(\mathbf{x}^{*},\mathbf{x}^{(t)}) - B_{\phi}(\mathbf{x}^{*},\mathbf{x}^{(t+1)}) \big] + \sum_{t=1}^{T} \frac{\eta_t}{2} \mu^2, \\
    & \leq \sqrt{T}\big[B_{\phi}(\mathbf{x}^{*},\mathbf{x}^{(1)}) - B_{\phi}(\mathbf{x}^{*},\mathbf{x}^{(T+1)}) \big] + \frac{\mu^2}{2}\sqrt{T}, \\
    & \leq \sqrt{T} B_{\phi}(\mathbf{x}^{*},\mathbf{x}^{(1)}) + \frac{\mu^2}{2}\sqrt{T}. \label{eqn: term_a_3 fedomd}
\end{align}

To bound (b), the exact same set of arguments detailed in Theorem \ref{thm: fedavg_static_regr_main} via Eq. \eqref{eqn: term b 1_main} - \eqref{eqn: term b 5_main} can be directly reused to obtain:
\begin{align}
    \sum_{t = 1}^{T} {F}^{(t)}(\mathbf{x}^{*}) - \sum_{t = 1}^{T} {F}^{(t)}(\mathbf{x}^{(t),*}) \leq 2T\Delta_{[1,T]}. \label{eqn: term b 5 fedomd}
\end{align}
Using the bounds for term (a) and term (b) presented via Eq. \eqref{eqn: term_a_3 fedomd} - \eqref{eqn: term b 5 fedomd}, we obtain the final bound for \textit{dynamic regret}, i.e., $R_{[1,T]}$ as:
\begin{align}
    & R_{[1,T]} \leq \sqrt{T} B_{\phi}(\mathbf{x}^{*},\mathbf{x}^{(1)}) + \frac{\mu^2}{2}\sqrt{T} + 2T\Delta_{[1,T]}.
\end{align}
\newpage
\section{Construction of optimistic Loss function $\Tilde{F}^{(t)}$ and Validation of Assumption \ref{assmptn: near stationary} for FedAvg, FedOMD} \label{sec: optimistic_estimator_verification}
In this section, we summarize the construction of optimistic Loss function $\Tilde{F}^{(t)}$ for \textit{FedOMD}, the same choice of optimistic Loss tracker holds for \textit{FedAvg} which can be verified by leveraging the mathematical justifications we provide in the subsequent discussion. More specifically, we provide a proof sketch verifying that Assumption \ref{assmptn: near stationary} holds for prescribed choice of the optimistic estimator. 
The iterative procedure for ML model updates for \textit{FedOMD} is illustrated via Eq. \eqref{eqn: phi_strongly_convex_defn_main} - \eqref{eqn: global_ML_aggr_fedomd_main} in Section \ref{sec: base_FL_disc_main}. At round $t$, the optimistic global Loss Estimator for \textit{FedOMD} algorithm can be constructed as:
\begin{align}
    & \Tilde{F}^{(t)} = \underbrace{\frac{1}{t}\sum_{\tau = 1}^{t} F^{(\tau)}({\mathbf{x}}^{(t)})}_\text{(a)} -\Tilde{c}\sqrt{\frac{log(T/\delta)}{\Bar{D}^{(t)}}}, \label{eqn: optimistic_est_def_1}\\
    & \Bar{D}^{(t)} \triangleq \sum_{\tau = 1}^{t} D^{(\tau)}, \label{eqn: optimistic_est_def_2}\\
\end{align}
where $D^{(t)}$ is defined via Eq. \eqref{eqn: ML dataset definition}. We note that the first term in RHS of Eq. \eqref{eqn: optimistic_est_def_1} denoted by (a) is the empirical mean of ML losses across all datapoint collected till time $t$ at the most recent ML model $\mathbf{x}^{(t)}$, $\Bar{D}^{(t)}$ is the cumulative number of datapoints collected till $t$. We highlight that in Eq. \eqref{eqn: optimistic_est_def_1}, $\Tilde{c}$ is an artifact of the application of Azuma-Hoeffding inequality for martingales with bounded variations \cite{cesa2004generalization}, \cite{azuma1967weighted}. Hence, as a consequence of the arguments provided in Theorem 3 of \cite{fedomdpaper}, we have the following:
\begin{align}
  \Tilde{F}^{(t)} &\leq \underset{\mathbf{x} \in \mathbbm{R}^d}{\min} ~\frac{1}{t} \sum_{\tau=1}^{t} F^{(\tau)}(\mathbf{x}), \label{eqn: optimistic_est_def_4} \\
   & \leq \frac{1}{t} \sum_{\tau=1}^{t} F^{(\tau)}(\mathbf{x}^{*}), \\
   & \leq \underset{\tau \leq t}{\max} ~F^{(\tau)}(\mathbf{x}^{*}), \\
   & \leq F^{(\tau_{max})}(\mathbf{x}^{(\tau_{max}), *}) + 2\Delta_{[1,t]}, \label{eqn: optimistic_est_def_5} \\
   & \leq \underset{\tau \leq t}{\max} ~F^{(\tau)}(\mathbf{x}^{(\tau), *}) + 2\Delta_{[1,t]},
\end{align}
where $\tau_{max} \triangleq \underset{\tau \leq t}{\argmax} ~F^{(\tau)}(\mathbf{x}^{*})$. We highlight that Eq. \eqref{eqn: optimistic_est_def_4} holds with high probability as a direct consequence of Azuma-Hoeffding's inequality. Eq. \eqref{eqn: optimistic_est_def_5} can be verified using the mathematical arguments detailed through Eq. \eqref{eqn: term b 1_main} - \eqref{eqn: term b 4_main}. This verifies Eq. \eqref{eqn: optimistic_est_1} in Assumption \ref{assmptn: near stationary}. Next, with probability $1-\delta$, we must have:
\begin{align}
    \sum_{\tau = 1}^{t} \Big[ {F}^{(\tau)}(\mathbf{x}^{(\tau)}) - \Tilde{F}^{(\tau)} \Big] & \leq \sum_{\tau = 1}^{t}  \Big[{F}^{(\tau)}(\mathbf{x}^{(\tau)}) - \frac{1}{\tau}\sum_{\tau' = 1}^{\tau} F^{(\tau')}({\mathbf{x}}^{(\tau)})\Big] + \mathcal{O}(\sqrt{t~log(T/\delta)}), \\
    & \leq t\Delta_{[1,t]} + \mathcal{O}(\sqrt{t~log(T/\delta)}),
\end{align}
where first term on the RHS of Eq. \eqref{eqn: optimistic_est_def_5} is due to direct application of cumulative \textit{concept drift} (see Definition \ref{defn: model drift}). Hence, Eq. \eqref{eqn: optimistic_est_2} of Assumption \ref{assmptn: near stationary} is valid with high probability with $\rho(t) = \mathcal{O}\Big(\sqrt{\frac{log(T/\delta)}{t}}\Big)$.
\newpage
{\color{black} \section{Intuitive understanding of Requirement \ref{assmptn: near stationary} via mathematical justifications } \label{sec: near_stationary_intuition}
For ease of our discussion, we state the requirement next.
\\
\begin{requirement*} [Base algorithm performance guarantee in a near-stationary environment] \label{assmptn: near stationary appendix} We assume that the base algorithm produces an auxiliary quantity $\Tilde{F}^{(t)}$ at the end of each global round of aggregation $t \in \{1,2, \cdots, T\}$ satisfying the following:  
\begin{align}
    & \Tilde{F}^{(t)} \leq \underset{\tau \in [1, t]}{\max} F^{(\tau)}(\mathbf{x}^{(\tau),*}) + \Delta_{[1, t]}, \label{eqn: optimistic_est_1_R024} \\
    & \frac{1}{t}\sum_{t = 1}^{t} [{F}^{(t)}(\mathbf{x}^{(t)}) - \Tilde{F}^{(t)}] \leq \rho(t) + \Delta_{[1, t]}. \label{eqn: optimistic_est_2_R024}
\end{align}
where $\rho(.)$ is described in Definition 3. Further, $\Delta_{[1,t]}$ represents the cumulative concept drift experienced with $\Delta_{[1,t]} \leq \rho(t)$ , i.e., near-stationary environment.
\end{requirement*}
In Section \ref{sec: problem_formulation}, prior to introducing Requirement \ref{assmptn: near stationary}, we mention that our algorithmic framework needs underlying baseline FL algorithms to have certain theoretical performance guarantees in environments where the drift is small, which we also alternatively denote as \textit{near-stationary} environments in our manuscript. With this requirement, our goal is to identify baseline FL algorithms which can be run standalone in aforementioned \textit{near-stationary} environments. Consequently, only such baseline algorithms could be wrapped with our algorithmic framework to deal with non-stationary environments characterized by high degrees of drifts.\\

To develop a mathematical intuition of this requirement, consider a \textit{perfectly stationary} learning setting, where global losses are collected on time-invariant datasets implying $F^{(\tau)} \rightarrow F$  and $\Delta_{[1, t]} = 0$ (i.e., cumulative \textit{concept drift} is 0). This also implies minimizers $\{\mathbf{x}^{(\tau),*}\}_{1 \leq \tau T}$ are also \textit{static} since $F$ is unchanged, we denote this minimizer by $\overline{x}^{*}$. Hence, for this particular learning setting, Eq. \eqref{eqn: optimistic_est_1_R024}, \eqref{eqn: optimistic_est_2_R024}, can be reduced to:
\begin{align}
    & \Tilde{F}^{(t)} \leq F(\overline{x}^{*}), \label{eqn: optimistic_est_1_R24} \\
    & \sum_{t = 1}^{t} [{F}^{(t)}(\mathbf{x}^{(t)}) - \Tilde{F}^{(t)}] \leq t.\rho(t). \label{eqn: optimistic_est_2_R24}
\end{align}
Note that, if we choose $\Tilde{F}^{(t)} = F(\overline{x}^{*})$, Eq. \eqref{eqn: optimistic_est_1_R24} holds trivially. To confirm that \eqref{eqn: optimistic_est_2_R24} holds when baseline algorithms are \textit{FedAvg, FedOMD} with this choice for $\Tilde{F}^{(t)}$, we first recall the results of Theorem \ref{thm: fedavg_static_regr_main}, \ref{thm: fedomd_convex_regr_main} in the manuscript.
\begin{theorem}[Dynamic Regret for Convex Loss function with \textit{FedAvg}] \label{thm: fedavg_static_regr_resp1} Assume that the underlying ML loss measure $f(\cdot;\cdot)$ satisfies Assumption 2 and local learning rates at the DPUs collectively represented by $\mathcal{N}$ are set to $\eta_t = \frac{1}{\sqrt{T}}$ for $t \in  \{1,2,\cdots, T\}$, the cumulative dynamic regret incurred by \textit{FedAvg} Algorithm is bounded by:
\begin{align}
    R_{[1,T]} \leq \frac{\sqrt{T}}{2} \| \mathbf{x}^{(1)} - \mathbf{x}^{*} \|^2 + \frac{{\mu}^2 \sqrt{T}}{2} + 2\underbrace{T\Delta_{[1, T]}}_\text{(*)}, \label{eq: fedavg_regr_resp1}
\end{align}
\end{theorem}

\begin{theorem}[Dynamic Regret with Convex Loss function for \textit{FedOMD}] \label{thm: fedomd_convex_regr_resp1} Assume that the underlying ML loss measure $f(\cdot;\cdot)$ satisfies Assumption 2 and local learning rates at the DPUs collectively represented by $\mathcal{N}$ are set to $\eta_t = \frac{1}{\sqrt{T}}$ for $t \in  \{1,2, \cdots, T \}$, the cumulative dynamic regret incurred by \textit{FedOMD} Algorithm is bounded by:
\begin{align}
    R_{[1,T]} \leq \sqrt{T} B_{\phi}(\mathbf{x}^{*},\mathbf{x}^{(1)}) + \frac{\mu^2}{2}\sqrt{T} + 2\underbrace{T\Delta_{[1, T]}}_\text{(*)}. \label{eq: fedomd_regr_resp1} 
\end{align}
\end{theorem}
Firstly, note that $\rho(t) = \mathcal{O}(\frac{1}{\sqrt{t}})$ for \textit{FedAvg, FedOMD}, which is the worst case regret when the drift term $(*)$ goes to 0 in a perfectly stationary setting as per Eq. \eqref{eq: fedavg_regr_resp1}, \eqref{eq: fedomd_regr_resp1}. So, the RHS becomes $\Tilde{\mathcal{O}}(\sqrt{t})$ for $t$ rounds with $\Delta_{[1, t]} = 0$, resembling RHS of Eq. \eqref{eqn: optimistic_est_2_R24} since $t.\rho(t) = \mathcal{O}(\sqrt{t})$ in this case. Furthermore, we observe that as long as $\Delta_{[1,t]} \leq \rho(t)$ for these baseline algorithms (i.e., \textit{near-stationarity} condition), \textit{dynamic regret} bounds eq. \eqref{eq: fedavg_regr_resp1}, \eqref{eq: fedomd_regr_resp1}, still stays equivalent to eq. \eqref{eqn: optimistic_est_2_R24} with the choice $\Tilde{F}^{(t)} = F^{(t)}(\mathbf{x}^{(t),*})$, and this choice also satisfies eq. \eqref{eqn: optimistic_est_1_R024} trivially. Therefore, Requirement \ref{assmptn: near stationary}  holds for a single instance of these baseline FL algorithms without the need to execute Master-FL algorithmic framework ranging from \textit{perfectly stationary} i.e., $\Delta_{[1, t]} = 0$ to \textit{near-stationary} $\Delta_{[1,t]} \leq \rho(t)$ learning regimes.\\

Our non-stationary detection and adaptation metholodgy gains relevance in the more interesting regime where drifts are large i.e., $\Delta_{[1, t]} > \rho(t)$. Note that, just a single instance of \textit{FedAvg, FedOMD} would end up incurring $\Tilde{O}(t.\Delta_{[1, t]})$ \textit{dynamic regret} irrespective of scale of $\rho(t)$ as suggested by Eq. \eqref{eq: fedavg_regr_resp1}, \eqref{eq: fedomd_regr_resp1}. Especially, in Master-FL (Algorithm \ref{detection}) \textbf{Test 1} and \textbf{Test 2} are executed in conjunction with multi-scale instantiations (Algorithm \ref{scheduling_algo}, \ref{MALG}) over the training horizons, to detect if indeed the environment significantly deviated from \textit{near-stationarity} behavior. In order to see how \textbf{Test 1} and \textbf{Test 2} in Master-FL are connected to Requirement \ref{assmptn: near stationary}, we will first restate them here.
\begin{align}
   & U_t = {\max}_{\tau \in [{ t_{new}}, t]} ~\tilde{F}^{(t)}, \\
   & \text{\underline{\textbf{Test 1:}} Current $\mathcal{A}_t$ is some order $k$ base instance.} \nonumber \\
   & \hspace{1.5cm} U_t \geq \sum_{\tau = \mathcal{A}.s}^{\mathcal{A}.e} ~F^{(t)} + 9\hat{\rho}(2^k), \label{eq: test1}\\
   & \text{\underline{\textbf{Test 2:}}} \nonumber \\
   & \hspace{1.5cm} \frac{1}{t - t_{new} + 1} \sum_{\tau = t_{new}}^{t} [F^{(t)} - \Tilde{F}^{(t)}] \geq 3 \hat{\rho}(t - t_{new} + 1), \label{eq: test2}
\end{align}
where $t_{new}$ is the first timestamp where learning begins after a restart (please to refer to Algorithm \ref{detection} in the manuscript). Note that, \textbf{Test 1} i.e., eq. \eqref{eq: test1}  intuitively acts as a proxy for Eq. \eqref{eqn: optimistic_est_1_R024} which attempts to detect ``sudden" or drastic environment drifts. And, \textbf{Test 2} i.e., eq. \eqref{eq: test2} tries to mimic the second condition of Requirement \ref{assmptn: near stationary}, i.e., Eq. \eqref{eqn: optimistic_est_2_R024} attempting to identify a large cumulative drift that ``gradually" accumulated over the training time. We use proxy tests since conditions of Requirement \ref{assmptn: near stationary} cannot be tested without knowledge of actual environment drifts i.e., $\Delta_{[1,t]}$.\\

Finally, note that \textbf{Test 1} and \textbf{Test 2} requires careful construction of optimistic loss measure sequence $\{\tilde{F}^{(t)}\}_{t \in [0,T]}$ that would also satisfy the conditions of Requirement \ref{assmptn: near stationary}. In Appendix \ref{sec: optimistic_estimator_verification}, we propose how to construct $\tilde{F}^{(t)}$ using history till $t$ for \textit{FedAvg, FedOMD} algorithms and verify the conditions stated in Requirement \ref{assmptn: near stationary}. We emphasize that our algorithmic framework can only be equipped with baseline algorithms for which a suitable $\{\tilde{F}^{(t)}\}_{t \in [0,T]}$ exists that satisfy Requirement \ref{assmptn: near stationary}. We use \textit{FedAvg, FedOMD} as we could construct such a sequence and verify that the conditions of the aforementioned requirement hold. Furthermore, we also mathematically show that our multi-scale instantiations still preserves the conditions of Requirement \ref{assmptn: near stationary} and does not restart learning as long as the environment drifts are within \textit{near stationary} regime ($\Delta_{[1,t]} \leq \rho(t)$) (Results of Lemma \ref{lemma:multi_scale_regr_main}, \ref{lemma: test_trigger_correctness_main} in the main text). The overall theoretical \textit{dynamic regret} analysis presented in Section \ref{sec: key_theory} supports the fact that with the aforementioned tests and suitable construction of $\tilde{F}^{(t)}$  indeed restarts multi-scale learning consistently upon violations on \textit{near-stationarity} in the environment and achieves \textit{sub-linear} convergence rates without knowledge of degree of the drifts.
}
\newpage
\section{Analysis for Multi-Scale Algorithm - Proof of Lemma \ref{lemma:multi_scale_regr_main}} \label{sec: multiscale_analysis}
For our analysis, we consider the base FL instance denoted by $\mathcal{A}$ and $t \in [\mathcal{A}.s, \mathcal{A}.e]$. Furthermore, we reiterate that the cumulative \textit{concept drift} is bounded as $\Delta_{[\mathcal{A}.s,t]} \leq \rho(t')$ as has been specified in the statement of the Lemma. In order to prove the initial claim presented in the Lemma statement, we first highlight that $\Tilde{F}^{(t)}$ corresponds to the optimistic estimator produced by base FL instance $\mathcal{A}'$ which is active during round $t$. We denote this optimistic quantity as $\Tilde{F}^{(t)}_{\mathcal{A}'}$. We note that due to Randomized Scheduling Procedure (Algorithm  \ref{scheduling_algo}), $\mathcal{A}'$ must be initiated within $[\mathcal{A}.s, t]$, in other words $\mathcal{A}'.s \geq \mathcal{A}.s$. Therefore, it is straightforward to see that the drift experienced by $\mathcal{A}'$ is upto $\Delta_{[\mathcal{A}.s,t]} \leq \rho(t')$. Moreover, this is further upper bounded by $\rho(t'')$, with $t''$ being the actual number of rounds when $\mathcal{A}'$ was executed and $\rho(.)$ being a decreasing function. This ensures that the necessary conditions for Assumption \ref{assmptn: near stationary} are true for $\mathcal{A}'$, which subsequently implies that the following holds:
\begin{align}
    \Tilde{F}^{(t)} = \Tilde{F}^{(t)}_{\mathcal{A}'} & \leq \underset{\tau \leq t: \mathcal{A}' ~\text{active at} ~\tau}{\max} ~F^{(\tau)}(\mathbf{x}^{(\tau),*}) + \Delta_{[\mathcal{A}'.s, t]}, \\
    &\leq \underset{\tau \in [\mathcal{A}.s, t]}{\max}F^{(\tau)}(\mathbf{x}^{(\tau),*}) + \Delta_{[\mathcal{A}.s, t]}
\end{align}
This completes the proof of Eq. \eqref{eqn: opt_bound_eq1_main}.\\

We continue our analysis with the aforementioned base FL algorithm instance denoted by $\mathcal{A}$ which is instantiated for execution over the rounds $t \in [\mathcal{A}.s, \mathcal{A}.e]$. 
We now focus on proving the second part of the Lemma statement presented in Eq. \eqref{eqn: opt_bound_eq2_main}. To this end, we first denote the collection of all order $k$ base FL algorithm instances generated within $[\mathcal{A}.s, t]$ by $\mathcal{S}_k$. Therefore, we have:
\begin{align}
    \sum_{\tau = \mathcal{A}.s}^{t} {F}^{(\tau)}(\mathbf{x}^{(\tau)}) ~-  \sum_{\tau = \mathcal{A}.s}^{t} \Tilde{F}^{(\tau)} & = \sum_{\tau = \mathcal{A}.s}^{t} \sum_{k = 0}^{m} \sum_{\Tilde{\mathcal{A}} \in \mathcal{S}_k} \mathbbm{1}[\Tilde{\mathcal{A}} ~\text{is active during round} ~\tau] \Big[ {F}^{(\tau)}_{\Tilde{\mathcal{A}}}(\mathbf{x}^{(\tau)}) -  \Tilde{F}^{(\tau)}_{\Tilde{\mathcal{A}}} \Big], \label{eqn: regr_active_m_1}\\
    & = \sum_{k = 0}^{m} \underbrace{\sum_{\Tilde{\mathcal{A}} \in \mathcal{S}_k} \sum_{\tau = \mathcal{A}.s}^{t} \mathbbm{1}[\Tilde{\mathcal{A}} ~\text{is active during round} ~\tau] \Big[ {F}^{(\tau)}_{\Tilde{\mathcal{A}}}(\mathbf{x}^{(\tau)}) -  \Tilde{F}^{(\tau)}_{\Tilde{\mathcal{A}}} \Big]}_\text{(a)}. \label{eqn: regr_active_m_2}
\end{align}
We note that Eq. \eqref{eqn: regr_active_m_1} is a direct consequence of the fact that if $\Tilde{\mathcal{A}}$ is the unique active FL base algorithm instance during round $\tau$
In the following, we first characterize the bound pertaining to term (a) in Eq. \eqref{eqn: regr_active_m_2} with a fixed order $m$. For our convenience, we denote $|\mathcal{S}_k| = \tilde{k}$ and $\mathcal{S}_m = \{\Tilde{\mathcal{A}}_1, \Tilde{\mathcal{A}}_2, \cdots , \Tilde{\mathcal{A}}_{\tilde{k}}\}$. Furthermore, each FL instance $\Tilde{\mathcal{A}}_i$ is instantiated for the rounds $[\Tilde{\mathcal{A}}_i.s, \Tilde{\mathcal{A}}_i.e]$, and the overlapping intervals for each such instance is denoted by $\mathcal{I}_i \triangleq [\Tilde{\mathcal{A}}_i.s, \Tilde{\mathcal{A}}_i.e] \cap [\mathcal{A}.s, t]$. Also, the cumulative \textit{concept drift} associated with the FL rounds in $\mathcal{I}_i$ is denoted by  $\Delta_{\mathcal{I}_i}$. Now, we bound the aforementioned term (a) as follows:
\begin{align}
     \nonumber \text{(a)} &=  \sum_{i = 1}^{\tilde{k}} \sum_{\tau = \mathcal{A}.s}^{t} \mathbbm{1}[\Tilde{\mathcal{A}}_i ~\text{is active during round} ~\tau] \Big[ {F}^{(\tau)}_{\Tilde{\mathcal{A}}}(\mathbf{x}^{(\tau)}) -  \Tilde{F}^{(\tau)}_{\Tilde{\mathcal{A}}} \Big], \\
    & \label{eqn: term_a_bound1} \leq \sum_{i = 1}^{\tilde{k}} \big(C(|\mathcal{I}_{i}|) +  |\mathcal{I}_{i}| \Delta_{\mathcal{I}_i} \big), \\
    & \label{eqn: term_a_bound2} \leq \tilde{k} \underbrace{C(\min \{2^m, t - \mathcal{A}.s + 1\})}_\text{(b)} + \underbrace{(t - \mathcal{A}.s + 1)\Delta_{[\mathcal{A}.s, t]}}_\text{(c)}.
\end{align}
We note that Eq. \eqref{eqn: term_a_bound1} is a direct consequence of Assumption \ref{assmptn: near stationary}. For verifying Eq. \eqref{eqn: term_a_bound2}, we note that for each FL instance $\Tilde{\mathcal{A}}_i$, we have $|\mathcal{I}_i| \leq \min \{\Tilde{\mathcal{A}}_i.e - \Tilde{\mathcal{A}}_i.s + 1,t - \mathcal{A}.s + 1  \} = \min\{2^k, t - \mathcal{A}.s + 1\}$ and term (b) appears in the bound especially due to $C(\cdot)$ being a increasing function. And, the second term (c) is a result of the fact that $\sum_{i} \Delta_{\mathcal{I}_i} \leq \sum_{\tau = \mathcal{A}.s}^{t} \Delta_{\tau} = \Delta_{[\mathcal{A}.s, t]}$ since intervals $\mathcal{I}_1, \mathcal{I}_2, \cdots, \mathcal{I}_k$ are non-intersecting within $[\mathcal{A}.s, t]$ by algorithm design.

Due to the Randomized Scheduling Procedure (Algorithm \ref{scheduling_algo}), for each order $k$, the expected number of base FL algorithm instances scheduled within the interval $[\mathcal{A}.s, t]$, i.e. $\mathbb{E}\big[|\mathcal{S}_k|\big]$, can be restricted as follows:
\begin{align}
    \mathbb{E}\big[|\mathcal{S}_k|\big] & \leq \frac{\rho(2^m)}{\rho(2^k)}\ceil*{ \frac{t- \mathcal{A}.s + 1}{2^k}}, \\
    & \leq \frac{\rho(2^m)}{\rho(2^k)}\Big[ \frac{t- \mathcal{A}.s + 1}{2^k} + 1 \Big], \\
    & \leq \frac{\rho(2^m)}{\rho(2^k)}\frac{t- \mathcal{A}.s + 1}{2^k} + 1.
\end{align}
Now, as a direct consequence of Bernstein's inequality \cite{bernstein}, with probability $1-\frac{\delta}{T}$, we have:
\begin{align}
    |\mathcal{S}_k| & \leq \mathbb{E}\big[|\mathcal{S}_k|\big] + \sqrt{2\mathbb{E}\big[|\mathcal{S}_k|\big]\log (T/\delta)} + \log (T/\delta), \\ 
    & \leq 2\mathbb{E}\big[|\mathcal{S}_k|\big] + 2 \log (T/\delta), \\
    & \label{eqn: |SM| bound} \leq 2 \Bigg[ \frac{\rho(2^m)}{\rho(2^k)}\frac{t- \mathcal{A}.s + 1}{2^k} + 1 \Bigg] + 2 \log (T/\delta).
\end{align}
Since, $\tilde{k} = |\mathcal{S}_k|$, Eq. \eqref{eqn: term_a_bound2} can be further upper bounded using Eq. \eqref{eqn: |SM| bound} as follows:
\begin{align}
    \text{(a)} & \nonumber  \leq  2 \Bigg[ \frac{\rho(2^m)}{\rho(2^k)}\frac{t- \mathcal{A}.s + 1}{2^k} + 1 \Bigg] C(\min \{2^k, t - \mathcal{A}.s + 1\}) + 2 \log (T/\delta) C(\min \{2^k, t - \mathcal{A}.s + 1\}) \\
     & \hspace{4mm} + (t - \mathcal{A}.s + 1)\Delta_{[\mathcal{A}.s, t]}, \\
    & \label{eqn: term_a_bound3} \leq 2 \Bigg[\frac{C(t- \mathcal{A}.s + 1)}{C(2^k)} + 2\Bigg] \log (T/\delta) C(\min \{2^k, t - \mathcal{A}.s + 1\}) + (t - \mathcal{A}.s + 1)\Delta_{[\mathcal{A}.s, t]}, \\
    & \label{eqn: term_a_bound4} \leq 6C(t- \mathcal{A}.s + 1)\log (T/\delta) + (t - \mathcal{A}.s + 1)\Delta_{[\mathcal{A}.s, t]}. 
\end{align}
We note that Eq. \eqref{eqn: term_a_bound3} is a direct consequence of the fact that $\rho(\cdot)$ is a decreasing function i.e., $\rho(2^m) \leq \rho(t- \mathcal{A}.s + 1)$.  Eq. \eqref{eqn: term_a_bound4} is justified due to $C(\cdot)$ being an increasing function. Next, we replace the upper bound of (a) as derived in Eq. \eqref{eqn: term_a_bound4} in Eq. \eqref{eqn: regr_active_m_2} to obtain the following bound with probability at least $1- \frac{\delta}{T}$:
\begin{align}
   \sum_{\tau = \mathcal{A}.s}^{t} {F}^{(\tau)}(\mathbf{x}^{(\tau)}) ~-  \sum_{\tau = \mathcal{A}.s}^{t} \Tilde{F}^{(\tau)} & \leq \sum_{k = 0}^{m} \big[ 6C(t- \mathcal{A}.s + 1)\log (T/\delta) + (t - \mathcal{A}.s + 1)\Delta_{[\mathcal{A}.s, t]} \big], \\
    & \leq 6(m + 1) C(t- \mathcal{A}.s + 1)\log (T/\delta) + (t - \mathcal{A}.s + 1) (m + 1) \Delta_{[\mathcal{A}.s, t]}. 
\end{align}
This concludes the formal verification of the second claim of the Lemma statement, i.e., Eq. \eqref{eqn: opt_bound_eq2_main}.\\

Now, in order to prove the bound for the number of base FL algorithm instances initiated, we utilize Eq. \eqref{eqn: |SM| bound}, and noting that with probability at least $1- \frac{\delta}{T}$, the following result holds:
\begin{align}
    \sum_{k = 0}^{m} |\mathcal{S}_k| & \leq \sum_{k = 0}^{m} 2 \Bigg[ \frac{\rho(2^m)}{\rho(2^k)}\frac{t- \mathcal{A}.s + 1}{2^k} + 2 \Bigg] \log (T/\delta), \\
    & \label{eqn: num_instances_bound1} \leq 2 \hat{m}\Bigg[ \frac{C(t- \mathcal{A}.s + 1)}{C(1)} + 2 \Bigg] \log (T/\delta), \\
    & \label{eqn: num_instances_bound2} \leq 6\hat{m}\frac{C(t- \mathcal{A}.s + 1)}{C(1)}\log (T/\delta),
\end{align}
where, we leverage $\rho(2^k)2^k = C(2^k) \geq C(1)$ and $\rho(2^m) \leq \rho(t- \mathcal{A}.s + 1)$ to obtain Eq. \eqref{eqn: num_instances_bound1} - \eqref{eqn: num_instances_bound2}.  
\clearpage
\section{Block Regret Analysis - Proof of Lemma \ref{eqn: blk_dyn_regr_main}} \label{Sec: block_dyn_regr_results}
In this section, we consider an arbitrary block of order $m$ for which Master-FL (Algorithm \ref{detection}) is run. We consider this block run for the rounds $[t_m, t_m + 2^m - 1]$. We decompose this single order $m$ block into successive intervals $\mathcal{I}_1 = [s_1, e_1]$, $\mathcal{I}_2 = [s_2, e_2]$, $\cdots$ , $\mathcal{I}_K = [s_K, e_K]$ ($s_1 = t_m, e_i + 1 = s_{i+1}, e_K = t_m + 2^m -1$). Furthermore, all these intervals satisfy:
\begin{align}
    \Delta_{\mathcal{I}_i} \leq \rho(|\mathcal{I}_i|). \label{eqn: delta_rho_reln_1}
\end{align}
Also, we reiterate that with our current set of baseline algorithms \textit{FedAvg}, \textit{FedOMD}, according to Theorem \ref{thm: fedavg_static_regr_main}, \ref{thm: fedomd_convex_regr_main} we can choose $C(t) = t\rho(t) = \min\{c_1\sqrt{t} + c_2, t \}$ since the losses are bounded in $[0,1]$.
\begin{definition} \label{def: block_endpoints_def}
We introduce $1 \leq \ell \leq K$, s.t. $E_m \in \mathcal{I}_{\ell}$ (basically $\ell$ denotes the index of the last stationary interval where the order $n$ block terminates). Also, let $\Tilde{e}_i = \min \{e_i, E_m \}$ and $\Tilde{\mathcal{I}}_i \triangleq [s_i, \Tilde{e}_i ]$ (hence, $|\Tilde{\mathcal{I}}_i| = 0, ~i > \ell$).   
\end{definition}
\begin{definition} \label{def: tau_i_m}
For $i \in \{1,2,\cdots, K\}$ and $k \in {0, 1, 2, ..., m}$, we define:
\begin{align}
    \tau_i(k) \triangleq \min \{\tau \in  \Tilde{\mathcal{I}}_i: ~ \Tilde{F}^{(\tau)} - {F}^{(\tau)}(\mathbf{x}^{(\tau), *})  \geq 12 \hat{\rho}(2^k)\}. 
\end{align}
If such a $\tau$ doesn't exist or $|\Tilde{\mathcal{I}}_i| = 0$, we set $\tau_i(k) = \infty$. Furthermore, for the convenience of our analysis, we introduce $\zeta_i(k) \triangleq [\Tilde{e}_i - \tau_i(k) + 1]_{+}$ (it is the length of the interval $[\tau_i(k), \Tilde{e}_i]$ when $\tau_i(k)$ is not $\infty$).
\end{definition}
In the following, we first prove three auxiliary results, i.e., Lemma \ref{lemma: single_block_regret_lemma-I}, \ref{lemma: masterlemma17}, \ref{lemma: ell_delta_bound} in Section \ref{sec: auxiliary_block_regr}. We use these results to finally prove Lemma \ref{eqn: blk_dyn_regr_main} in Section \ref{sec: block_regr_lemma_subsec}.
\subsection{Auxiliary Block Regret Results} \label{sec: auxiliary_block_regr}
\begin{lemma} \label{lemma: single_block_regret_lemma-I}
Let the high-probability events described in Lemma \ref{lemma:multi_scale_regr_main} hold, then with high probability the following holds for a block $\mathcal{B} = [t_m, E_m]$:
\begin{align}
    & \sum_{\tau = t_m}^{E_m} \big[{F}^{(\tau)}(\mathbf{x}^{(\tau)}) - \Tilde{F}^{(\tau)} \big] \leq 4 \hat{C}(2^m), \label{eqn: block_regr_1} \\
    & \sum_{\tau = t_m}^{E_m} \big[ \Tilde{F}^{(\tau)} - {F}^{(\tau)}(\mathbf{x}^{(\tau), *}) \big] \leq 96\hat{m} \sum_{i = 1}^{\ell} \hat{C}(|\Tilde{\mathcal{I}}_i|) + 60 \sum_{k = 0}^{m} \frac{\rho(2^k)}{\rho(2^m)} \hat{C}(2^k) \log (T/\delta), \label{eqn: block_regr_2}  
\end{align}
which leads to the following block dynamic regret bound:
\begin{align}
   R_{[t_m, E_m]} = \sum_{\tau = t_m}^{E_m} \big[{F}^{(\tau)}(\mathbf{x}^{(\tau)}) - {F}^{(\tau)}(\mathbf{x}^{(\tau), *}) \big] \leq \Tilde{\mathcal{O}} \Big(\sum_{i = 1}^{\ell} C(|\mathcal{I}_{i}|) + \sum_{k = 0}^{m} \frac{\rho(2^k)}{\rho(2^m)} C(2^k) \Big). \label{eqn: single_block_dynamic_reg-I}
\end{align}
\end{lemma}
\begin{proof}
In order to prove Eq. \eqref{eqn: block_regr_1}, we note that as a consequence of \textbf{Test 2} in {Master-FL} (line 16 Algorithm \ref{detection}), the following holds:
\begin{align}
    \sum_{\tau = t_m}^{E_m} \big[{F}^{(\tau)}(\mathbf{x}^{(\tau)}) - \Tilde{F}^{(\tau)} \big] & \leq 3 \hat{C}(E_m - t_m + 1) + 1, \\
    & \leq 4 \hat{C}(2^m).
\end{align}
In order to prove Eq. \eqref{eqn: block_regr_2}, $\forall ~i = 1,2,\cdots, K$, we have:
\begin{align}
    \sum_{\tau \in \Tilde{\mathcal{I}}_i} \big[ \Tilde{F}^{(\tau)} - {F}^{(\tau)}(\mathbf{x}^{(\tau), *}) \big] & \leq 12 \sum_{\tau \in \Tilde{\mathcal{I}}_i} \Big[\mathbbm{1}[\Tilde{F}^{(\tau)} - {F}^{(\tau)}(\mathbf{x}^{(\tau), *}) \leq 12 \hat{\rho}(2^m) ]\hat{\rho}(2^m) \nonumber \\
    & \hspace{5mm} + \sum_{k = 1}^{m}\mathbbm{1}[ \hat{\rho}(2^{k}) \leq \Tilde{F}^{(\tau)} - {F}^{(\tau)}(\mathbf{x}^{(\tau), *}) \leq 12 \hat{\rho}(2^{k-1}) ] \hat{\rho}(2^{k-1})  \nonumber \\
    & \hspace{5mm} + \mathbbm{1}[\Tilde{F}^{(\tau)} - {F}^{(\tau)}(\mathbf{x}^{(\tau), *}) > 12 \hat{\rho}(0) ] \Big], \\
    & \leq 12 \Big[|\Tilde{\mathcal{I}}_i| \hat{\rho}(2^{m})  + \sum_{k = 1}^{m} \hat{\rho}(2^{k-1}) \zeta_i (k)  + \rho(1)\zeta_i(0) \Big], \\
    & \leq 12 |\Tilde{\mathcal{I}}_i| \hat{\rho}(2^{m}) + 24 \sum_{k = 0}^{m} \hat{\rho}(2^{k}) \zeta_i (k).
\end{align}
Summing over all the non-stationary intervals and observing that $\sum_{i = 1}^{\ell} |\Tilde{\mathcal{I}}_i| \leq 2^m$, we obtain:
\begin{align}
    \sum_{\tau = t_m}^{E_m} \big[ \Tilde{F}^{(\tau)} - {F}^{(\tau)}(\mathbf{x}^{(\tau), *}) \big] & \leq 12.2^n\hat{\rho}(2^{m}) + 24\sum_{k = 0}^{m} \sum_{i = 1}^{\ell} \hat{\rho}(2^{k}) \zeta_i (k), \\
    & = 12 \hat{C}(2^m) + 24\sum_{k = 0}^{m} \underbrace{\sum_{i = 1}^{\ell} \hat{\rho}(2^{k}) \zeta_i(k)}_\text{(a)}. \label{eqn: masterlemma16_dynamic_regr_1}
\end{align}
Next, we bound term (a) in Eq. \eqref{eqn: masterlemma16_dynamic_regr_1} for each $k$ as follows:
\begin{align}
    \sum_{i = 1}^{\ell} \hat{\rho}(2^{k}) \zeta_i(k) =  \underbrace{\sum_{i = 1}^{\ell} \hat{\rho}(2^{k}) \min \{\zeta_i(m), 4.2^k\}}_\text{(b)} + \underbrace{\sum_{i = 1}^{\ell} \hat{\rho}(2^{k})[\zeta_i(k) - 4.2^k]_{+}}_\text{(c)}~. \label{eqn: masterlemma16_dynamic_regr_2}
\end{align}
We bound term (b) in Eq. \eqref{eqn: masterlemma16_dynamic_regr_2} as follows:
\begin{align}
    \sum_{i = 1}^{\ell} \hat{\rho}(2^{k}) \min \{\zeta_i(k), 4.2^k\} & \leq 4 \sum_{i = 1}^{\ell} \hat{\rho}(2^{k}) \min \{\zeta_i(k), 2^k\}, \\
    & \leq 4 \sum_{i = 1}^{\ell} \hat{\rho}(\min \{\zeta_i(k), 2^k\}) \min \{\zeta_i(k), 2^k\}, \\
    & = 4 \sum_{i = 1}^{\ell} \hat{C} (\min \{\zeta_i(k), 2^k\}), \\
    & \leq 4 \sum_{i = 1}^{\ell} \hat{C} (|\Tilde{\mathcal{I}}_i|).
\end{align}
We bound the term (c) in Eq. \eqref{eqn: masterlemma16_dynamic_regr_2} in Lemma \ref{lemma: masterlemma17}, thereby obtaining Eq. \eqref{eqn: block_regr_2} - \eqref{eqn: single_block_dynamic_reg-I} as specified in the Lemma statement.
\end{proof}
\begin{lemma} \label{lemma: masterlemma17}
Let the high probability events described in Lemma \ref{lemma:multi_scale_regr_main} hold, then the following holds with high probability:
\begin{align}
    \sum_{i = 1}^{\ell} \hat{\rho}(2^{k})[\zeta_i(k) - 4.2^k]_{+} \leq \frac{2\rho(2^k)}{\rho(2^m)}\hat{C}(2^k) \log (T/\delta).
\end{align}
\end{lemma}
\begin{proof}
\begin{align}
    [\zeta_i(k) - 4.2^k]_{+} = [\Tilde{e}_i - \tau_i(k) + 1 - 4.2^k]_{+} \label{eqn: masterlemma17_eq1}
\end{align}
In the following analysis, we first focus on the quantity: $Q_i$: the total number of rounds in the interval $[\tau_i(k), \Tilde{e}_i - 2.2^k]$ where an order-$k$ base FL algorithm can potentially get scheduled. To this end, we note that the Randomized Scheduling Procedure (Algorithm \ref{scheduling_algo}) allows us to provide a lower bound for $Q_i$ as:
\begin{align}
    Q_i \triangleq \sum_{t \in \mathcal{I}_i} \mathbbm{1} \big[ t \in [\tau_i(k), \Tilde{e}_i - 2.2^k], ~\text{and} ~(t - t_m) ~\text{mod} ~2^k = 0 \big] \geq \frac{[\Tilde{e}_i - \tau_i(k) + 1 - 4.2^k]_{+}}{2^k}, \label{eqn: masterlemma17_eq2}
\end{align}
We note that RHS of Eq. \eqref{eqn: masterlemma17_eq2} is precisely the term we are interested in the Lemma as corroborated by Eq. \eqref{eqn: masterlemma17_eq1}. Henceforth, we aim to provide an upper bound for LHS of Eq. \eqref{eqn: masterlemma17_eq2}. Formally, we define the following set of events on the FL round index $t$:
\begin{align*}
    & W_t = \{t \in \mathcal{I}_i ~\text{and} ~t \in [\tau_i(k), e_i - 2.2^k] \}, \nonumber \\
    & X_t = \{t \leq E_m -2.2^k \}, \nonumber \\
    & Y_t = \{t \leq E_m  ~\text{and} ~(t - t_m) ~\text{mod} ~2^k = 0\}, \nonumber \\
    & Z_t = \{ \exists ~\text{order}-k ~\text{base FL algorithm with start point scheduled at} ~t \}, \nonumber \\
    & V_t = \{ \exists \tau \in [t_n, t] ~\text{s.t.} W_{\tau} \cap Y_{
    \tau} \cap Z_{\tau} \neq \emptyset \}, \nonumber \\
\end{align*}
Therefore, we can express the summation of quantity $Q_i$ over $i \in [\ell]$ as:
\begin{align}
    \sum_{i = 1}^{\ell} Q_i = \sum_{i = 1}^{K} Q_i = \sum_{t = t_m}^{t_m + 2^m - 1} \mathbbm{1} \big[W_t \cap X_t \cap Y_t \big] \leq \underbrace{\sum_{t = t_m}^{t_m + 2^m - 1} \mathbbm{1} \big[W_t \cap Y_t \cap \overline{V}_t \big]}_\text{(a)} + \underbrace{\sum_{t = t_m}^{t_m + 2^m - 1} \mathbbm{1} \big[X_t \cap V_t \big]}_\text{(b)}. \label{eqn: masterlemma17_eq3}
\end{align}
For term (a), we first highlight that $Z_t$ happens with probability $\frac{\rho(2^m)}{\rho(2^k)}$ given $W_t \cap Y_t$ owing to Randomized Scheduling Procedure (Algorithm \ref{scheduling_algo}). Hence, term (a) is the count of number of trials needed for the first order-$k$ algorithm to be scheduled with probability of success being $\frac{\rho(2^m)}{\rho(2^k)}$. Hence, with probability $1-\frac{\delta}{T}$, we have the following bound for term (a):
\begin{align}
    \sum_{t = t_m}^{t_m + 2^m - 1} \mathbbm{1} \big[W_t \cap Y_t \cap \overline{V}_t \big] \leq 1 + \frac{\log (T/\delta)}{ - \log \big(1- \frac{\rho(2^m)}{\rho(2^k)} \big)} \leq \frac{2\rho(2^k)}{\rho(2^m)} \log (T/\delta). \label{eqn: masterlemma17_eq7}
\end{align}
In order to bound term (b) in Eq. \eqref{eqn: masterlemma17_eq3}, we first note that event $V_t$ corresponds to existence of an order-$k$ base FL algorithm $\mathcal{A}$ such that $\mathcal{A}.s = t^{*}, ~t^{*} \leq t$ and $\tau_i (k) \leq t^{*} \leq e_i -2.2^k$. Furthermore, $\mathcal{A}.e = \mathcal{A}.s + 2^k - 1 = t^{*} + 2^k - 1 \leq e_i - 2^k - 1 \leq e_i$, which implies $[\mathcal{A}.s, \mathcal{A}.e] \subseteq \mathcal{I}_i$. Consequently, $X_t \cap V_t$ mean that $\mathcal{A}.e = \mathcal{A}.s + 2^k - 1 \leq t + 2^k - 1 < E_m$, hence at $t = \mathcal{A}.e$ the larger order-$m$ block is still running. As a result, \textbf{Test 1} is performed within the ongoing order-$m$ block. Due to Lemma \ref{lemma:multi_scale_regr_main}, the following holds:
\begin{align}
    \frac{1}{2^k} \sum_{\tau = \mathcal{A}.s}^{\mathcal{A}.e} F^{(\tau)}(x^{(\tau)}) & \leq \frac{1}{2^k} \sum_{\tau = alg.s}^{alg.e} \Tilde{F}^{(\tau)} + \hat{\rho}(2^k) + \hat{n}\Delta_{[alg.s, alg.e]}, \\
    & \leq \underset{\tau \in \mathcal{I}_i}{\max} ~F^{(\tau)}(x^{(\tau),*}) + \hat{\rho}(2^k) + (\hat{n} + 1)\Delta_{\mathcal{I}_i}, \label{eqn: masterlemma17_eq4} \\
    & = F^{(\tau_{max})}(x^{(\tau_{max}),*}) + \hat{\rho}(2^k) + (\hat{m} + 1)\Delta_{\mathcal{I}_i}, \\
    & \leq F^{(\tau_{max})}(x^{*})+ \hat{\rho}(2^k) + (\hat{m} + 1)\Delta_{\mathcal{I}_i}, \\
    & \leq F^{(\tau_i(k))}(x^{*}) + \hat{\rho}(2^k) + (\hat{m} + 2)\Delta_{\mathcal{I}_i}, \label{eqn: masterlemma17_eq5} \\
    & \leq F^{(\tau_i(k))}(x^{(\tau_i(k)),*}) + \hat{\rho}(2^k) + (\hat{m} + 4)\Delta_{\mathcal{I}_i}, \label{eqn: masterlemma17_eq5_}
\end{align}
where Eq. \eqref{eqn: masterlemma17_eq4} is a consequence of $[\mathcal{A}.s, \mathcal{A}.e] \subseteq \mathcal{I}_i$ and Eq. \eqref{eqn: masterlemma17_eq5} holds due to the fact that $|\underset{\tau \in \mathcal{I}_i}{\max}~F^{(\tau)}(x^{*}) - F^{(\tau_i(k))}(x^{*})| \leq \Delta_{\mathcal{I}_i}$. Eq. \eqref{eqn: masterlemma17_eq5_} holds due to the exact same set of arguments presented detailed in Theorem \ref{thm: fedavg_static_regr_main} via Eq. \eqref{eqn: term b 1_main} - \eqref{eqn: term b 5_main}.\\

In Eq. \eqref{eqn: masterlemma17_eq5_} using the definition of $\tau_i (k)$ (see Definition \ref{def: tau_i_m}) and the fact that $\Delta_{\mathcal{I}_i} \leq \rho(|\mathcal{I}_i|) \leq \rho(2^k) \leq \frac{\hat{\rho}(2^k)}{6\hat{m}}$, we obtain:
\begin{align}
   & \frac{1}{2^k} \sum_{\tau = \mathcal{A}.s}^{\mathcal{A}.e} F^{(\tau)}(\mathbf{x}^{(\tau)})  \leq \Tilde{F}^{(\tau_i (k))} - 10 \hat{\rho}(2^k), \\
    & \frac{1}{2^k} \sum_{\tau = \mathcal{A}.s}^{\mathcal{A}.e} F^{(\tau)}(\mathbf{x}^{(\tau + 1)}) + 10 \hat{\rho}(2^k) \leq U_{\mathcal{A}.e}, \label{eqn: masterlemma17_eq6}
\end{align}
Eq. \eqref{eqn: masterlemma17_eq6} is owed to $\mathcal{A}.e \geq \tau_i(k)$ and by definition of $U_{\tau}$ provided in {Master-FL}. We highlight that Eq. \eqref{eqn: masterlemma17_eq6} must trigger a restart at time $\mathcal{A}.e < E_m$ due to \textbf{Test 1} in Master-FL rendering  $\mathbbm{1} \big[X_t \cap V_t \big] = 0$. Hence, combining the results in Eq. \eqref{eqn: masterlemma17_eq1} - \eqref{eqn: masterlemma17_eq7}, we finally obtain :
\begin{align}
    \sum_{i = 1}^{\ell} \hat{\rho}(2^{k})[\zeta_i(k) - 4.2^k]_{+} &= \sum_{i = 1}^{\ell} \hat{\rho}(2^{k}) [\Tilde{e}_i - \tau_i(k) + 1 - 4.2^k]_{+}, \\
    & \leq \hat{\rho}(2^{k}) 2^k \sum_{i = 1}^{\ell} Q_i, \\
    & = \hat{C}(2^k) \sum_{i = 1}^{\ell} Q_i, \\
    & \leq \frac{2 \rho(2^k)}{\rho(2^m)} \hat{C}(2^k) \log (T/\delta).
\end{align}
This proves the result as in the statement of the Lemma.
\end{proof}
\clearpage
\begin{lemma} \label{lemma: ell_delta_bound}
Let $\mathcal{B} = [t_m, E_m]$ denotes the order-$m$ block for which {Master-FL} (Algorithm \ref{detection}) is executed. Then $\ell$ as described in Definition \ref{def: block_endpoints_def} satisfies:
\begin{align}
  & \ell \leq L_{\mathcal{B}}, \label{eqn: ell_L_bound}\\
    & \mathcal{\ell} \leq 1 + 2 c_{1}^{-\frac{2}{3}}\Delta_{\mathcal{B}}^{\frac{2}{3}}{|\mathcal{B}|}^{\frac{1}{3}} + \Delta_{\mathcal{B}}, \label{eqn: ell_Delta_bound}
\end{align}
\end{lemma}
where $t_m, E_m, \ell$ are as described in Definition \ref{def: block_endpoints_def} and $L_{\mathcal{B}}, \Delta_{\mathcal{B}}$ follows the definition of $L, \Delta$ (see Definition \ref{defn: model drift}, \ref{defn: num_drifts}) for interval $\mathcal{B}$. 
\begin{proof}
In order to verify the first claim in the Lemma statement, it is straightforward to check that one possible approach to ensure that Eq. \eqref{eqn: delta_rho_reln_1} holds true is to construct each such $\mathcal{I}_i$ as stationary. Therefore, each such interval will have $\Delta_{\mathcal{I}_i} = 0$. Hence, such construction directly implies $\ell \leq L_{\mathcal{B}}$.   
\\

In order to corroborate the second claim of the Lemma, we note that $\mathcal{B}$ can be alternatively partitioned into constituent stationary intervals $\mathcal{I}_1, \mathcal{I}_2, \cdots, \mathcal{I}_{\ell}$ where $\mathcal{I}_i = [s_i, e_i]$ such that $\Delta_{[s_i, e_i]} \leq \rho (e_i - s_i + 1)$, $\Delta_{[s_i, e_i + 1]} \geq \rho (e_i - s_i + 2)$, except for the last interval i.e. $i \leq \ell - 1$. Therefore, by definition of $\Delta_{\mathcal{B}}$, we have:
\begin{align}
    \Delta_{\mathcal{B}} &\geq \sum_{i = 1}^{\ell - 1} \Delta_{[s_i, e_i]}, \\
    & \geq \sum_{i = 1}^{\ell - 1} \rho (e_i - s_i + 2), \\
    & \geq \sum_{i = 1}^{\ell - 1} \min \{c_1{(e_i - s_i + 2)}^{- \frac{1}{2}} , 1\}, \\
    & \geq \sum_{i = 1}^{\ell - 1} \min \{\frac{c_1}{2}{(e_i - s_i + 1)}^{- \frac{1}{2}} , 1\}, \\
    & = \frac{c_1}{2} \sum_{i = 1}^{{\ell}_1}{(e_i - s_i + 1)}^{- \frac{1}{2}} + \sum_{i = 1}^{{\ell}_2} 1 ~. \label{eqn: ell_bound_part1}
\end{align}
where, we partition $\ell_1 + \ell_2 = \ell - 1$. Since, each term in \eqref{eqn: ell_bound_part1} is bounded by $\Delta_{\mathcal{B}}$, we have $\ell_2 \leq \Delta_{\mathcal{B}}$. Now, by leveraging Holder's inequality, we get:
\begin{align}
    \ell_1 &\leq \Big(\sum_{i = 1}^{\ell_1} (e_1 - s_1 + 1)^{-\frac{1}{2}} \Big)^{\frac{2}{3}} \Big(\sum_{i = 1}^{\ell_1} (e_1 - s_1 + 1) \Big)^{\frac{1}{3}}, \\
    & \leq 2 c_{1}^{-\frac{2}{3}}\Delta_{\mathcal{B}}^{\frac{2}{3}}{|\mathcal{B}|}^{\frac{1}{3}}.
\end{align}
Combining bounds for $\ell_1$, $\ell_2$ and using the fact that $\ell_1 + \ell_2 = \ell - 1$, we get the bound for $\ell$ as follows:
\begin{align}
    \mathcal{\ell} \leq 1 + 2 c_{1}^{-\frac{2}{3}}\Delta_{\mathcal{B}}^{\frac{2}{3}}{|\mathcal{B}|}^{\frac{1}{3}} + \Delta_{\mathcal{B}}.
\end{align}
\end{proof}
\subsection{Proof of Lemma \ref{eqn: blk_dyn_regr_main}} \label{sec: block_regr_lemma_subsec}
First, we note that this Lemma extends the results produced in Lemma \ref{lemma: single_block_regret_lemma-I}, wherein we provide more concrete bounds for \textit{dynamic regret} of each order $m$ block while precisely tying to non-stationary measures, i.e., $L, \Delta$. To this end, using single block \textit{dynamic regret} result stated in lemma \ref{lemma: single_block_regret_lemma-I} pertaining to {Master-FL} as Eq. \eqref{eqn: single_block_dynamic_reg-I}, we have:
\begin{align}
    {R}_{\mathcal{B}} \leq \Tilde{\mathcal{O}} \Big(\underbrace{\sum_{i = 1}^{\ell} C(|\mathcal{I}_{i}|)}_{\text{(a)}} + \underbrace{\sum_{k = 0}^{m} \frac{\rho(2^k)}{\rho(2^m)} C(2^k)}_{\text{(b)}} \Big). \label{eqn: single_block_dynamic_reg}
\end{align}
Next, we proceed to separately bound terms (a), (b) in \eqref{eqn: single_block_dynamic_reg}. Firstly, we bound term (a) as follows:
\begin{align}
    \Tilde{\mathcal{O}} \Big({\sum_{i = 1}^{\ell} C(|\mathcal{I}_{i}|)}\Big) & = \Tilde{\mathcal{O}} \Big(\sum_{i = 1}^{\ell} \min \{c_1\sqrt{|\mathcal{I}_{i}|} + c_2, t \} \Big), \\
    & \leq \Tilde{\mathcal{O}}\Big(\sum_{i = 1}^{\ell} (c_1\sqrt{|\mathcal{I}_{i}|} + c_2) \Big), \\
    & \leq \Tilde{\mathcal{O}}\Big(c_1 \sqrt{\ell |\mathcal{I}_{i}|} + c_2\ell \Big). \label{eqn: term_a_bound}
\end{align}
Plugging in \eqref{eqn: ell_L_bound} into \eqref{eqn: term_a_bound}, we obtain ${R}_{L}(\mathcal{B})$. Using  \eqref{eqn: ell_Delta_bound} we obtain ${R}_{\Delta}(\mathcal{B})$.
Hence, we have:
\begin{align}
    \Tilde{\mathcal{O}} \Big({\sum_{i = 1}^{\ell} C(|\mathcal{I}_{i}|)}\Big) \leq \Tilde{\mathcal{O}} \Big( \min \Big\{ {R}_{L}(\mathcal{B}), {R}_{\Delta}(\mathcal{B}) \Big \} \Big). \label{eqn: term_a_final}
\end{align}
Now, in order to bound term (b), we proceed as follows:
\begin{align}
    \frac{\rho(2^k)}{\rho(2^m)} C(2^k) &= \frac{C(2^k)^2}{C(2^m)}2^{m-k}, \\
    & = \mathcal{O} \Big(\frac{\min\{c_1^2 2^k + c_2^2,2^{2k} \}}{c_1 2^{m/2} + c_2}2^{m-k} + \frac{\min\{c_1^2 2^k + c_2^2,2^{2k} \}}{2^{m}}2^{m-k} \Big), \\
    & = \mathcal{O}\Big(c_1 2^{m/2} + \min \{ \frac{c_2^2}{c_1}2^{m/2 - k}, \frac{1}{c_1}2^{m/2 + k} \}  + c_2^2 2^{-k}  \Big), \\
    & = \mathcal{O} \Big(c_1 2^{m/2} + \frac{c_2}{c_1}2^{m/2}  + c_2^2 2^{-k} \Big). \label{eqn: term_b_bound}
\end{align}
Hence, for term (b), we obtain the following bound:
\begin{align}
     \Tilde{\mathcal{O}}\Big( \sum_{k = 0}^{m} \frac{\rho(2^k)}{\rho(2^m)} C(2^k) \Big) \leq \Tilde{\mathcal{O}} \Big(\Big( c_1 + \frac{c_2}{c_1} \Big) 2^{m/2} + c_{2}^{2}  \Big). \label{eqn: term_b_final}
\end{align}
Combining \eqref{eqn: term_a_final}, \eqref{eqn: term_b_final}, we obtain the final expression for \textit{dynamic regret} for block $\mathcal{B}$ as:
\begin{align}
   {R}_{\mathcal{B}} \leq \Tilde{\mathcal{O}} \Big( \min \Big\{ {R}_{L}(\mathcal{B}), {R}_{\Delta}(\mathcal{B}) \Big \} + \Big( c_1 + \frac{c_2}{c_1} \Big) 2^{m/2} + c_{2}^{2} \Big).
\end{align}
\clearpage
\section{Single Epoch Regret Analysis, Correctness of Stationarity tests and Bound on epochs} \label{sec: epoch_regr_analysis_desc}
In Section \ref{lemma: single_epoch_regr_discussion}, we derive the \textit{dynamic regret} bound for a single epoch as claimed in Lemma \ref{lemma: single_epoch_regr_main}. Next, in Section \ref{sec: test_trigger_correctness_desc}, we prove the statistical correctness of \textbf{Test 1} and \textbf{Test 2} triggering events as stated in Lemma \ref{lemma: test_trigger_correctness_main}. In Section \ref{sec: bound_epoch_num_main_desc}, we formally bound the number of epochs within $T$ FL rounds as asserted by Lemma \ref{lemma: bound_epoch_num_main}. 
\subsection{Proof of Lemma \ref{lemma: single_epoch_regr_main}} \label{lemma: single_epoch_regr_discussion}
Let, $\mathcal{B}_1, \mathcal{B}_2, \cdots, \mathcal{B}_m$ be the blocks contained in epoch $\mathcal{E}$. Clearly, $|\mathcal{E}| = \Theta(2^m)$. Therefore, as a consequence of Lemma \ref{eqn: blk_dyn_regr_main}, the \textit{dynamic regret} of epoch $\mathcal{E}$ is bounded as:
\begin{align}
   {R}_{\mathcal{E}} \leq \Tilde{\mathcal{O}} \Big(\underbrace{\min \{\sum_{k = 0}^{m} {R}_{L}(\mathcal{B}_k), \sum_{k = 0}^{m} {R}_{\Delta}(\mathcal{B}_k) \}}_\text{(a)} + \underbrace{\Big(c_1 + \frac{c_2}{c_1} \Big)\sum_{k = 0}^{m} 2^{k/2} + \sum_{k = 0}^{m} c_2^2}_\text{(b)} \Big). \label{eqn: single_epoch_regr_eq1}
\end{align}
Using Holder's inequality, term (a) in Eq. \eqref{eqn: single_epoch_regr_eq1} can be bounded as follows:
\begin{align}
    \sum_{k = 0}^{m} {R}_{L}(\mathcal{I}_k) &\leq c_1 \sqrt{ \Big( \sum_{k = 0}^{m} L_{\mathcal{B}_k} \Big) \Big( \sum_{k = 0}^{m} |\mathcal{B}_k| \Big)} + c_2 \sum_{k = 0}^{m} L_{\mathcal{B}_k}, \label{eqn: single_epoch_regr_eq2}\\
    & \leq c_1 \sqrt{(L_{\mathcal{E}} + m) |\mathcal{E}|} + c_2(L_{\mathcal{E}} + m), \\
    &\leq \Tilde{\mathcal{O}} \Big(c_1 \sqrt{L_{\mathcal{E}} |\mathcal{E}|} + c_2 L_{\mathcal{E}} \Big), \\
    & = \Tilde{\mathcal{O}} \Big( {R}_{L}(\mathcal{E}) \Big).
\end{align}
Whereas, for the $R_{\Delta}(\cdot)$ component in term (a) of Eq. \eqref{eqn: single_epoch_regr_eq1}, we obtain:
\begin{align}
    \sum_{k = 0}^{m} {R}_{\Delta}(\mathcal{B}_k) = \Tilde{\mathcal{O}} \Big( {R}_{\Delta}(\mathcal{E}) \Big).
\end{align}
Since $m = \mathcal{O}(\log T)$, term (b) in Eq. \eqref{eqn: single_epoch_regr_eq1} can be bounded as follows:
\begin{align}
    \Big(c_1 + \frac{c_2}{c_1} \Big)\sum_{k = 0}^{m} 2^{k/2} + \sum_{k = 0}^{m} c_2^2 &= \Tilde{\mathcal{O}} \Big( \Big(c_1 + \frac{c_2}{c_1} \Big)2^{m/2} + {c_2^2} \Big), \\
    & = \Tilde{\mathcal{O}} \Big( \Big(c_1 + \frac{c_2}{c_1} \Big)\sqrt{|\mathcal{E}|} + {c_2^2} \Big). \label{eqn: single_epoch_regr_eq3}
\end{align}
Hence, using bounds for terms (a) and (b) as obtained via Eq. \eqref{eqn: single_epoch_regr_eq2}-\eqref{eqn: single_epoch_regr_eq3} in Eq. \eqref{eqn: single_epoch_regr_eq1}, we get the \textit{dynamic regret} associated with a single epoch run bounded as:
\begin{align}
    {R}_{\mathcal{E}} \leq \Tilde{\mathcal{O}} \Big(\min \Big\{ {R}_{L}(\mathcal{E}), {R}_{\Delta}(\mathcal{E}) \Big \} + (c_1+ \frac{c_2}{c_1}) \sqrt{|\mathcal{E}|} + c_2^2 \Big).
\end{align}
\subsection{Proof of Lemma \ref{lemma: test_trigger_correctness_main}} \label{sec: test_trigger_correctness_desc}
In this proof, we aim towards verifying that \textbf{Test 1} does not fail with high probability. Let $t = \mathcal{A}.e$ where $\mathcal{A}$ is any order-$k$ base FL algorithm instance triggered within a block of order-$m$. Furthermore, we denote the start time of this particular order-$m$ block as $t_m$. Therefore, with probability $1- \frac{\delta}{T}$,
\begin{align}
   U_{t} &= \underset{\tau \in [t_m, t]}{\max} \Tilde{F}^{(\tau)}, \\
   & \leq \underset{\tau \in [t_m, t]}{\max} F^{(\tau)}(\mathbf{x}^{(\tau),*}) + \Delta_{[t_m, t]}, \label{eqn: test_trigger_correctness_0}\\
   & = F^{(\tau_{max})}(\mathbf{x}^{(\tau_{max}),*}) + \Delta_{[t_m, t]}, \\
   & \leq F^{(\tau_{max})}(\mathbf{x}^{*}) + \Delta_{[t_m, t]}, \label{eqn: test_trigger_correctness_01} \\
   & \leq \frac{1}{2^m} \sum_{\tau = \mathcal{A}.s}^{t} F^{(\tau)}(\mathbf{x}^{*}) + 2\Delta_{[t_m, t]},  \label{eqn: test_trigger_correctness_02}\\
   & \leq \frac{1}{2^m} \sum_{\tau = \mathcal{A}.s}^{t} F^{(\tau)}(\mathbf{x}^{(\tau), *}) + 4\Delta_{[t_m, t]}, \label{eqn: test_trigger_correctness_1} \\
   & \leq \frac{1}{2^k} \sum_{\tau = \mathcal{A}.s}^{t} F^{(\tau)}(\mathbf{x}^{(\tau)}) + 2\sqrt{\frac{\log (T/\delta)}{2^k}} + 4\rho(t - t_0 + 1), \label{eqn: test_trigger_correctness_2}\\
   & \leq \frac{1}{2^k} \sum_{\tau = \mathcal{A}.s}^{t} F^{(\tau)}(\mathbf{x}^{(\tau)}) + \hat{\rho}(2^k) + 4\rho(t - t_0 + 1), \label{eqn: test_trigger_correctness_3}\\
   & \leq \frac{1}{2^k} \sum_{\tau = \mathcal{A}.s}^{t} F^{(\tau)}(\mathbf{x}^{(\tau)}) + 2\hat{\rho}(2^k), \label{eqn: test_trigger_correctness_4}
\end{align}

where Eq. \eqref{eqn: test_trigger_correctness_0} is a direct consequence of Eq. \eqref{eqn: opt_bound_eq1_main} in Lemma \ref{lemma:multi_scale_regr_main}. Furthermore, Eq. \eqref{eqn: test_trigger_correctness_01} holds because of the optimality of the instantaneous comparators $\{\mathbf{x}^{(\tau,*)}\}$ over static comparator $\mathbf{x}^{*}$. Subsequently, we use the definition of $\Delta$ (see Definition \ref{defn: model drift}) thereby obtaining Eq. \eqref{eqn: test_trigger_correctness_02}. Next, the arguments presented in Eq. \eqref{eqn: term b 1_main} - \eqref{eqn: term b 5_main} allows us to obtain  Eq. \eqref{eqn: test_trigger_correctness_1}. Thereafter, we use Azuma-Hoeffding's inequality in conjunction with the fact that $\mathbbm{E}[F^{(\tau)}(\mathbf{x}^{(\tau)})] \geq F^{(\tau)}(\mathbf{x}^{(\tau), *})$ to obtain Eq. \eqref{eqn: test_trigger_correctness_2}. In Eq. \eqref{eqn: test_trigger_correctness_3}-\eqref{eqn: test_trigger_correctness_4}, we use the definition of $\hat{\rho}(t)$ (see Definition \ref{defn: rho_def}) and the fact that $\rho(\cdot)$ is a decreasing function. This concludes our proof investigating the correctness of trigger events pertaining to \textbf{Test 1} of Master-FL algorithm.\\

Moreover, as a consequence of Eq. \eqref{eqn: opt_bound_eq2_main} in Lemma \ref{lemma:multi_scale_regr_main}, with probability $1-\frac{\delta}{T}$, the following holds:
\begin{align}
    \frac{1}{t -t_m + 1} \sum_{\tau = t_m}^{t} F^{(\tau)}(\mathbf{x}^{(\tau)}) - \Tilde{F}^{(\tau)} \leq \hat{\rho}(t - t_m + 1) + \Delta_{[t_m,t]} \leq 2\hat{\rho}(t - t_m + 1),
\end{align}
which implies correctness of triggering \textbf{Test 2}. 
\subsection{Proof of Lemma \ref{lemma: bound_epoch_num_main}} \label{sec: bound_epoch_num_main_desc}
We note that the proof of Eq. \eqref{eqn: master_N_L_bound_main} in Lemma \ref{lemma: bound_epoch_num_main} is a simple extension of Eq. \eqref{eqn: ell_L_bound} in Lemma \ref{lemma: ell_delta_bound}, therefore requires the exact same set of arguments.
\\

In order to prove Eq. \eqref{eqn: master_N_Delta_bound_main}, we proceed similar to Lemma \ref{lemma: ell_delta_bound}. The time horizon over which {Master-FL} is run as provided in the Lemma statement is $[t_0, E]$ (i.e. $T = E -t_0$, is the total number of FL rounds). If $[t_0, E]$ is not the last epoch over which {Master-FL} is run, then $\Delta_{[t_0, E]} > \rho(E -t_0 + 1)$ must hold with high probability due to Lemma \ref{lemma: test_trigger_correctness_main}. To this end, we construct partitions of $[t_0, E]$ into constituent stationary epochs $\mathcal{E}_1, \mathcal{E}_2, \cdots, \mathcal{E}_{M}$ where $\mathcal{E}_i = [s_i, e_i]$ such that $\Delta_{[s_i, e_i]} \leq \rho (e_i - s_i + 1)$, $\Delta_{[s_i, e_i + 1]} \geq \rho (e_i - s_i + 2)$, except for the last epoch i.e. $i \leq M - 1$. Therefore, by definition of $\Delta_{[t_0, E]}$, we have:
\begin{align}
    \Delta_{[t_0, E]} &\geq \sum_{i = 1}^{M - 1} \Delta_{[s_i, e_i]}, \\
    & \geq \sum_{i = 1}^{M - 1} \rho (e_i - s_i + 2), \\
    & \geq \sum_{i = 1}^{M - 1} \min \{c_1{(e_i - s_i + 2)}^{- \frac{1}{2}} , 1\}, \\
    & \geq \sum_{i = 1}^{M - 1} \min \{\frac{c_1}{2}{(e_i - s_i + 1)}^{- \frac{1}{2}} , 1\}, \\
    & = \frac{c_1}{2} \sum_{i = 1}^{{M}_1}{(e_i - s_i + 1)}^{- \frac{1}{2}} + \sum_{i = 1}^{{M}_2} 1 ~. \label{eqn: num_epoch_bound_part1}
\end{align}
wherein, we construct $M_1 + M_2 = M - 1$. We observe that each term in \eqref{eqn: num_epoch_bound_part1} is bounded by $\Delta_{[t_0, E]}$, so we have $M_2 \leq \Delta_{[t_0, E]}$. Furthermore, as a consequence of Holder's inequality, we obtain:
\begin{align}
    M_1 &\leq \Big(\sum_{i = 1}^{M_1} (e_i - s_i + 1)^{-\frac{1}{2}} \Big)^{\frac{2}{3}} \Big(\sum_{i = 1}^{M_1} (e_i - s_i + 1) \Big)^{\frac{1}{3}}, \\
    & \leq 2 c_{1}^{-\frac{2}{3}}\Delta^{\frac{2}{3}}{T}^{\frac{1}{3}}.
\end{align}
Combining bounds for $M_1$, $M_2$ and using the fact that $M_1 + M_2 = M - 1$, we get the bound for $M$ as follows:
\begin{align}
    M \leq 1 + 2 c_{1}^{-\frac{2}{3}}\Delta^{\frac{2}{3}}{T}^{\frac{1}{3}} + \Delta,
\end{align}
where $\Delta = \Delta_{[t_0,E]}$ is the cumulative \textit{concept drift} over $T = E-t_0$ rounds.
\newpage
\section{Dynamic Regret of Master-FL Algorithm - Proof of Theorem \ref{thm: final_dynamic_regr_master_main}} \label{sec: final_dynamic_regr_master}
First, we denote the number of epochs actually ran by the Master-FL in $[1, T]$ as $\{\mathcal{E}_i \}_{i=1}^{i=M}$. Therefore, as a consequence of Lemma \ref{lemma: single_epoch_regr_main}, the \textit{dynamic regret} incurred by Master-FL is bounded by:
\begin{align}
   {R}_{[1,T]} &=  \Tilde{\mathcal{O}} \Big( \sum_{i = 1}^{M} \min \{{R}_L(\mathcal{E}_i) , {R}_{\Delta}(\mathcal{E}_i) \} + \frac{c_2}{c_1} \sum_{i = 1}^{M} \sqrt{|\mathcal{E}_i|} + c_2^{2}M \Big),  \\
   & = \Tilde{\mathcal{O}} \Big(  \min \{  \underbrace{\sum_{i = 1}^{M} {R}_L(\mathcal{E}_i)}_\text{(a)} ,  \underbrace{\sum_{i = 1}^{M}  {R}_{\Delta}(\mathcal{E}_i)}_\text{(b)}  \} + \underbrace{\frac{c_2}{c_1} \sum_{i = 1}^{M} \sqrt{|\mathcal{E}_i|}}_\text{(c)} + \underbrace{c_2^{2}M}_\text{(d)} \Big).
\end{align}
We note that due to the stationarity of the epochs $\{\mathcal{E}_i \}_{i=1}^{i=M}$, the following holds:
\begin{align}
    \sum_{i = 1}^{M} L_{\mathcal{E}_i} \leq L + M-1, \label{eqn: epoch_L_bound_main_thm}
\end{align}
Therefore, as a consequence of Holder's inequality in conjunction with Eq. \eqref{eqn: epoch_L_bound_main_thm}, we first bound (a) as follows:
\begin{align}
    \sum_{i = 1}^{M}{R}_L(\mathcal{E}_i) &\leq \Tilde{\mathcal{O}} \Big(c_1\sqrt{T(L+M-1)} + c_2(L+M-1) \Big), \\
    & \leq \Tilde{\mathcal{O}}\Big(c_1 \sqrt{LT} + c_2L \Big). \label{eqn: N_bound}
\end{align}
where Eq. \eqref{eqn: N_bound} is a consequence of Eq. \eqref{eqn: master_N_L_bound_main} in Lemma \ref{lemma: bound_epoch_num_main}.
In order to bound (b), we proceed as follows:
\begin{align}
    \sum_{i = 1}^{M} {R}_{\Delta}(\mathcal{E}_i) &\leq \Tilde{\mathcal{O}} \Big(c_1^{\frac{2}{3}}\Delta^{\frac{1}{3}} T^{\frac{2}{3}} + c_1\sqrt{MT} + c_1\sqrt{\Delta T} + {c_2}c_1^{-\frac{2}{3}}\Delta^{\frac{2}{3}}T^{\frac{1}{3}} + c_2(M + \Delta) \Big), \\
    &\leq \Tilde{\mathcal{O}} \Big( c_1^{\frac{2}{3}}\Delta^{\frac{1}{3}} T^{\frac{2}{3}} + c_1\sqrt{T} + c_1\sqrt{\Delta T} + {c_2}c_1^{-\frac{2}{3}}\Delta^{\frac{2}{3}}T^{\frac{1}{3}} + c_2\Delta \Big), \label{eqn: N_remove}
\end{align}
where Eq. \eqref{eqn: N_remove} is a consequence of Eq. \eqref{eqn: master_N_Delta_bound_main} in Lemma \ref{lemma: bound_epoch_num_main}. Next, using Cauchy–Schwarz inequality, we bound (c) as:
\begin{align}
    \frac{c_2}{c_1} \sum_{i = 1}^{M} \sqrt{|\mathcal{E}_i|} & \leq \frac{c_2}{c_1} \sqrt{MT}, \\
    & \leq \min \Big\{\frac{c_2}{c_1} \sqrt{LT} , \frac{c_2}{c_1}\sqrt{T} + c_2 c_1^{- \frac{4}{3}}\Delta^{\frac{1}{3}}T^{\frac{2}{3}} + \frac{c_2}{c_1} \sqrt{\Delta T} \Big\}, \label{eqn: T3_final_bound}
\end{align}
where Eq. \eqref{eqn: T3_final_bound} is a consequence of equations \eqref{eqn: master_N_L_bound_main} and \eqref{eqn: master_N_Delta_bound_main} from Lemma \ref{lemma: bound_epoch_num_main} which provide the bound on number of epochs $M$ expressed in terms of non-stationary measures $L, \Delta$, and horizon length $T$. Similarly, (d) is bounded as:
\begin{align}
    c_2^{2}M \leq \min \Big \{ c_2^{2}L, c_2^{2} +  c_{2}^{2} c_1^{-\frac{2}{3}}\Delta^{\frac{2}{3}}T^{\frac{1}{3}} + c_{2}^{2} \Delta  \Big \}.
\end{align}
Collecting only the terms that dominate in the exponent of $T$, we obtain the final \textit{dynamic regret} bound incurred by {Master-FL} as:
\begin{align}
    {R}_{[1,T]} = \Tilde{\mathcal{O}} \Big( \min \Big\{ \Big(c_1 + \frac{c_2}{c_1} \Big)\sqrt{LT} , c_2 c_1^{- \frac{4}{3}}\Delta^{\frac{1}{3}}T^{\frac{2}{3}}  + \sqrt{T} \Big\} \Big). 
\end{align}
\newpage
{\color{black} \section{Extension of Current Algorithmic Framework to a Stochastic Model Update Setting} \label{app: mini-batch-explanation}
In order to extend our framework to a mini-batch stochastic gradient update setting, note that it is necessary to understand how the performance of base FL algorithms i.e., \textit{FedAvg}, \textit{FedOMD} would change in such a setting. In the stochastic setting, the definition of \textit{dynamic regret} is as follows:
\begin{align}
   R_{[1,T]} = \sum_{t = 1}^{T} \mathbb{E}\big[{F}^{(t)}(\mathbf{x}^{(t)})\big] ~-  \sum_{t = 1}^{T} {F}^{(t)}(\mathbf{x}^{(t), *}), \label{defn: stoc dynamic regret}
\end{align}
where the expectation is w.r.t the mini-batches. In the ensuing discussion, we first identify the translations in terms of model update procedure for \textit{FedAvg, ~FedOMD} algorithms in a stochastic setting. For DPU $n \in \mathcal{N}$, with mini-batch fraction $\gamma_{n} \in [0,1]$, the stochastic FL model update is now denoted by \texttt{STOC-FL-UPDATE} for \textit{FedAvg} can be expressed as:
\begin{align}
    \texttt{STOC-FL-UPDATE:} \hspace{5mm} \mathbf{x}_{n}^{(t)} = \mathbf{x}^{(t-1)} - \eta_t \Tilde{\nabla} {F}_{n}^{(t)}(\mathbf{x}^{(t-1)}).  
\end{align}
And, for \textit{FedOMD}, the update expressions are:
\begin{align}
    & \texttt{STOC-FL-UPDATE:} \hspace{5mm} \mathbf{x}^{(t+1)}_{n} = \underset{x \in \mathbb{R}^p}{\argmin} ~\psi_{n}(\mathbf{x}; \mathbf{x}^{(t)}), \label{eqn:bregman_minimize_func_main} \\
    & \hspace{5mm} \psi_{n}(\mathbf{x}; \mathbf{x}^{(t)}) \triangleq \langle \Tilde{\nabla} {F}_{n}^{(t)}(\mathbf{x}^{(t)}), \mathbf{x} \rangle + \frac{1}{\eta_t} B_{\phi} (\mathbf{x}; \mathbf{x}^{(t)}). \label{eqn:bregman_loss_func_defn_main}
\end{align}
wherein we denote $\Tilde{\nabla} {F}(\cdot)$ to denote the mini-batch gradients. We remark that ${\nabla} {F}(\cdot) \rightarrow \Tilde{\nabla} {F}(\cdot)$ is the only change upon switching to a stochastic setting, while the original \texttt{FL-UPDATE} expressions for \textit{FedAvg, ~FedOMD} ( eq. 17, 42, 43 in the manuscript) remain unchanged otherwise. At round $t$, recall that the full-batch gradient for DPU $n$ at any arbitrary model $\mathbf{x}$ is given by:
\begin{align}
    {\nabla} {F}_{n}^{(t)}(\mathbf{x}) = \frac{1}{|{{\mathcal{D}}}_{{n}}^{(t)}|}{\underset{\xi \in {{\mathcal{D}}}_{{n}}^{(t)}}{\sum} \nabla {f}(\mathbf{x}} ; \xi),
\end{align}
where $f(\cdot~;~\cdot)$ is the underlying ML loss function. For the stochastic gradient, let us denote $\Tilde{{\mathcal{D}}}_{{n}}^{(t)}$ as the mini-batch dataset of size ratio $\gamma_{n}$ generated via random sampling without replacement from ${{\mathcal{D}}}_{{n}}^{(t)}$ i.e., $\Tilde{{\mathcal{D}}}_{{n}}^{(t)} \subseteq {{\mathcal{D}}}_{{n}}^{(t)}$ and $|\Tilde{{\mathcal{D}}}_{{n}}^{(t)}| = \gamma_{n} |{{\mathcal{D}}}_{{n}}^{(t)}|$. Essentially, the mini-batch gradient is:
\begin{align}
    \Tilde{{\nabla}} {F}_{n}^{(t)}(\mathbf{x}) = \frac{1}{|\Tilde{{\mathcal{D}}}_{{n}}^{(t)}|}{\underset{\xi \in {\Tilde{\mathcal{D}}}_{{n}}^{(t)}}{\sum} \nabla {f}(\mathbf{x}} ; \xi),
\end{align}
For every data-point $\xi \in {{\mathcal{D}}}_{{n}}^{(t)}$, assign a indicator random variable $Z_{\xi}$ that captures its inclusion in the mini-batch dataset $\Tilde{{\mathcal{D}}}_{{n}}^{(t)}$. Formally,
\begin{align}
& Z_{\xi} = 1, \  \ \xi \in \Tilde{{\mathcal{D}}}_{{n}}^{(t)}, \\
& Z_{\xi} = 0, \  \ \xi \notin \Tilde{{\mathcal{D}}}_{{n}}^{(t)}.
\end{align}
With the above definition of $Z_{\xi}$, we obtain the following:
\begin{align}
    & P[Z_{\xi} = 1] = \frac{{|\Tilde{{\mathcal{D}}}_{{n}}^{(t)}|-1 \choose |{{\mathcal{D}}}_{{n}}^{(t)}| -1}}{{|\Tilde{{\mathcal{D}}}_{{n}}^{(t)}| \choose |{{\mathcal{D}}}_{{n}}^{(t)}|}} = \frac{|\Tilde{{\mathcal{D}}}_{{n}}^{(t)}|}{|{{\mathcal{D}}}_{{n}}^{(t)}|} = \gamma_{n},
\end{align}
This allows us to calculate the expectation of stochastic gradient $\Tilde{{\nabla}} {F}_{n}^{(t)}(\mathbf{x})$ as follows:
\begin{align}
    \mathbb{E}\big[\Tilde{{\nabla}} {F}_{n}^{(t)}(\mathbf{x})\big] & = \mathbb{E}\Bigg[\frac{1}{|\Tilde{{\mathcal{D}}}_{{n}}^{(t)}|}{\underset{\xi \in {\Tilde{\mathcal{D}}}_{{n}}^{(t)}}{\sum} \nabla {f}(\mathbf{x}} ; \xi)\Bigg], \label{eq: stoc_grad_mean1} \\
    & = \mathbb{E}\Bigg[\frac{1}{|\Tilde{{\mathcal{D}}}_{{n}}^{(t)}|}{\underset{\xi \in {{\mathcal{D}}}_{{n}}^{(t)}}{\sum} Z_{\xi}\nabla {f}(\mathbf{x}} ; \xi)\Bigg], \\
    & = \frac{1}{|\Tilde{{\mathcal{D}}}_{{n}}^{(t)}|} {\underset{\xi \in {{\mathcal{D}}}_{{n}}^{(t)}}{\sum} \nabla {f}(\mathbf{x}} ; \xi) \mathbb{E}[Z_{\xi}], \\
    & = \frac{1}{\gamma_{n} |{{\mathcal{D}}}_{{n}}^{(t)}|} {\underset{\xi \in {{\mathcal{D}}}_{{n}}^{(t)}}{\sum} \nabla {f}(\mathbf{x}} ; \xi). \gamma_n, \\
    & = \frac{1}{|{{\mathcal{D}}}_{{n}}^{(t)}|}{\underset{\xi \in {{\mathcal{D}}}_{{n}}^{(t)}}{\sum} \nabla {f}(\mathbf{x}} ; \xi) = {\nabla} {F}_{n}^{(t)}(\mathbf{x}),
\end{align}
hence, the mini-batch stochastic gradient is indeed an unbiased estimator of the full-batch gradient. Furthemore, note that:
\begin{align}
    \| \Tilde{{\nabla}} {F}_{n}^{(t)}(\mathbf{x})\big \| &= \|\frac{1}{|\Tilde{{\mathcal{D}}}_{{n}}^{(t)}|}{\underset{\xi \in {\Tilde{\mathcal{D}}}_{{n}}^{(t)}}{\sum} \nabla {f}(\mathbf{x}} ; \xi) \|, \\
    &\leq  \frac{1}{|\Tilde{{\mathcal{D}}}_{{n}}^{(t)}|} {\underset{\xi \in {\Tilde{\mathcal{D}}}_{{n}}^{(t)}}{\sum} \|\nabla {f}(\mathbf{x}} ; \xi) \|, \\
    &\leq \frac{1}{|\Tilde{{\mathcal{D}}}_{{n}}^{(t)}|} \underset{\xi \in {\Tilde{\mathcal{D}}}_{{n}}^{(t)}}{\sum} \mu = \mu, \label{eq: grad_norm_lipschitz}
\end{align}
where, eq. \eqref{eq: grad_norm_lipschitz} is due to $\mu$-Lipschitz property of underlying ML loss function $f(\cdot;\cdot)$ (Assumption 1 in the manuscript). Therefore, the mini-batch gradient resembles the true gradient both in expectation and norm as seen from eq. \eqref{eq: stoc_grad_mean1} - 
\eqref{eq: grad_norm_lipschitz}.\\

In the \textit{dynamic regret} analysis of base FL algorithms \textit{FedAvg, ~FedOMD} i.e., (proof of Theorem 1 in Section V-A eq. 20-38, and proof of Theorem 2 in Appendix C eq. 66-84), a stochastic gradient based update approach implies the following changes throughout: ${\nabla} {F}_{n}^{(t)}(\mathbf{x}) \rightarrow \mathbb{E}\big[\Tilde{{\nabla}} {F}_{n}^{(t)}(\mathbf{x})\big]$, $\| {{\nabla}} {F}_{n}^{(t)}(\mathbf{x})\big \|   \rightarrow \| \Tilde{{\nabla}} {F}_{n}^{(t)}(\mathbf{x})\big \|$. However, since $\mathbb{E}\big[\Tilde{{\nabla}} {F}_{n}^{(t)}(\mathbf{x})\big] = {\nabla} {F}_{n}^{(t)}(\mathbf{x})$ and  $\| \Tilde{{\nabla}} {F}_{n}^{(t)}(\mathbf{x})\big \| \leq \mu $, which in other words mean that properties of the proxy for true gradient remain unchanged, there will be no impact for the overall \textit{dynamic regret} expressions for \textit{FedAvg, FedOMD} algorithms. Therefore, the interpretation and characterization (last part of Section V-A) of the vanilla base FL algorithms will remain unchanged in the stochastic setting as well.\\

Our remaining theoretical results pertaining to overall \textit{dynamic regret} incurred by integrating \textit{FedAvg} or \textit{FedOMD} with proposed Master-FL non-stationary learning setup (Algorithm \ref{scheduling_algo}- \ref{detection}) presented in Section \ref{sec: multi_scale_near_stationary_analysis} - \ref{sec: dynamic_regr_maintext_disc} also stays the same under a stochastic setting (i.e., all intermediate lemmas and final theorem holds as it is with stochasticity). This is firstly owed to the fact that the construction of optimistic loss estimators $\{\Tilde{F}^{(t)}\}_{t \in [0,T]}$ as proposed will still satisfy conditions of Requirement \ref{assmptn: near stationary} with the same analysis for the aforementioned base algorithms in the stochastic setting (i.e., the mathematical proof in Appendix \ref{sec: optimistic_estimator_verification} of the manuscript remains unchanged with the same formulations of $\Tilde{F}^{(t)}$). Furthermore, in our analysis presented in Section \ref{sec: multi_scale_near_stationary_analysis} - \ref{sec: dynamic_regr_maintext_disc}, we do not use gradient terms. 
}
\end{document}